\documentclass[12pt]{article}
\usepackage{times}  
\usepackage{helvet} 
\usepackage{courier}  
\usepackage[hyphens]{url}  
\usepackage{graphicx} 
\usepackage{graphicx}  
\usepackage{natbib}
\usepackage{amsmath}
\usepackage{mathtools}
\usepackage{amsthm}
\usepackage{bm}
\usepackage{graphicx}
\usepackage[american]{babel}
\usepackage{microtype}
\usepackage{multirow}
\usepackage{diagbox}
\usepackage{thmtools}
\usepackage{thm-restate}
\usepackage{booktabs}
\usepackage{color}
\usepackage{xcolor}
\definecolor{Green}{RGB}{76, 136, 107}
\usepackage{footnote}
\usepackage{subfigure} 
\usepackage{epsfig}
\usepackage{dcolumn}
\usepackage{enumitem}
\usepackage{mathrsfs}
\usepackage{threeparttable}
\makeatletter
\renewcommand{\maketag@@@}[1]{\hbox{\m@th\normalsize\normalfont#1}}%
\makeatother
\usepackage{bbding} 
\usepackage{pifont} 
\usepackage{nomencl} 
\makenomenclature
\usepackage{adjustbox}

\usepackage{algorithm}
\usepackage{algorithmic}

\usepackage{amssymb}
\usepackage{bm}
\usepackage{textcomp} 
\newenvironment{sciabstract}{
\begin{quote} \baselineskip14pt\small\hfil {\bf Abstract} \hfil\\[3pt]}
{\end{quote}\vspace{6pt}}

\makeatletter
\newif\if@restonecol
\makeatother
\newtheorem{definition}{Definition}

\newtheorem{remark}{Remark}
\title{Scaling Multi-Objective Security Games Provably via Space Discretization Based Evolutionary Search}

\author{
Yu-Peng Wu$^{1}$, Hong Qian$^{1,2,*}$, Rong-Jun Qin$^{3,4}$, Yi Chen$^{5}$, Aimin Zhou$^{1,2}$\\
\small{
$^{1}$School of Computer Science and Technology, East China Normal University, Shanghai, China}\\
\small{
$^{2}$Shanghai Institute of AI for Education, East China Normal University, Shanghai, China}\\
\small{
$^{3}$National Key Laboratory for Novel Software Technology, Nanjing University, Nanjing, China}\\
\small{
$^{4}$Polixir Technologies, Nanjing, China}\\
\small{
$^{5}$Department of Computer Science and Artificial Intelligence, Wenzhou University, Wenzhou, China}\\
\small{
51215901031@stu.ecnu.edu.cn,\ hqian@cs.ecnu.edu.cn,}\\
\small{
qinrj@lamda.nju.edu.cn,\ kenyoncy2016@gmail.com,\ amzhou@cs.ecnu.edu.cn}\\
\small{$^*$Corresponding Author}
}

\date{}

\begin{document}

\baselineskip16pt

\maketitle 

\begin{sciabstract}
In the field of security, multi-objective security games (MOSGs) allow defenders to simultaneously protect targets from multiple heterogeneous attackers. MOSGs aim to simultaneously maximize all the heterogeneous payoffs, e.g., life, money, and crime rate, without merging heterogeneous attackers. 
In real-world scenarios, the number of heterogeneous attackers and targets to be protected may exceed the capability of most existing state-of-the-art methods, i.e., MOSGs are limited by the issue of scalability.
To this end, this paper proposes a general framework called SDES based on many-objective evolutionary search to scale up MOSGs to large-scale targets and heterogeneous attackers. SDES consists of four consecutive key components, i.e., discretization, optimization, evaluation, and refinement. 
Specifically, SDES first discretizes the originally high-dimensional continuous solution space to the low-dimensional discrete one by the maximal indifference property in game theory. This property helps evolutionary algorithms (EAs) bypass the high-dimensional step function and ensure a well-convergent Pareto front. 
Then, a many-objective EA is used for optimization in the low-dimensional discrete solution space to obtain a well-spaced Pareto front.
To evaluate solutions, SDES restores solutions back to the original space via greedily optimizing a novel divergence measurement.
Finally, the refinement in SDES boosts the optimization performance with acceptable cost. 
Theoretically, we prove the optimization consistency and convergence of SDES.
Experiment results show that SDES is the first linear-time MOSG algorithm for both large-scale attackers and targets. SDES is able to solve up to 20 attackers and 100 targets MOSG problems, while the state-of-the-art (SOTA) methods can only solve up to 8 attackers and 25 targets ones. Ablation study verifies the necessity of all components in SDES.
\end{sciabstract}


\section{Introduction} \label{sec:introduction}
Security games (SGs) study how to efficiently use limited resources to defend targets against potential attackers in the security field. SGs have facilitated a variety of successful applications, including ARMOR, IRIS, and GUARDS for airport protection~\citep{ARMOR, IRIS, GUARDS}, TRUSTS for transit system protection~\citep{TRUSTS}, PROTECT for coastal protection~\citep{PROTECT} and others~\citep{otherApplication}. SGs can be modeled as Stackelberg games~\citep{StackelberGame}, where the solution concept, Strong Stackelberg Equilibrium (SSE), a refinement of the Nash Equilibrium, always exists~\citep{SSE}. Traditional SG work usually uses scalar-based decision-making methods, e.g., Bayesian SGs algorithms~\citep{MultiLPs, DOBSS, ERASER}. 
However, the scalar-based methods are not suitable for handing heterogeneous SGs, which involves multiple non-mergeable attackers. For instance, the loss of property and life is not of the same magnitude~\citep{MOSG}. 
 
This yields multi-objective security games (MOSGs)~\citep{MOSG,extended-MOSG}, which aims to find a well-convergent and well-spaced Pareto front (PF). Herein, ``well-spaced'' refers to good diversity and distribution~\citep{MOEA-survey}.
MOSGs model heterogeneous, possibly conflicting, attackers by multi-objective optimization. 
In MOSGs, each attacker maximizes its own payoff by attacking an appropriate target, while the defender maximizes payoffs from defending against all heterogeneous attackers. 
%
Modeling security games as multi-objective optimization has the benefits of a PF for decision-makers, and less prior knowledge demands such as probability distribution over potential attackers or historical attack data~\citep{extended-MOSG}.
With the deepening of MOSGs research, real-world problems present simultaneous large-scale requirements in attackers (e.g., a wide variety of crime types) and targets (e.g., massive infrastructure in public safety field).

The existing methods address MOSG problems with large-scale targets at the cost of greatly reducing the number of attackers they are able to handle.
Specifically, a faster convergence with the large target number is achieved by ORIGAMI algorithm (optimizing resources in games using maximal indifference)~\citep{ERASER}, which uses the maximal indifference property to address fully rational adversary games. The maximal indifference property focuses on the most attractive targets so as to make the large-scale target MOSG tractable. ORIGAMI-M and ORIGAMI-A~\citep{MOSG} extend ORIGAMI~\citep{ERASER} from single-objective optimization to multi-objective optimization and they are insensitive to the target dimension. 
However, the sample complexity of ORIGAMI-M and ORIGAMI-A finding out a well-spaced PF is exponentially with the number of attackers.
Some pruning-based methods are then proposed to reduce the algorithmic complexity via pre-skipping inferior policy~\citep{MOSG,extended-MOSG}. However, they are still restricted by the number of attackers. The reason is that they adopt the $\epsilon$-constraint framework, which optimizes the convergence and diversity of the solution simultaneously~\citep{DMOEA-eC}. As the number of attackers increase, $\epsilon$-constraint framework requires exponential solutions to guarantee a well-spaced PF. 
Meanwhile, MOSG problem of large-scale attackers are in demand in the real world. 
For example, in the field of public security, the criminal justice system covers a wide range of criminal offenses. Even according to their pertinence, there are more than ten kinds of heterogeneous attack types, including drug crimes, homicide, inchoate crimes, cybercrime, juvenile crimes, sex crimes, etc.~\footnote{https://www.justia.com/criminal/offenses/}
Therefore, it is urgent and important to design a method that can handle MOSG problem of large-scale heterogeneous attackers without losing the large number of targets.

To this end, this paper proposes a framework termed space discretization based evolutionary search (SDES) that is able to simultaneously extend targets and heterogeneous attackers to the large-scale scenario to meet the real-world demand. 
In the large-scale scenario, traditional methods~\citep{MOSG,extended-MOSG} fail to find a well-convergent and well-spaced PF, since they optimize the convergence and diversity of the solution simultaneously. 
However, directly applying heuristic algorithms such as evolutionary algorithms (EAs) is not suitable to scale up MOSGs. This is because it encounters a high-dimensional step function whose landscape changes frequently. 
In SDES, the goal of the well-convergent and well-spaced PF is decoupled into the convergence and the distribution of solutions. 
SDES achieves these two goals separately by the proposed discretization, optimization, evaluation and refinement components.
The optimization component aims to achieve a well-spaced PF, while the evaluation and refinement components aim to make the well-spaced PF well converge. Specifically, the discretization component transforms the high-dimensional continuous solution space into the low-dimensional discrete one by the maximal indifference property. This game property can guarantee solution convergence and avoid directly optimizing the high-dimensional step function.
The optimization component searches a set of low-dimensional solutions by many-objective EAs (e.g., NSGA-III~\citep{NSGA-III}) and an adaptive reference direction set generator (e.g., Riesz~\citep{ref-dirs-energy}).
In the evaluation component, the low-dimensional solutions are mapped to the high-dimensional ones via a greedy algorithm.
This component can find out solutions to approximate the true PF under the convergence assumption.
If a MOSG problem is too complex to satisfy the convergence assumption, the refinement component tries to improve the solution convergence. 

To the best of our knowledge, SDES is the first linear sample complexity framework if either the number of targets or attackers is fixed.
Theoretically, we prove the consistency of the greedy algorithm. Furthermore, a sufficient convergence condition for the SDES framework is disclosed.
Empirically, on the widely-used MOSG benchmark, the linearity of the SDES framework is verified if either the number of targets or attackers is fixed. The results show that SDES outperforms the state-of-the-art (SOTA) methods in terms of scalability, time efficiency, and effectiveness. Specifically, SDES can solve up to 20 attackers and 100 targets MOSG problems, while the SOTA methods can only solve up to 8 attackers and 25 targets MOSG problems. Ablation study verifies the necessity of all components in SDES.

In general, the contributions of this paper include:

\begin{itemize}
    \item We generalize the related theory of the \textit{attack set} $\Gamma$ from the single attacker's scenario~\cite{ERASER} to the multiple heterogeneous attacker's scenario. Specifically, (i) The relationship between the \textit{resource allocation} $\bm{c}$ and $\Gamma$ is formally analyzed, cf. Lemma~\ref{pro:property1.2}. (ii) The step change of the optimization objective of MOSG (a high-dimensional step function) is quantified, which is meaningful to characterize the difficulty of MOSG problem, cf. Lemma~\ref{pro:property1.1}.
    
    \item We transform the complex original MOSG problem (a high-dimensional step functions) into a combinatorial optimization problem by the discretization component of SDES. Then, we propose a novel solution divergence measurement in the evaluation component to assist greedy algorithms to bypass the potential combinatorial explosion. Furthermore, we prove the optimization consistency and convergence of SDES.

    \item We propose the first linear sample complexity framework SDES if either the number of targets or attackers is fixed. SDES makes the MOSG application scenarios extend from medium-scale scenarios~\cite{MOSG,extended-MOSG} (number of attackers $3\sim5$, number of targets $25\sim100$) to large-scale scenarios (number of attackers $6\sim20$, number of targets $200\sim1000$).
\end{itemize}

The rest of the paper is organized as follows. Section~\ref{sec:rel_wor} reviews the related work. Section~\ref{sec:pre} introduces the preliminaries of SGs and MOSGs. Section~\ref{sec:met} presents the framework and sample complexity of SDES. Section~\ref{sec:val_ana} further analyses the theoretical guarantee of SDES. Section~\ref{sec:exp} shows the experiment results of SDES compared with other SOTA methods with respect to scalability, time efficiency, effectiveness and others. Finally, Section~\ref{sec:conclusion} concludes the paper.

\section{Related Work}
\label{sec:rel_wor}
Modeling SG can be roughly categorized into single-objective and multi-objective methods, where Bayesian Stackelberg games and multi-objective security games (MOSGs) are their representative approaches respectively. This section first introduces those two research lines separately, followed by their corresponding work on large-scale problems.


\textbf{Bayesian Stackelberg Games.}
The initial research studies on the Bayesian Stackelberg games primarily investigate the optimal resource allocation yielding the best defender rewards~\cite{MultiLPs, DOBSS, ERASER}. MultiLPs~\cite{MultiLPs} can transform the game tables of different attackers into a large standard game table through Harsanyi transformation to solve the Bayesian Stackelberg game. However, the idea of MultiLPs is to enumerate linear program problems of all possible strategy profiles, its complexity varies exponentially with the attacker dimension. Although the SSE of Bayesian Stackelberg has been proved to be an NP-hard problem~\cite{MultiLPs}, researchers can utilize some characteristics of the problem to design more sophisticated algorithms, e.g., the independence of attackers. DOBSS~\cite{DOBSS}, the first algorithm successfully applied to ARMOR~\cite{ARMOR} system, takes advantage of the independence of attackers to convert several linear program problems into a mixed-integer linear program problem that can be solved by effectively programming-based approaches. In addition, ERASER~\cite{ERASER} directly solves SGs in mixed-integer linear program formulation and makes the target dimension no longer affect the solving efficiency. Although such Bayesian Stackelberg game methods can yield accurate optimal resource allocation, their single-objective optimization modeling limited their performance in SG problems with heterogeneous attackers.

\textbf{Multi-Objective Security Games.}
Compared with Bayesian Stackelberg games, MOSGs can better meet the real-world needs of heterogeneous attackers~\citep{MOSG,extended-MOSG}. The concept of MOSG formulation is systematically proposed in~\citep{MOSG} for the first time by a multi-objective optimization framework combining the $\epsilon$-constraint framework and lexicographic method. Similar to MultiLPs~\citep{MultiLPs}, the idea of~\citep{MOSG} is to enumerate all strategy profiles and call ERASER solver~\citep{ERASER} several times to yield a PF. Although~\citep{MOSG} pre-skips some policy profiles with pruning techniques, the attacker dimension is still a bottleneck. Later, on the basis of~\citep{MOSG}, an extended study~\citep{extended-MOSG} systematically introduces various pruning methods, but it still fails to eradicate the curse of dimensionality brought in attacker dimension by $\epsilon$-constraint framework and lexicographic method. Meanwhile, solution distribution and convergence are two key indicators in multi-objective task solution evaluation, which are coupled by $\epsilon$-constraint methods~\citep{MOSG,extended-MOSG}. To ensure solution uniformity, those methods face the curse of dimensionality in the exponentially growing objective space. To our knowledge, although the performance of the existing algorithm is no longer affected by the target dimension, the attacker dimension is still the bottleneck limiting most traditional methods. In this paper, our work handle large-scale problem in both target and heterogeneous attacker dimensions by decoupling solution distribution and convergence.  

\textbf{Large-Scale Security Games.} Although there are a lot of works focusing on large-scale SGs, most methods cannot deal with the general large-scale MOSGs well. In~\citep{patroller-PSG}, researchers solve specific large-scale security games, but this work cannot be easily applied to the general large-scale problem. N. Basilico et al.~\citep{patroller-PSG} study an extensive-form infinite-horizon underlying game for patrolling security games under specific assumptions about attacker behavior. 
In EASG~\citep{EASG}, the researches provides a large-scale game-independent EA-based algorithm, which is tailored to sequential SGs. 
At the same time, programming-based methods like~\citep{ERASER} can deal with large-scale target or resource scenarios, which generally merge the loss of attackers and violate the problem setting of heterogeneous attackers. Besides, most programming-based algorithms cannot solve large-scale problems of multi-dimension simultaneously, e.g., targets and attackers. In addition, a large number of studies have focused on network security games (NSGs)~\citep{NSGZero,NFSP}. 
NSGZero~\citep{NSGZero} and NFSP~\citep{NFSP} solve the large-scale extensive-form NSGs and large-scale resources NSGs by deep-learning methods respectively. However, some real-world problems cannot be modeled as NSGs. For example, RAND~\citep{Scale-up-Approximation-ARA-RAND} proposes a unified model of large-scale SGs to solve the deployed federal air marshals and threat screening problem. This paper handles large-scale MOSGs by space discretization, which transforms solution representation from the high-dimensional continuous space to the low-dimensional discrete one. The proposed SDES is the first linear sample complexity EA-based framework for MOSGs if either the number of targets or attackers is fixed, which simultaneously extends targets and attackers to the large-scale scenario.

\section{Preliminaries}
\label{sec:pre}
\subsection{Security Game (SG)}
SG, a special kind of \textit{defender-attackers} Stackelberg game, consists of one \textit{defender} $\mathcal{D}$ and $N$ heterogeneous \textit{attackers} $\mathcal{A}_i$, $T$ targets and a resource ratio $r\in[0,1]$, where $i\in[N]$ and $[N]=\{1,\ldots,N\}$ for a positive integer $N$. All notations in this paper is summarized in~\ref{app:notation}. The continuous strategy of defender $\mathcal{D}$ can be represented as the coverage vector $\bm {c}\in [0, 1]^T$, where $||\bm{c}||_1 \leq r \cdot T$ and $c_t$ represents the resources allocated on the $t$-th target to defend all attackers and the probability of successfully covering the $t$-th target~\citep{Multi-defender21}. 
The strategy of $\mathcal{A}_i$ can be represented as the attack vector 
$\bm {a}_i=(a_{i1}, \ldots, a_{iT}) \in \{0, 1\}^T$, where $\sum_{t=1}^T {a_{it}} = 1$ and $a_{it}=1$ means that $t$ is the \textit{attacked target} of $\mathcal{A}_i$, denoted as $at_i$. 
Since multiple heterogeneous attackers do not interfere with each other and pursue their payoff maximization, the $N$-attacker SG is often transformed into $N$ two-player Stackelberg games. Given the strategy profile $(\bm{c}, \bm{a}_i)$, each player's payoff can be expressed as the payoff structure $\bm{U}=(U_1,\ldots,U_N)$, cf. Equation~\eqref{equ:equation2} and~\eqref{equ:equation2-2}. 

In a two-player Stackelberg game, $U_i$ consists of $\mathcal{A}_i$'s payoff $U_i^a(\bm {c}, \bm {a}_i)$ and $\mathcal{D}$'s payoff $U_i^d(\bm {c}, \bm {a}_i)$. 
If the attacked target $at_i$ is fully covered (i.e., $c_t=1$), $\mathcal{A}_i$'s or $\mathcal{D}$'s payoff is denoted as $U_i^{c,a}(t)$ or $U_i^{c,d}(t)$. If $at_i$ is completely uncovered (i.e., $c_t=0$), $\mathcal{A}_i$'s or $\mathcal{D}$'s payoff is denoted as $U_i^{u,a}(t)$ or $U_i^{u,d}(t)$. Since $c_t$ is the cover probability, $\mathcal{A}_i$'s or $\mathcal{D}$'s payoff on $t$ is

\begin{equation} 
\label{equ:equation1}
U_i^{a}(c_t)=c_t U_i^{c,a}(t) + (1-c_t) U_i^{u,a}(t)\,,
\end{equation}
\begin{equation} 
\label{equ:equation1-2}
U_i^{d}(c_t)=c_t U_i^{c,d}(t) + (1-c_t) U_i^{u,d}(t)\,.
\end{equation}
Equation~\eqref{equ:equation1} and \eqref{equ:equation1-2} show $U_i^{a}(c_t)\propto {\frac {1}{c_t}}$, $U_i^{d}(c_t)\propto {c_t}$, since all $U_i^c$ and $U_i^u$ follow the assumption $U_i^{c,a} \leq U_i^{u,a}$, $U_i^{c,d} \geq U_i^{u,d}$. This means that $\mathcal{A}_i$ does not want $at_i$ to be covered, while the defender does the opposite. Finally, given the strategy profile $( \bm {c}, \bm {a}_i )$, $\mathcal{A}_i$'s or $\mathcal{D}$'s payoff for $T$ targets is

\begin{equation} \label{equ:equation2}
U_i^{a}(\bm {c}, \bm {a}_i) = \sum_{t=1}^T {a_{it} U_i^{a}(c_t)}\,,
\end{equation}
\begin{equation} \label{equ:equation2-2}
U_i^{d}(\bm {c}, \bm {a}_i) = \sum_{t=1}^T {a_{it} U_i^{d}(c_t)}\,.
\end{equation}



For $\mathcal{A}_i$, the collection of targets yielding the best $U^a_i(c_t)$ is called the \textit{attack set} $\Gamma_i(\bm{c})$, defined as:

\begin{equation}
\label{equ:equationBRa}
\Gamma_i(\bm {c})=\left\{t\mid U^a_i(c_t)\geq U^a_i(c_{t'})\forall t'\in [T]\right\}\,.
\end{equation}

$\Gamma_i(\bm{c})$ is defined as the set of targets which yields the maximum expected payoff for $\mathcal{A}_i$. The target $t\in\Gamma_i(\bm{c})$ have the same attraction to $\mathcal{A}_i$, while the attraction of target $t\notin\Gamma_i(\bm{c})$ is $0$. The realistic meaning of $\Gamma_i(\bm{c})$ is narrowing the solution space and guaranteeing the optimal solution.

The Stackelberg interaction between $\mathcal{D}$ and $\mathcal{A}_i$ is the $\mathcal{A}_i$ computes $\Gamma_i(\bm{c})$ and then breaks ties in favor of $\mathcal{D}$~\citep{SSE}. Finally, the $A_i$'s expected payoff is $\max_{t\in[T]}U_i^a(c_t)$, and the $D$'s expected payoff against $\mathcal{A}_i$ is $\max_{t\in\Gamma_i(\bm{c})}U_i^d(c_t)$.

\subsection{Multi-Objective Security Games (MOSGs)}
\label{sec:Problem_Definition}
In MOSGs, $\mathcal{D}$ wants to defend $T$ targets against $N$ heterogeneous $\mathcal{A}_i$ with limited resources.
Under the fully rational attacker behavior assumption, $\bm {a}_i$ is deterministic, thus the player's payoff $U_i^{a}(\bm {c}, \bm {a}_i)$ or $U_i^{d}(\bm {c}, \bm {a}_i)$ can be reformulated as $U_i^{a}(\bm {c})$ or $U_i^{d}(\bm {c})$. The objective function of MOSGs is

\begin{equation}
\label{equ:equationOptFunc}
\max_{\bm{c}\in [0,1]^T}\bm {F}(\bm {c})=\left\{ U_1^{d}(\bm {c}),\ldots,U_N^{d}(\bm {c}) \right\}\,.
\end{equation}

MOSGs can obtain multiple solutions. The definition of the Pareto dominance relation is widely used in the solution comparison, cf. Definition~\ref{def:Pareto}.

\begin{definition}[Dominance Relation and Pareto optimality]\label{def:Pareto}
$\bm{c}$ weak dominates $\bm{c}'$ if $U_i^d(\bm{c})\geq U_i^d(\bm{c}')$, $\forall i\in[N]$, denoted as $\bm{c}\succeq \bm{c}'$. $\bm{c}$ \textbf{dominates}  $\bm{c}'$ if $U_i^d(\bm{c})\geq U_i^d(\bm{c}')$, $\forall i\in[N]$, and $U_j^d(\bm{c})>U_j^d(\bm{c}')$, $\exists j\in[N]$, denoted as $\bm{c}\succ \bm{c}'$.
The \textit{coverage vector} $\bm{c}$ is $Pareto\ solution$ if there is no other $\bm{c}'\in C$ satisfying $\bm{c}\succ \bm{c}'$. The set of non-dominant $coverage\ vectors$ is called \textbf{Pareto set}.
\end{definition}


\section{The Proposed Method}
\label{sec:met}
The proposed SDES framework consists of discretization, optimization, evaluation, and refinement components, as depicted in Figure~\ref{fig:framework}. Traditional methods for MOSGs mainly focus on high-dimensional continuous solution space $\bm{X}=[0,1]^T$. However, the discretization component shows that searching on low-dimensional discrete solution space $\bm{X}'=\Gamma(\bm{c})$ (cf. Equation~\eqref{equ:equationBRa}) is better.
Therefore, the optimization component optimizes solutions on $\bm{X}'$, and the evaluation component evaluates solutions on $\bm{X}$. Finally, since MOSGs may not meet the convergence assumption, the refinement component attempts to improve solutions with acceptable cost.

\begin{figure}
\centering
\includegraphics[width=0.5\columnwidth]{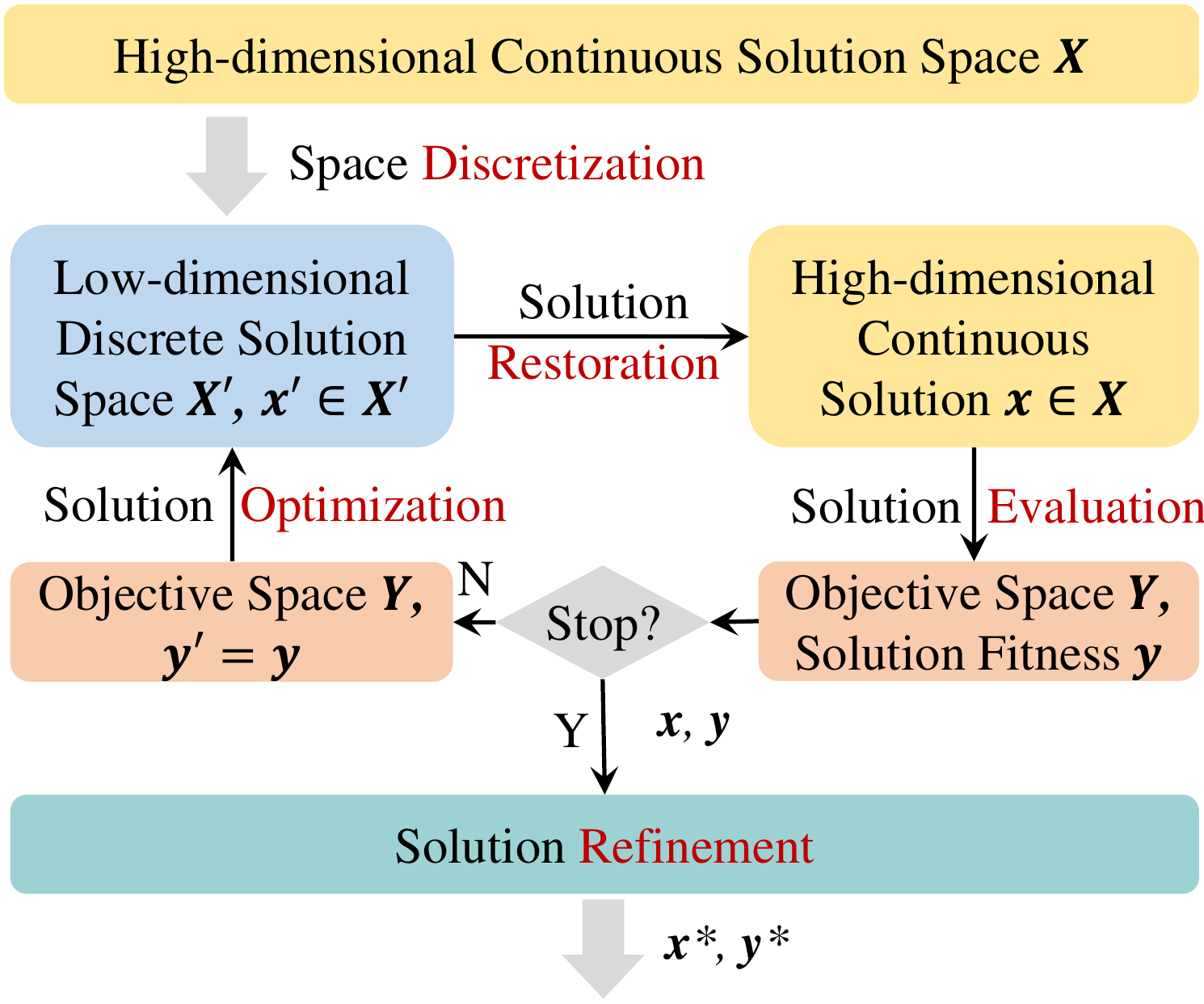}
\caption{The framework of SDES. SDES consists of discretization, optimization, evaluation, and refinement components. SDES evaluates solutions in the high-dimensional continuous space $\bm{X}$ and optimizes solutions in the low-dimensional discrete space $\bm{X}'$. The objective space of $\bm{X}$ and $\bm{X'}$ is the same, i.e, $\bm{Y}$. $\bm{y}'$ is the fitness of low-dimensional solution $\bm{x}'$, while $\bm{y}$ is the fitness of the corresponding high-dimensional solution $\bm{x}$, and $\bm{y}'$ equals $\bm{y}$. If the stop criterion is met, the refinement component refines high-dimensional solutions and returns an approximate PF.}
\label{fig:framework}
\end{figure}

\subsection{Discretization}
\label{sec:dis}
\subsubsection{Why Discretization}
\label{sec:property}
 In the fully rational SGs, since $U^d_i$ is a function of $\bm {c}$ in Equation~\eqref{equ:equation1}, many programming-based methods optimize the problem on $\bm {c}$ directly. 
 However, general heuristic algorithms cannot solve MOSG problem well by searching $\bm{c}$ directly, as disclosed in Lemma~\ref{pro:property1.2} and~\ref{pro:property1.1}. 
 This section reveals two properties of $\Gamma(\bm{c})$, i.e., the vulnerability of $\Gamma(\bm{c})$ and the maximal indifference of $\Gamma(\bm{c})$. The former shows that it is very difficult for the general heuristic or gradient-based optimization algorithms to optimize high-dimensional $\bm {c}$ directly. The latter provides the way of discretization and the theoretical support for optimizing low-dimensional discrete $\Gamma(\bm{c})$. 

\textbf{Vulnerability of $\Gamma(\bm{c})$.}
\label{sec:property1}
This section shows that general heuristic and gradient optimization algorithms cannot optimize high-dimensional $\bm {c}$ directly in the continuous solution space, because there are lots of flat and jump regions in the continuous solution space. The following illustrates the vulnerability of $\Gamma(\bm{c})$ from the perspective of security game as Lemma~\ref{pro:property1.2} and the perspective of optimization as Lemma~\ref{pro:property1.1}. The proof details are shown in~\ref{app:proof:method}.

\begin{restatable}[Vulnerability of $\Gamma(\bm{c})$ from Security Game Perspective]{lemma}{lemmaSGperspective}\label{pro:property1.2}
For $\mathcal{A}_i$, given $\Gamma_i(\bm {c})$, $\Gamma_i(\bm {c}+\bm{v})$, where $\bm{c},\bm{v}\in[0,1]^T$ and $\bm{v}$ is a vector with one small non-zero component $v_{t'}$, suppose $||\bm{c}+\bm{v}||_1\le r\cdot T$, one has that
\begin{equation}
\begin{cases}
{\Gamma_i(\bm{c}+\bm{v})=\Gamma_i(\bm{c}),}&{\text{if}}\ {t'\notin \Gamma_i(\bm{c})},
\\
{\Gamma_i(\bm {c}+\bm{v})=\Gamma_i(\bm {c})\setminus \{t'\},}&{\text{if}}\ {t'\in \Gamma_i(\bm{c})\ \text{and}\ v_{t'}>0},
\\
{\Gamma_i(\bm {c}+\bm{v})=\{t'\},}&{\text{if}}\ {t'\in \Gamma_i(\bm{c})\ \text{and}\ v_{t'}<0}.
\end{cases}
\end{equation}
\end{restatable}

From the perspective of security game, the reason for Lemma~\ref{pro:property1.2} is that allocating resources to $t\notin \Gamma_i(\bm{c})$ does not increase the payoff since $\mathcal{A}_i$ only attack $t\in\Gamma_i(\bm{c})$. Specifically, the actual $\vert\Gamma_i(\bm{c})\vert$ is far less than $T$ due to limited resources. Large-scale target scenarios only provide sparse feedback for optimization. Furthermore, the change of $\bm{c}$ causes a drastic payoff $U_i^d(\bm{c})$ jump. Figure~\ref{fig:Ua_Ud} shows the existence of many flat and jump regions.

\begin{restatable}[Vulnerability of $\Gamma(\bm{c})$ from Optimization Perspective]{lemma}{lemmaoptperspective}\label{pro:property1.1}
For $\mathcal{A}_i$, the same $\Gamma_i(\bm{c})$ and $\Gamma_i(\bm{c}+\bm{v})$ as Lemma~\ref{pro:property1.2} are provided. Given the attacked target $at_i$ corresponding to $\Gamma_i(\bm{c})$, if $at_i\in\Gamma_i(\bm{c}+\bm{v})$, then $U_i^d(\bm {c})$$=$$U^d_i(\bm{c}+\bm{v})$. Otherwise, $U_i^d(\bm{c}+\bm{v})-U_i^d(\bm{c})\geq gap_i$, where $gap_i=\max_{t\in\Gamma_i(\bm{c})}U^d_i(c_t) - \max_{t\in\Gamma_i(\bm{c})\setminus\{at_i\}}U^d_i(c_t)$.
\end{restatable}

From the perspective of optimization, as shown in Figure~\ref{fig:Ua_Ud}, $U_i^d(\bm {c})$ does not change smoothly. Meanwhile, the space of $\bm{c}$ may reach thousands of non-differentiable dimensions, making the problem hard to converge.

\begin{figure}[htbp]
\centering
\includegraphics[width=0.6\columnwidth]{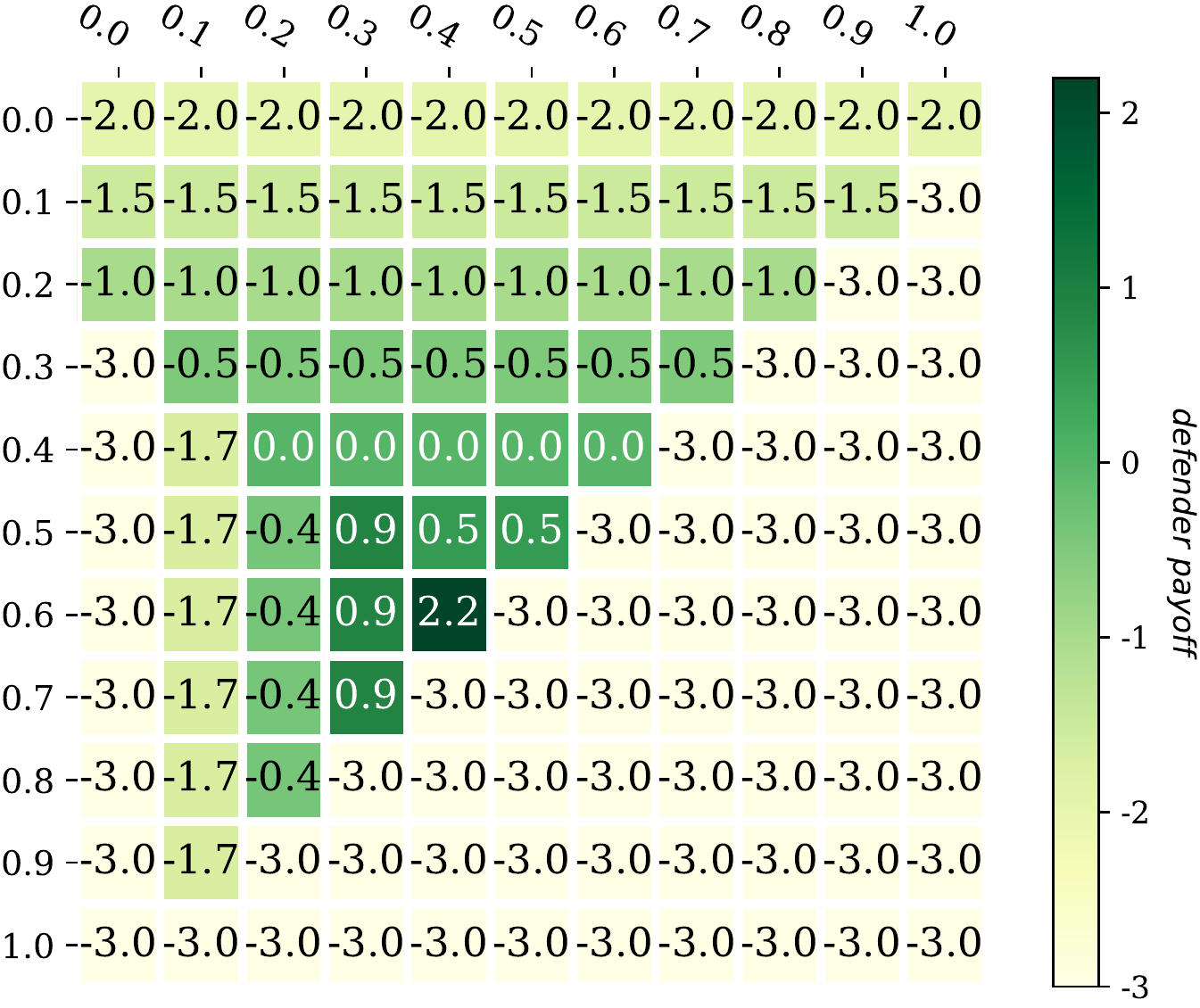}
\caption{The heatmap of $\mathcal{D}$'s payoff $U^d_i(\bm{c})$. It is an illustration in the small-scale $N=T=2$ MOSG problem. The axes mean the resources allocated on $t_1$ and $t_2$. The defender payoff optimization is a maximum task, whose goal is to find the darkest area, e.g., $2.2$. 
The payoff of the area whose total resources violate the constraint, i.e., $\Vert\bm{c}\Vert_1\leq r\cdot T=1$, equals the global minimum, e.g., $-3.0$.}
\label{fig:Ua_Ud}
\end{figure}

\textbf{Maximal Indifference of $\Gamma(\bm{c})$.}
Lemma~\ref{pro:property1.2} and~\ref{pro:property1.1} reveal that the landscape of large-scale MOSG problem within $[0,1]^T$ is a high-dimensional step function whose landscape changes frequently, which is difficult to solve. Fortunately, Lemma~\ref{pro:property2} shows that it is feasible to solve large-scale MOSG problems in the discrete space $\Gamma(\bm{c})$.

\begin{restatable}[Maximal Indifference of $\Gamma(\bm{c})$~\citep{ERASER}]{lemma}{lemmaoptimalexpansion}\label{pro:property2}
For $\mathcal{A}_i$, 
the optimal expansion rule of $\Gamma_i(\bm{c})$ is adding $t$ in the descending order of $U^{u,a}_i(t)$.
\end{restatable}

Lemma~\ref{pro:property2} is the core of the exact algorithm ORIGAMI~\citep{ERASER}. The optimal expansion of $\Gamma_i(\bm{c})$ refers to the process of adding a target $t\notin\Gamma_i(\bm{c})$ to $\Gamma_i(\bm{c})$ without affecting the original $\Gamma_i(\bm{c})$. The optimal expansion of $\Gamma_i(\bm{c})$ yields the optimal $\Gamma_i(\bm{c})$, which maximizes the efficiency of resource utilization. Lemma~\ref{pro:property2} describes the optimal way for expanding $\Gamma_i(\bm{c})$ from $\vert\Gamma_i(\bm{c})\vert=1$ to $\vert\Gamma_i(\bm{c})\vert=T$, called the \textit{optimal target order}. 
The significance of Lemma~\ref{pro:property2} is to determine the optimal target order in advance, which reduces the solution search space.
The search algorithm no longer needs to care about the content details of the optimal $\Gamma_i(\bm{c})$ anymore, but only the optimal size $\vert\Gamma_i(\bm{c})\vert$ because of the optimal target order. Therefore, the algorithm only needs to use the one-dimensional integer to represent the optimal $\Gamma_i(\bm{c})$. Similarly, in the $N$-attacker SG, the $N$-dimensional integer is enough for representing the optimal ($\Gamma_1(\bm{c}),\ldots,\Gamma_N(\bm{c})$), which is defined as the \textit{attack group} $\Gamma_s(\bm{c})$.


\subsubsection{How to Realize Discretization}
\label{sec:ind_rep}
According to Lemma~\ref{pro:property2}, 
$T$-length \textit{continuous} solution representation $\bm{c}$ termed \textbf{$\bm{c}$-code} can be uniquely discretized into $N$-length \textit{integral} (\textit{discrete}) solution representation $\bm{I}$ termed \textbf{$\bm{I}$-code}, generally $N\ll T$.

\begin{figure}[htbp]
\centering
\includegraphics[width=0.6\columnwidth]{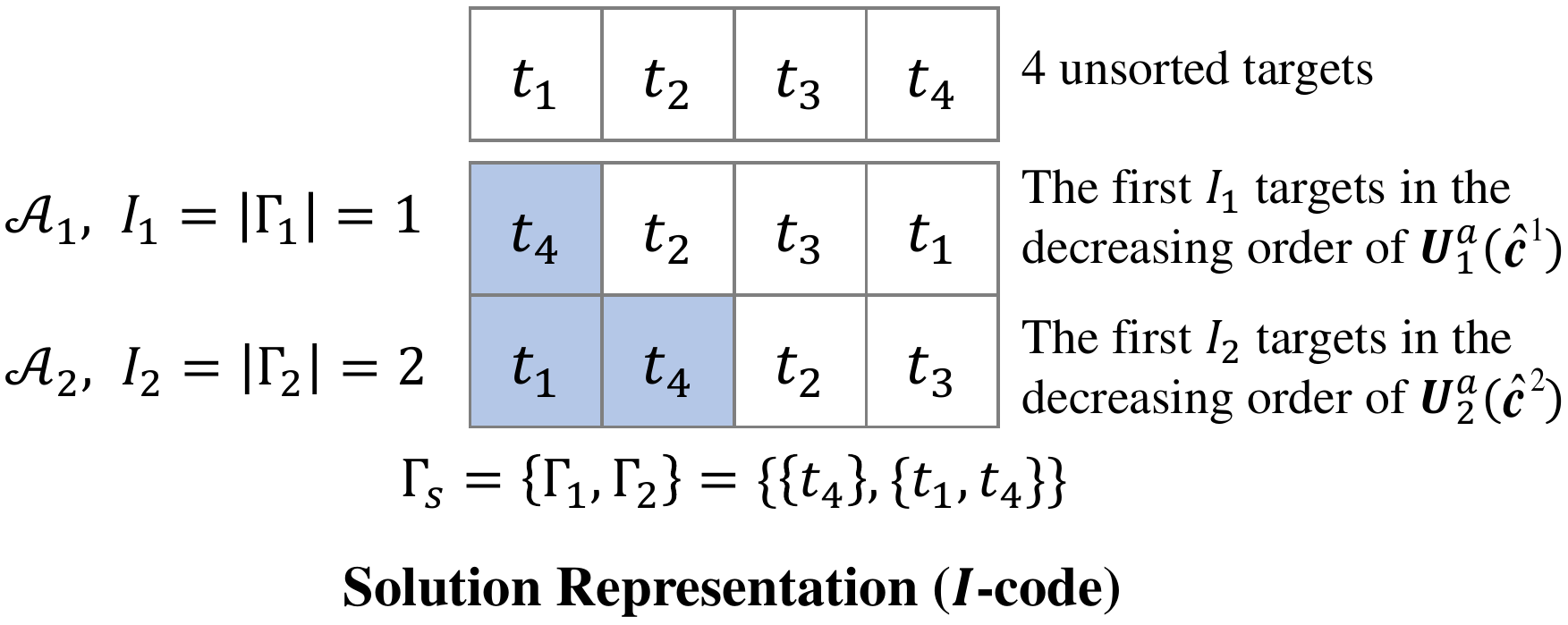}
\caption{The solution representation about $I$-code. It is an illustration introducing \textit{attack set} $\Gamma_i(\bm{c})$, and \textit{attacker group} $\Gamma_s(\bm{c})$ in the $N=2, T=4$ MOSG problem. 
$I_i$ represents $\Gamma_i(\bm{c})$'s size $\vert\Gamma_i(\bm{c})\vert$. 
$I_i$ uniquely determines $\Gamma_i(\bm{c})$, that is, $\Gamma_i(\bm{c})$ consists of the first $I_i$ targets in the decreasing order of $U^{u,a}_i(t)$.
All $\Gamma_i(\bm{c})$ make up $\Gamma_s(\bm{c})$. For example, $\bm{I}=(1, 2)$ indicates $\Gamma_s=\{\{t_4\}$, $\{t_1, t_4\}\}$. Since both $\mathcal{A}_1$ and $\mathcal{A}_2$ attack only one target, $\bm{I}$ can induce 2 possible $\bm{at}$, i.e., $(t_4, t_1)$ or $(t_4, t_4)$. 
} 
\label{fig:sol_rep}
\end{figure}

For the implementation of discretization component, $\mathcal{A}_i$'s optimal $\Gamma_i$ satisfies $\Gamma_i(\bm{c})\in\{\{t_i^{(j)}\}_{j=1}^1,\ldots,\{t_i^{(j)}\}_{j=1}^T\}$, where $t_i^{(1)},\cdots,t_i^{(T)}$ is the targets in the optimal order of descending $U^{u,a}_i(t)$. Figure~\ref{fig:sol_rep} gives an illustration of discretization process. Specifically, for $\mathcal{A}_i$, compute the optimal target order in the descending order of $U_i^{u,a}(t)$. The discrete solution representation $\bm{I}$ can uniquely represent $\Gamma_s(\bm{c})$.


\subsection{Optimization}
\label{sec:opt}
The goal of MOSGs, cf. Equation~\eqref{equ:equationOptFunc}, is to find a solution set $C^*\subseteq [0,1]^T$ that corresponds to a well-spaced PF.  
Since the discretization component maps solutions from $c$-code representation to simplified $I$-code representation, each $I$-code corresponds to a set of $c$-code, denoted as $\tilde{C}$. To find $C^*$, the optimization component selects one non-dominant solution $\bm{c}^*$ from each $\tilde{C}$. 
Problem~\eqref{equ:equationOptFunc} is transformed to Equation~\eqref{equ:equationBi-OptFunc}.
\begin{equation}
\label{equ:equationBi-OptFunc}
\begin{aligned}
& \max_{\bm {c}\in C^*}\bm {F}(\bm {c})=\max_{\bm {c}\in C^*}( U_1^{d}(\bm {c}),\ldots,U_N^{d}(\bm {c})) \,,\\
& s.t.\ C^*\!=\left\{\bm{c}^*_i\mid\bm{c}^*_i\succeq \bm{c}'_i, \bm{c}^*_i\in\tilde{C}_i, \bm{c}'_i\in\tilde{C}_i\setminus\{\bm{c}^*_i\}, i\in [\vert\tilde{C}\vert]\right\}\,.
\end{aligned}
\end{equation}
Equation~\eqref{equ:equationBi-OptFunc} focuses on the diversity and distribution of solutions. The convergence of found $\bm{c}^*_i$ is detailed in Section~\ref{sec:rat_sim}. Equation~\eqref{equ:equationBi-OptFunc} requires only weak dominance $\bm{c}^*_i\succeq\bm{c}'_i$. The reason is to deal with situations where multiple non-dominant solutions have the same payoff.

For the implementation of the optimization component, we use a stable many-objective evolutionary algorithm NSGA-III~\citep{NSGA-III} and an energy-based adaptive reference direction set generator Riesz~\citep{ref-dirs-energy}. We should point out that NSGA-III and Riesz are not the only options. The reasons for the option of this paper will be fully explained in Section~\ref{sec:secCon_eva}.
Specifically, Riesz defines a set of problem-independent reference vectors in advance, which are uniformly distributed in the objective space and invariant throughout the NSGA-III evolution process. Riesz uses $\vert\tilde{C}\vert$ (population size) reference vectors to divide the space of $\bm{c}$ into multiple subspaces $\tilde{C}$. NSGA-III returns one non-dominant solution $\bm{c}^*$ from each $\tilde{C}$ NSGA-III uses Riesz's reference vectors to maintains the diversity among the population. 
In each iteration of optimization, the many-objective evolutionary algorithm assigns at least one solution to each reference vector to maintain a set of well-spaced solutions. Meanwhile, NSGA-III preserves well-convergent solutions by the dominance relation. Especially, NSGA-III preserves the non-domination layers in order until a layer cannot be preserved completely due to the limited population size. NSGA-III then preserves the top well-spaced solutions. More details are shown in~\ref{app:NSGA-III}.


\subsection{Evaluation}
\label{sec:eva}
\subsubsection{Why Evaluation}
In each iteration of the optimization component in SDES, the general many-objective EA operators generates new offspring $I$-code with unknown fitness. 
To evaluate $I$-code, the evaluation component formulates the process of mapping $I$-code to $c$-code as a combinatorial optimization problem, cf. Equation~\eqref{equ:c_gamma}. And then, to solve the combinatorial optimization and bypass the potential combinatorial explosion, the evaluation component provides a greedy algorithm (BitOpt) and a solution divergence measurement for BitOpt. An illustration of the evaluation component is shown in Figure~\ref{fig:L-code}.

\begin{figure}[!htbp]
\centering
\includegraphics[width=\textwidth]{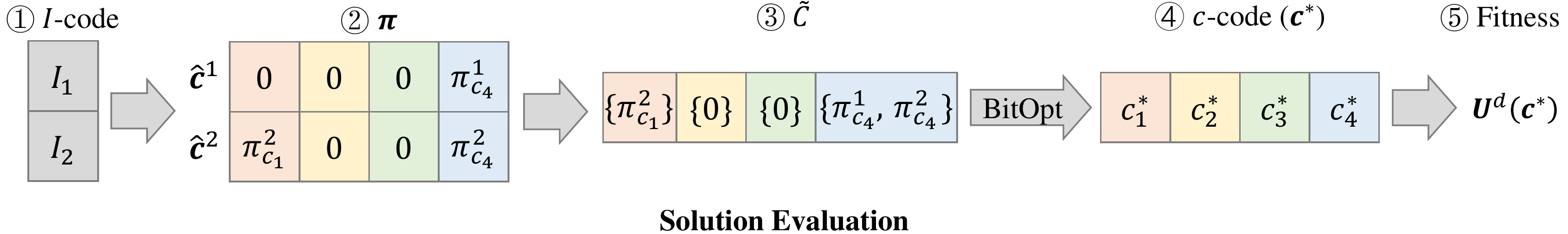}
\caption{The solution evaluation. It is an illustration of in the problem as same as Figure~\ref{fig:sol_rep}. In the step \ding{172}, $I$-code indicates \textit{attack set} $\Gamma_1(\bm{c})=\{t_4\}$ and $\Gamma_2(\bm{c})=\{t_1,t_4\}$. In the step \ding{173}, the indifference equation calculates \textit{alternative set} $\pi_{c_1}=\{\pi_{c_1}^2\}$ for $c_1$ and $\pi_{c_4}=\{\pi_{c_4}^1$, $\pi_{c_4}^2\}$ for $c_4$. In the step \ding{174}, all alternative sets get \textit{subspace} $\tilde{C}$ by Cartesian product, i.e., $\tilde{C}=\pi_{c_1}\times\pi_{c_4}$. In the step \ding{175}, a greedy algorithm named BitOpt is used to find an optimal solution $\bm{c}^*$ and bypass the potential combinatorial explosion (here set $c_2=c_3=0$ directly). In the step \ding{176}, the fitness would guide the direction of the optimization component.
}
\label{fig:L-code}
\end{figure}

\subsubsection{How to Realize Evaluation}
\label{sec:ind_eva}

\textbf{Formalize the Combinatorial Optimization Problem.}
The evaluation component wants to map $I$-code to $c$-code to obtain solution fitness $\bm{U}^d(\bm{c})$. Since $I$-code and $c$-code are connected by $\Gamma(\bm{c})$, i.e., $I$-code$\to$$\Gamma(\bm{c})$$\to$\linebreak$c$-code, a combinatorial optimization problem for $\Gamma(\bm{c})$ can be formalized.

The combinatorial optimization problem is visualized in steps \ding{172}$\sim$\ding{174} of Figure~\ref{fig:L-code}. Specifically, in steps \ding{172}$\sim$\ding{173}, recall that $\mathcal{A}_i$'s $\Gamma_i(\bm{c})$ is defined as the set of all targets that produce the maximum $U^a_i(c_t)$. Therefore, if $t\notin\Gamma_i(\bm{c})$, $\mathcal{A}_i$'s resources alternative for $t$-th target must be $0$, denoted as $\pi_{c_t}^i=0$. Otherwise, $\pi_{c_t}^i$ can be calculated from the indifference equation $\pi_{c_t}^i=\frac{U^a(c_{t'})-U^{u,a}_i(t)}{U^{c,a}_i(t)-U^{u,a}_i(t)}$~\citep{ERASER}, where $t'\in\Gamma_i(\bm{c})$ and it comes from Equation~\eqref{equ:equation1}. In steps \ding{173}$\sim$\ding{174}, the search domain of each $c$-code component $c_t$ contains $N$ solution alternatives, denoted as the alternative set $\pi_{c_t}=\{\pi_{c_t}^1,\ldots,\pi_{c_t}^N\}$. Overall, the subspace $\tilde{C}$ mentioned in Equation~\eqref{equ:equationBi-OptFunc} is the Cartesian product of $T$ alternative sets $\pi_{c_t}$, i.e., $\tilde{C}=\bigtimes_{t=1}^T {\pi_{c_t}}$. The combinatorial optimization problem is formalized as finding the non-dominant solution $\bm{c}^*\in\tilde{C}$. 
Steps \ding{172}$\sim$\ding{174} in Figure~\ref{fig:L-code} is formalized as

\begin{equation}
\label{equ:c_gamma}
\begin{aligned}
& \bm{c}^*=\mathop{\arg\max}_{\bm{c}\in\bigtimes_{t=1}^T {\pi_{c_t}}}\bm{F}(\bm{c})\,,\\
s.t. & \pi_{c_t}=\{\pi_{c_t}^1,\ldots,\pi_{c_t}^N\}\forall t\in[T]\,,\\
&{\pi_{c_t}^i}=
\begin{cases}
\frac{U^a(c_{t'})-U^{u,a}_i(t)}{U^{c,a}_i(t)-U^{u,a}_i(t)},&{\text{if}}\ {t\in\Gamma_i(\bm{c})},\\
    {0,}                                        &{\text{otherwise.}}
\end{cases}\forall i\in[N]\,,
\end{aligned}
\end{equation}
where $\bm{F}(\bm{c})$ is $\mathcal{D}$'s comprehensive payoff $(U_1^d(\bm{c}),\ldots,U_N^d(\bm{c}))$. Lemma~\ref{pro:property2} shows that the alternative set $\pi_{c_t}$ must contain $c^*_t$. Therefore, Equation~\eqref{equ:c_gamma} guarantees that the single Pareto solution obtained by the evaluation component is the exact optimal solution~\citep{ERASER}, which provides a convergence guarantee for the SDES framework.



\textbf{Prepare Measurement for the Combinatorial Optimization Problem.} Before solving Equation~\eqref{equ:c_gamma}, here introduces novel measurements about $\Gamma(\bm{c})$ to guide the map from $I$-code to $c$-code, cf. Definition~\ref{def:sim_mea} and~\ref{def:dis_mea}. 
Specifically, this section first introduces Definition~\ref{def:sim_mea} measuring the divergence of $\Gamma_i(\bm{c})$ in the single attacker scenario, and generalizes it to Definition~\ref{def:dis_mea} measuring the divergence of $\Gamma_s(\bm{c})$ in the $N$-attackers scenario.

Since $\mathcal{A}_i$ finally only attacks one target $at_i$, the divergence measurement in Definition~\ref{def:sim_mea} focuses on the consistency of $at_i$ between two different $\Gamma_i(\bm{c})$.
\begin{definition}[Divergence between $\Gamma_i(\bm{c})$]\label{def:sim_mea}
For $\mathcal{A}_i$, given two different attack sets $\Gamma^a_i(\bm{c})$ and $\Gamma^b_i(\bm{c})$ and their attacked target $at_i^a$ and $at_i^b$, if $at_i^a=at_i^b$, then $\Gamma^a_i(\bm{c})$ and $\Gamma^b_i(\bm{c})$ are not divergent. Otherwise, $\Gamma^a_i(\bm{c})$ and $\Gamma^b_i(\bm{c})$ are divergent. 
\end{definition}

From Definition~\ref{def:sim_mea}, the semantic of $\Gamma^a_i(\bm{c})$ and $\Gamma^b_i(\bm{c})$ may be the same as long as they both contain the same attacked target, i.e., $at_i^a=at_i^b$. The divergence between $\Gamma_i^a(\bm{c})$ and $\Gamma_i^b(\bm{c})$ is defined as a discrete metric, cf. Equation~\eqref{equ:sim_equ}.


\begin{equation}  
\label{equ:sim_equ}
{K(\Gamma^a_i(\bm{c}),\Gamma^b_i(\bm{c}))}=
\begin{cases}
0, & {\text{if}}\ at^a_i = at^b_i,\\
{1,} & {\text{otherwise.}}
\end{cases}
\end{equation}


\begin{definition}[Divergence between $\Gamma_{s}(\bm{c})$]\label{def:dis_mea}
The divergence of $\Gamma_s^a(\bm{c})$ and $\Gamma_s^b(\bm{c})$ is the sum of $N$ divergences of $\Gamma_i^a(\bm{c})$ and $\Gamma_i^b(\bm{c})$, which is defined as $K_s(\Gamma_s^a(\bm{c}),\Gamma_s^b(\bm{c}))=\sum_{i=1}^NK(\Gamma_i^a(\bm{c}),\Gamma_i^b(\bm{c}))$.
\end{definition}

Definition~\ref{def:dis_mea} extends the Definition~\ref{def:sim_mea} from the single attacker scenario to the multiple attackers scenario. The smaller the value of $K_s(\cdot,\cdot)$, the lower the divergence. In addition, as a metric, Definition~\ref{def:dis_mea} has non-negativity and holds the axioms of symmetry and triangle inequality, but without the identity of indiscernible axiom. The reason is Definition~\ref{def:sim_mea} and~\ref{def:dis_mea} only focus on $at_i$, but not all targets in $\Gamma_i(\bm{c})$.

\textbf{Solve the Combinatorial Optimization Problem.}
Since the divergence measurement $K_s(\cdot,\cdot)$ is capable of measuring the distance between two $\Gamma_s(\bm{c})$, the idea of solving Equation~\eqref{equ:c_gamma} is to find a fixed ideal attack group $\hat{\Gamma}_s$ and to search for $\bm{c}\in\tilde{C}$ corresponding to $\Gamma_s(\bm{c})$ with close to $\hat{\Gamma}_s$. Furthermore, to bypass the potential combinatorial explosion caused by Equation~\eqref{equ:c_gamma}, the greedy algorithm is used to optimize $\bm{c}$ bit by bit. Steps \ding{174}$\sim$\ding{175} in Figure~\ref{fig:L-code} show the visualization of greedy algorithm. Overall, with $K_s(\cdot,\cdot)$ as the metric and greedy algorithm as the optimization method, $c^*_t\in\pi_{c_t}$ is searched to minimize $K_s(\hat{\Gamma}_s, \Gamma_s(\bm{c}))$. Section~\ref{sec:val_ana} would discuss the consistency between the results of the greedy algorithm and the true PF.

Here introduces the construction details of the ideal attack group $\hat{\Gamma}_s$. 
If all limited resources are allocated to defend $\mathcal{A}_i$ as the way in Lemma~\ref{pro:property2}, the resource allocation by Lemma~\ref{pro:property2} is called the ideal solution for $\mathcal{A}_i$ and denoted as $\hat{\bm{c}}^i$. $N$ ideal solutions $(\hat{\bm{c}}^1,\ldots,\hat{\bm{c}}^N)$ can yield $\mathcal{D}$'s ideal payoff vector $(U_1^d(\bm{c}),\ldots,U_N^d(\bm{c}))$, which would not be dominated by any payoff $\bm{U}^d(\bm{c})$, where $\bm{c}\in\tilde{C}$. The attack group corresponding to $N$ ideal solutions is denoted as $\hat{\Gamma}_s=\{\Gamma_1(\hat{\bm{c}}^1),\ldots,\Gamma_N(\hat{\bm{c}}^N)\}$. The ideal attacked target is denoted as $\hat{\bm{at}}=(\hat{at}_1,\ldots,\hat{at})$. The attacked target is the key to calculate divergence measurement $K_s(\cdot,\cdot)$. Note that $\hat{\bm{at}}$ can be calculated in advance and only needs to be calculated once.



\textbf{The Pseudocode of the Evaluation Component.} As shown in Algorithm~\ref{alg:pseudo1}, the process of the evaluation component consists of three steps.
For step 1, necessary information, $\hat{\bm{at}}$, $P$ and $\bm{S}$, is prepared in Line~\ref{code:alg1:7}-\ref{code:alg1:18}. The attraction set $P$ and the counter vector $\bm{S}$ are used for counting the attract relationship. Specifically, $P[t]$ records the set of attackers interested in the target $t$, i.e., $P[t]=\{i|t\in\Gamma_i(\hat{\bm{c}}^i)\}$. $S_i$ counts the number of $\Gamma_i(\hat{\bm{c}}^i)\in\Gamma_s(\hat{\bm{c}}^i)$ containing target $t$. 
For step 2, in Line~\ref{code:alg1:BSM_begin}-\ref{code:alg1:24}, the greedy algorithm named bit-wise optimization (BitOpt) repeats $T$ times to form the optimal resources vector $\bm{c}^*$. Finally, in Line~\ref{code:alg1:fitness}, the step 3 return the payoff $\bm{U}^d(\bm{c}^*)=( U^d_1(\bm{c}^*),\ldots,U^d_N(\bm{c}^*))$ as solution fitness. Note that the sample complexity of the evaluation component is $O(NT)$. The pseudocode of BitOpt is shown in~\ref{app:det_met}.

\begin{algorithm}
\caption{$I$-code Solution Evaluation.}
\label{alg:pseudo1}
\begin{algorithmic}[1]
    \REQUIRE $\bm{I}$, $\bm{A}$, $r\cdot T$. \small\textcolor{Green}{\% $I_i$ represents the size of attack set, i.e., $I_i=\vert\Gamma_i(\hat{\bm{c}}^i)\vert$, where $\hat{\bm{c}}^i$ represents the ideal coverage vector initialized by Lemma~\ref{pro:property2}}. The $i$-th row of $\bm{A}$ represents $\mathcal{A}_i$ target order, i.e., the index vector of targets in decreasing of $U_i^{u,a}(t)$. $r\cdot T$ represents the total resource amount.
    \ENSURE 
    \STATE$\bm {c}^*\gets$ the zero best coverage vector.
    \STATE$\hat{\bm{at}}\gets$ the ideal attacked target vector calculated by $(\bm{c}^1,\ldots,\bm{c}^N)$.
    \STATE$P\gets$ the empty attraction set.
    \STATE$\bm{S}\gets$ the zero counter vector.
    \FOR{$0\leq i < N$}\label{code:alg1:7}
        \STATE$\Gamma_i(\hat{\bm{c}}^i) = \bm{A_i}[:I_i]$.
        \small \textcolor{Green}{\% $[:I_i]$ represents selecting the top $I_i$ targets.}
        \FOR{$t \in \Gamma_{i}(\hat{\bm{c}}^i)$}
            \STATE$S_t = S_t + 1$.
            $P[t] = P[t]\cup \{i\}$.
        \ENDFOR\label{code:alg1:12}
        
    
    
    
    \ENDFOR\label{code:alg1:18}
    \FOR{$t$, $count_1$ in enumerate($\bm{S}$)}\label{code:alg1:BSM_begin}
        \STATE$c^*_t=BitOpt(count_1, P, t, \bm{I} ,\bm{A}, \hat{\bm{at}})$, cf. Algorithm~\ref{alg:pseudo2} in~\ref{app:fra_pseu_BSM}.

        \IF{$\Vert\bm{c}^*\Vert_1 > r\cdot T$}
            \STATE\textbf{return} $infeasible$.
        \ENDIF
    \ENDFOR\label{code:alg1:24}
    \STATE\textbf{return} $\bm{U}^d(\bm{c}^*)$. 
    \small \textcolor{Green}{\% $\mathcal{D}$'s payoff for defending all $\mathcal{A}_i$s.}\label{code:alg1:fitness}
\end{algorithmic}
\end{algorithm}

\subsection{Refinement}
After the iterations of above components, SDES calls the MIN-COV~\citep{MOSG} subroutine once to refine solutions. MIN-COV is a lightweight adjustment part, and all comparison methods mentioned in~\citep{MOSG,extended-MOSG} use MIN-COV. For $\mathcal{A}_i$, MIN-COV tries to find a better solution that yields no decrease in payoff and consumes fewer resources. Specifically, MIN-COV accepts the current optimal resource allocation scheme $\bm{c}$ and the lower bound constraint $\bm{b}$, and attempts to find $\bm{c}'$ where $\Vert\bm{c}'\Vert_1\leq\Vert\bm{c}\Vert_1$ and $U_i^a(\bm{c}')\geq\bm{b}$. Since MIN-COV is designed for the $\epsilon$-constraint framework, it cannot directly be used, because $\bm{b}$ represents the lower bound for incremental updates that do not exist in many-objective EAs. Therefore, the refinement component adopts the following adjustments. $\bm{c}$ represents $\bm{c}^*$ found in solution evaluation, while $\bm{b}$ represents the solution’s fitness corresponding to $\bm{c}$, i.e., $\bm{b}=U_i^a(\bm{c})$. The meaning of MIN-COV for many-objective EAs is to find a new non-dominant solution with fewer resources. Then redistribute the remaining resources to try to find further payoff growth.

\subsection{SDES Framework and Sample Complexity Analysis}\label{sec:fra_com}

\begin{algorithm}
\caption{The SDES Framework.}
\label{alg:who_fra}
\begin{algorithmic}[1]
    \REQUIRE $T$-length continuous solutions ($I$-code) initialization.
    \ENSURE \quad
    
    \STATE Discretizing $T$-length high-dimensional continuous solutions ($I$-code) into $N$-length low-dimensional discrete solutions ($c$-code), $N\ll T$.
    
    
    \WHILE{Terminal criteria not satisfy}

        \STATE Optimizing $I$-code by NSGA-III.

        \STATE Mapping $I$-code to $c$-code by greedy algorithm (preparation for evaluating $I$-code).

        \STATE Evaluating $c$-code and returning $c$-code's fitness as $I$-code's fitness.
        


    \ENDWHILE
    

    \STATE Refining $c$-code.

    \STATE\textbf{return} the refined high quality $c$-code.

\end{algorithmic}
\end{algorithm}

The whole SDES framework is shown as Algorithm~\ref{alg:who_fra}. SDES is faster than other comparison method, e.g., the SOTA $\epsilon$-constraint method ORIGAMI-M. The efficiency of SDES reflects in two points during iterations: the complexity of the evaluation component and the number of the solutions to be refined in the refinement component. 

As for the complexity of the evaluation component, ORIGAMI-M finds a solution satisfying the bound constraint in $O(NT^3)$, where $O(T^2)$ is used to accelerate the convergence of the bound constraint~\cite{MOSG}. As for the sample complexity of SDES, SDES transform the MOSG problem into the combinatorial optimization problem. After transformation, SDES shrinks the solution domain into a feasible range, so that the number of inefficient feasibility constraint detections is reduced. As shown in Algorithm~\ref{alg:pseudo1}, the complexity of the evaluation component is $O(NT)$, which is a linear complexity if either the number of targets or attackers is fixed. 
As for the number of the solutions to be refined in the refinement component, SDES calls the refinement component in constant times, while ORIGAMI-M call it exponential times with respect to the number of attackers~\citep{MOSG, extended-MOSG}.

\section{Theoretical Analysis}
\label{sec:val_ana}
\subsection{Optimization Consistency}
\label{sec:rat_sim}
This section proves the consistency of the combinatorial optimization and the proposed bit-wise optimization in the restoration component. It involves (i) Optimizing the similairity of ideal solution $\{\bm{\hat{c}}^{1},\ldots,\bm{\hat{c}}^{N}\}$ and feasible solution $\tilde{\bm{c}}$ can also find the solutions on the true PF of the MOSG problems.
(ii) The lower the divergence of $\Gamma_s$, the closer payoff of the solutions. The proof detail is shown in Appendix~\ref{app:proof:3}.

\begin{restatable}{lemma}{lemmaTheoAnal}\label{lem:lemma1}
If $\Gamma$ is constructed from the optimal expansion rule in Lemma~\ref{pro:property2} and $\Gamma(\bm{c})\subseteq\Gamma(\bm{c}')$, then $U^d(\bm{c})\leq U^d(\bm{c}')$.
\end{restatable}

\begin{restatable}{lemma}{obsalgorithmfir}\label{obs:algorithm1}
The ideal solution $\{\bm{\hat{c}}^{1},\ldots,\bm{\hat{c}}^{N}\}$ satisfies $U_i^d(\bm{\tilde{c}})\leq U_i^d(\bm{\hat{c}}^{i}), \forall i\in[N]$ and $(U_i^d(\bm{\hat{c}}^{1}),\ldots,\\U_i^d(\bm{\hat{c}}^{N}))$ cannot be dominated by any solution on the PF. 
\end{restatable}

\begin{remark}
    Lemma~\ref{lem:lemma1} shows that $\Gamma$ must not be worse than its subset.  Lemma~\ref{obs:algorithm1} indicates that (i) the ideal solution is an upper bound on $\forall\tilde{\bm{c}}\in\Theta$ and it is non-inferior to the solution on the PF, and (ii) the purpose of ideal solution construction is to replace finding the true PF of the MOSG problem.
\end{remark}




\begin{restatable}[Optimization Consistency]{theorem}{obsalgorithmthi}\label{obs:algorithm3}
If $\tilde{\Gamma}_{s}$ is similar to $\hat{\Gamma}_{s}$, i.e., $K_s(\tilde{\Gamma}_{s},\hat{\Gamma}_{s})=0$, then $U^d(\bm{\tilde{c}})$ is on the PF and $U_i^d(\bm{\tilde{c}})=U_i^d(\bm{\hat{c}}^{i})$, $\forall i\in[N]$. 
\end{restatable}

\begin{remark}
    Theorem~\ref{obs:algorithm3} shows that optimizing $K_s(\cdot,\cdot)$ between the ideal solution and feasible solution is consistent to directly optimizing MOSG problems via bit-wise optimization. Meanwhile, Lemma~\ref{lem:lemma1} has been supplemented by Theorem~\ref{obs:algorithm3}.
\end{remark}


\subsection{Convergence Guarantee of Bit-wise Optimization}
\label{sec:rat_bit}
Bit-wise optimization is a greedy algorithm, and its intuition comes from the following SGs property. $\mathcal{D}$ allocates $c_t, t\in[T]$ independently, as long as $\sum_{t=1}^T{c_t}\leq r\cdot T$.
That is, the resource allocation on each target $c_t$ does not affect each other. Therefore, convergence guarantee is induced under Assumption~\ref{asp:BitOpt}.
The proof is shown in Appendix~\ref{app:fra_pseu_BSM}.

\begin{restatable}[Split Assumption]{assumption}
{splitassumption}\label{asp:BitOpt}
In the $t$-th iteration, $\pi_{c_t}^j$ is selected, $t\in\Gamma_i$ for all $i\in[N]$ satisfies $\pi_{c_t}^i\geq\pi_{c_t}^j$.
\end{restatable}

\begin{restatable}[Convergence Guarantee]{theorem}{theoremRatBitbybit}\label{obs:val_ana1}
Under Assumption~\ref{asp:BitOpt}, the result of bit-wise optimization $\tilde{\bm{c}}$ converges to the PF, i.e., $U_i^d(\bm{\tilde{c}})=U_i^d(\bm{\hat{c}}^{i}),\forall i\in[N]$. 
\end{restatable}
\begin{remark}
    Theorem~\ref{obs:val_ana1} gives a solution convergence guarantee for Bit-wise Optimization under Assumption~\ref{asp:BitOpt}, which means that bit-wise optimization select the smallest alternative when there are multiple $\mathcal{A}$ attracted by $t$.
\end{remark}

\section{Experiments}
\label{sec:exp}
\subsection{Experiment Setup}
\label{sec:secCon_eva}
\textbf{Experiment Environment.} The experimental environment is consistent with~\citep{MOSG,extended-MOSG}, including (i) the randomly-generated security games, i.e., $\mathcal{D}$'s and $\mathcal{A}$'s payoff obey the uniform integral distribution of [1, 10], while the defender and attacker payoffs obey the uniform integral distribution of [-10, -1], (ii) the problem scale is divided into three dimensions: the attacker number $N$, the target needed to be protected number $T$, and the resource ratio $r$, satisfying the following constraint: $N\geq3$, $T=[25, 50, 75, 100, 200, 400, 600, 800, 1000]$, $r=0.2$, the resource equals $r\cdot T$, 
where $r$ is generally much less than $1$ due to the limited resources,
and (iii) the maximum cap of the runtime $M$ is set to $30$ minutes (mins) for clock time. The programming language is Python, and the hardware conditions are AMD Ryzen Threadripper 2990WX 32-Core Processor, 3600 CPU MHz.

\textbf{Comparison Algorithms.} All MOSG problems below are represented as minimization tasks. This paper compares 6 algorithms simultaneously, including 4 algorithms based on $\epsilon$-constraint framework and two general many-objective EAs.
The $\epsilon$-constraint algorithms include ORIGAMI-M and ORIGAMI-A~\citep{MOSG}, ORIGAMI-M-BS and DIRECT-MIN-COV~\citep{extended-MOSG}.
All $\epsilon$-constraint algorithms use step $\epsilon=1$. ORIGAMI-A uses the threshold $\alpha=0.001$ for binary search termination. 
The general many-objective EAs include $c$-code EA with continuous solution representation, $I$-code EA with integral solution representation. Furthermore, in the large-scale scenario, even SOTA many-objective EAs cannot converge well due to Lemma~\ref{pro:property1.2} and~\ref{pro:property1.1}. Therefore, we do not consider those algorithms for comparison.

\textbf{Implementation Details of many-objective EAs in SDES.} Many-objective EAs, the most common one in solving MaOPs, are used as the solution method in this paper. The solution method can also be replaced by other zeroth-order optimization methods. More specifically, we use NSGA-III~\citep{NSGA-III}, the reference set based EA, for the following important reasons. (i) The reference set based method can easily customize the reference direction set, so many-objective EAs can better adapt to different problem scales. (ii) NSGA-III is enough for low-dimensional discrete space optimization after discretization. Therefore, we choose a classical many-objective EA as the solution method. The termination criterion $max\_{}gen$ is fixed to $300$, and the population size $pop\_{}size$ is fixed to $400$, except $N=3$ case, i.e., $max\_{}gen=pop\_{}size=50$, because this problem scale is simple for NSGA-III. The additional sensitivity analysis in Appendix~\ref{app:more_sen_ana} shows that SDES is insensitive to $pop\_{}size$.
The implementation of NSGA-III uses an open-source framework: multi-objective optimization in Python (PyMOO)~\citep{pymoo}. More implementation details about NSGA-III can be found in Appendix~\ref{app:NSGA-III}.
All experiments are repeated 30 times. Observe that NSGA-III uses Riesz~\citep{ref-dirs-energy} instead of the traditional das-Dennis or multi-layer Das-Dennis method~\citep{das-Dennis}. The Das-Dennis-based approach is greatly affected by the increasing number of attackers, while Riesz can generate a customized number of well-spaced reference direction sets by energy-based approaches.

\textbf{Quality Measurement.}
For MaOPs, an ideal algorithm can obtain an approximate solution set satisfying that all solutions in the approximate solution set are as convergent as possible and as diverse as possible to the PF~\citep{MOEA-survey}. The quality measurements of convergence and diversity are

Hypervolume Metric (\textit{HV}).
The hypervolume indicator~\citep{HV}, describes the volume of the space dominated by the Pareto set approximation to the reference point in the objective space and can be expressed as

\begin{equation}
\label{equ:equHV}
    HV(A,\bm{r})=\mathcal{L}\left(\cup_{\bm{a}\in A}\left\{\bm{b}\mid\bm{a}\succ \bm{b}\succ \bm{r}\right\}\right)\,,
\end{equation}
where $A$ represents the Pareto set approximation to be evaluated, $\bm{r}$ is a reference point prepared in advance, $\mathcal{L}(\cdot)$ represents the Lebesgue measure of a set $\mathcal{L}(B)=\int_{\bm{b}\in B}\bm{1}_B(\bm{b})\,\mathrm{d}\bm{b}$. $\bm{1}_B$ is the characteristic function of $B$ that $\bm{1}_B(\bm{b})=1$ if $\bm{b}\in B$, otherwise $\bm{1}_B(\bm{b})=0$, and $\bm{a}\succ\bm{b}$ denotes $\bm{a}$ dominates $\bm{b}$.

Modified Inverted Generational Distance (IGD$^+$).
\textit{IGD}$^+$~\citep{IGD+} is the improved version of
$IGD$.
$IGD$ measures the Euclidean distance from the reference point set $Z$ to the Pareto set approximation $A$. 
\textit{IGD}$^+$ is

\begin{equation}
    IGD^+(A) = \frac{1}{\vert Z\vert}\left(\sum_{i=1}^{\vert Z\vert}{d^+_i}^2\right)^{1/2}\,,
\end{equation}
where $d^+_i=\max\{a_i-z_i,0\}$ and $a_i$, $z_i$ are the $i$-th component of $\bm{a}$, $\bm{z}$. $\bm{d^+}=(d^+_1,\ldots,d^+_{\vert Z\vert})$ calculates only the positive component of contributes. 

\subsection{Scalability}
The scalability of all five algorithms across all problem scales within $M=30$ mins is shown in Table~\ref{tab:runtime:all}.
Table~\ref{tab:runtime:all} reveals that only ORIGAMI-M and ORIGAMI-M-BS can complete the task of $T=1000, N=3$, but SDES is less affected by the increasing number of $T$. Furthermore, with the increasing number of $N$, The maximum number of $T$ of comparison algorithms drops rapidly when $N$ increases, while the maximum number of $T$ of SDES decreases slowly. The comparison methods can only complete the large-scale tasks of $T=25$, but SDES can still work on the problem with large $T$. Notably, SDES can solve $N=20$ problems with ease.

\begin{table}[htbp]
\caption{The scalability of SDES and comparison algorithms. The maximum number of targets of each algorithm is shown when the number of attackers ranges from 3 to 20. The symbol ``--'' means that the algorithm cannot complete the MOSG problem under the corresponding number of attackers and targets within maximum time $M=30$ mins.}
\label{tab:runtime:all}
    \centering
    \begin{tabular}{c|c|c|c|c|c|c|c|c}
    \toprule
    \diagbox{Method}{$\#$Attacker} & 3& 4& 5& 6& 7& 8& 9$\sim$18& 19$\sim$20\\  \midrule
    SDES & 1000 & 1000 & 400 & 200 & 200 & 200 & 200 & 100\\
    DIRECT-MIN-COV & 600 & 200 & 100 & 100 & 50 & 25 & --& --\\
    ORIGAMI-A & 200 & 100 & 75 & 25 & -- & 25 & -- & --\\
    ORIGAMI-M-BS & 1000 & 200 & 100 & 75 & 50 & -- & -- & --\\
    ORIGAMI-M & 1000 & 400 & 100 & -- & -- & -- & -- & --\\
    \toprule
    \end{tabular}
\end{table}


\subsection{Time Efficiency}
\label{sec:Runtime}
We show that SDES not only has the ability to accomplish large-scale MOSGs, but also is more time-efficient than comparison methods. The time efficiency experiments (only results completed within $M=30$ mins are displayed) include (i) Fix $N=3$ and scale up $T=200\sim 1000$. We set $T\geq200$ since all algorithms can solve $T<200$ problems quickly, cf. Figure~\ref{fig:runtime:all}(a). (ii) Fix $T=25$ and scale up $N=3\sim9$, since no comparison methods can complete $N>9$ task even $T$ only equals $25$, cf. Figure~\ref{fig:runtime:all}(b).


\begin{figure}[!t]
\centering
\begin{minipage}[l]{\columnwidth}\centering
\includegraphics[width=0.49\textwidth]{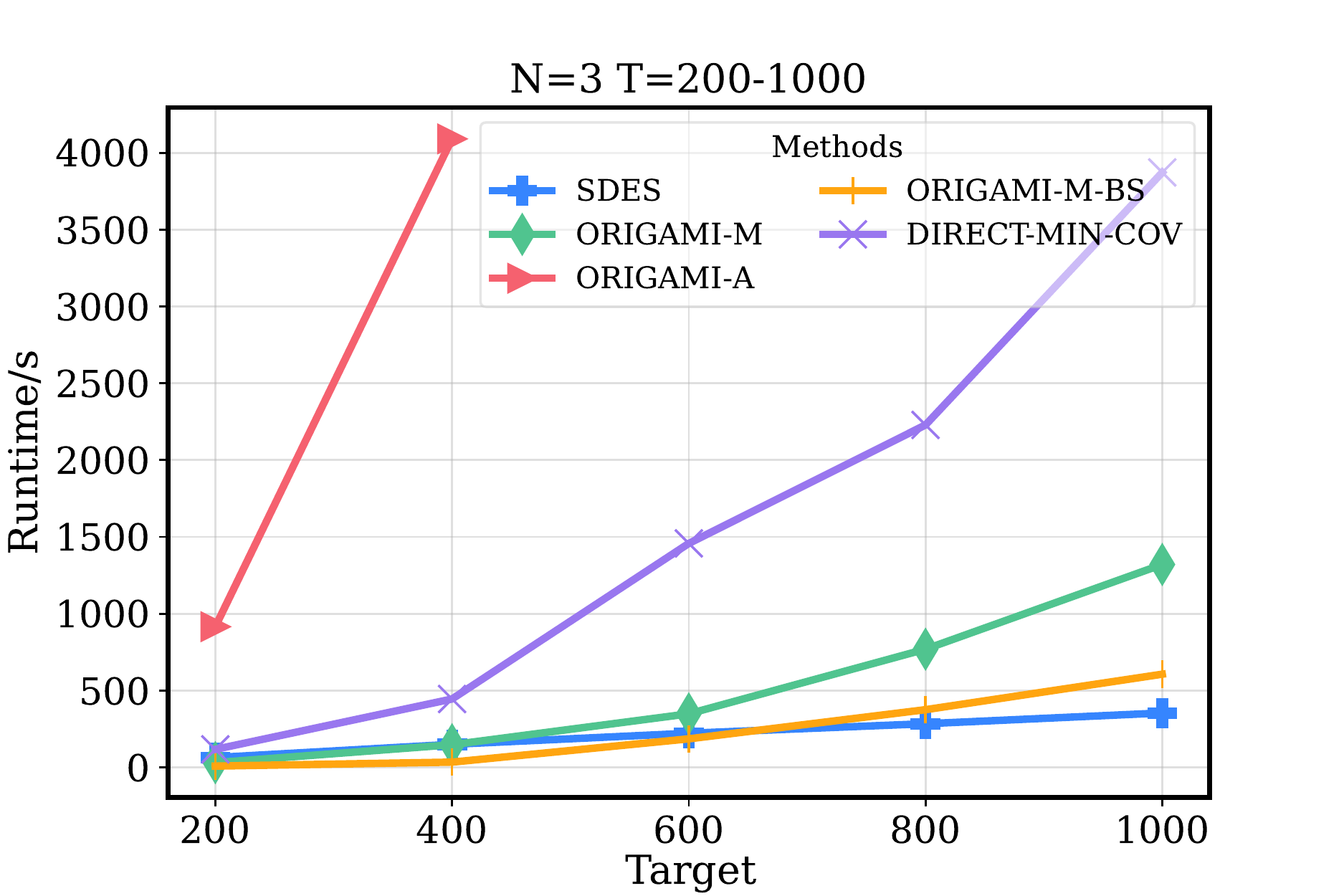}
\includegraphics[width=0.49\textwidth]{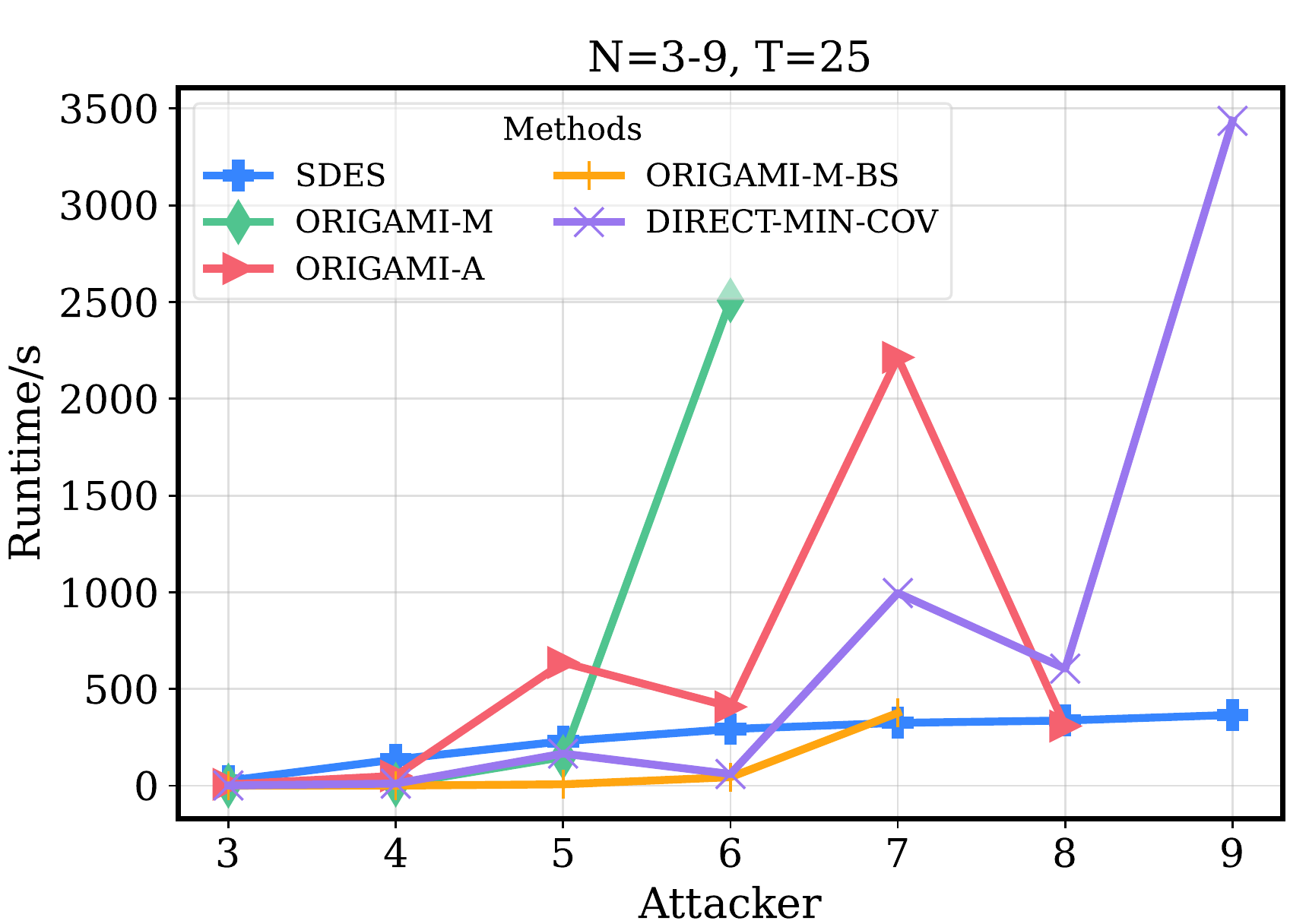}\\
\quad\qquad(a) Scaling up targets.\qquad\qquad\qquad
(b) Scaling up objectives.
\caption{The runtime indicator of SDES and comparison algorithms.}
\label{fig:runtime:all}
\end{minipage}
\end{figure}

Figure~\ref{fig:runtime:all}(a) and (b) show that SDES increases slowly. Furthermore, Figure~\ref{fig:runtime:all}(b) shows that all comparison methods fail as the number of attackers increases. Although DIRECT-MIN-COV is the most efficient method which can complete $N=8$, $T=25$ task in $M=30$ mins, DIRECT-MIN-COV is one of the worst methods with respect to both \textit{HV} and \textit{IGD}$^+$ performance indicators, cf. Figure~\ref{fig:Rankbar}. Meanwhile, although the runtime of ORIGAMI-M-BS also increases slowly in medium-scale MOSGs, it fails when $N>7$. SDES can easily accomplish $N=8$, $T=25$ task, and the performance of SDES is satisfactory.
The runtime of SDES varies steadily with the number of attackers, while the runtime of other methods mainly depends on the problem-solving difficulty. For example, when $N=6$ or $N=9$, the runtime of almost all comparison methods is reduced since those problems are less difficult than $N=5$ or $N=7$ problem. 

\begin{figure}[h!t]
\centering
\includegraphics[width=0.5\linewidth]{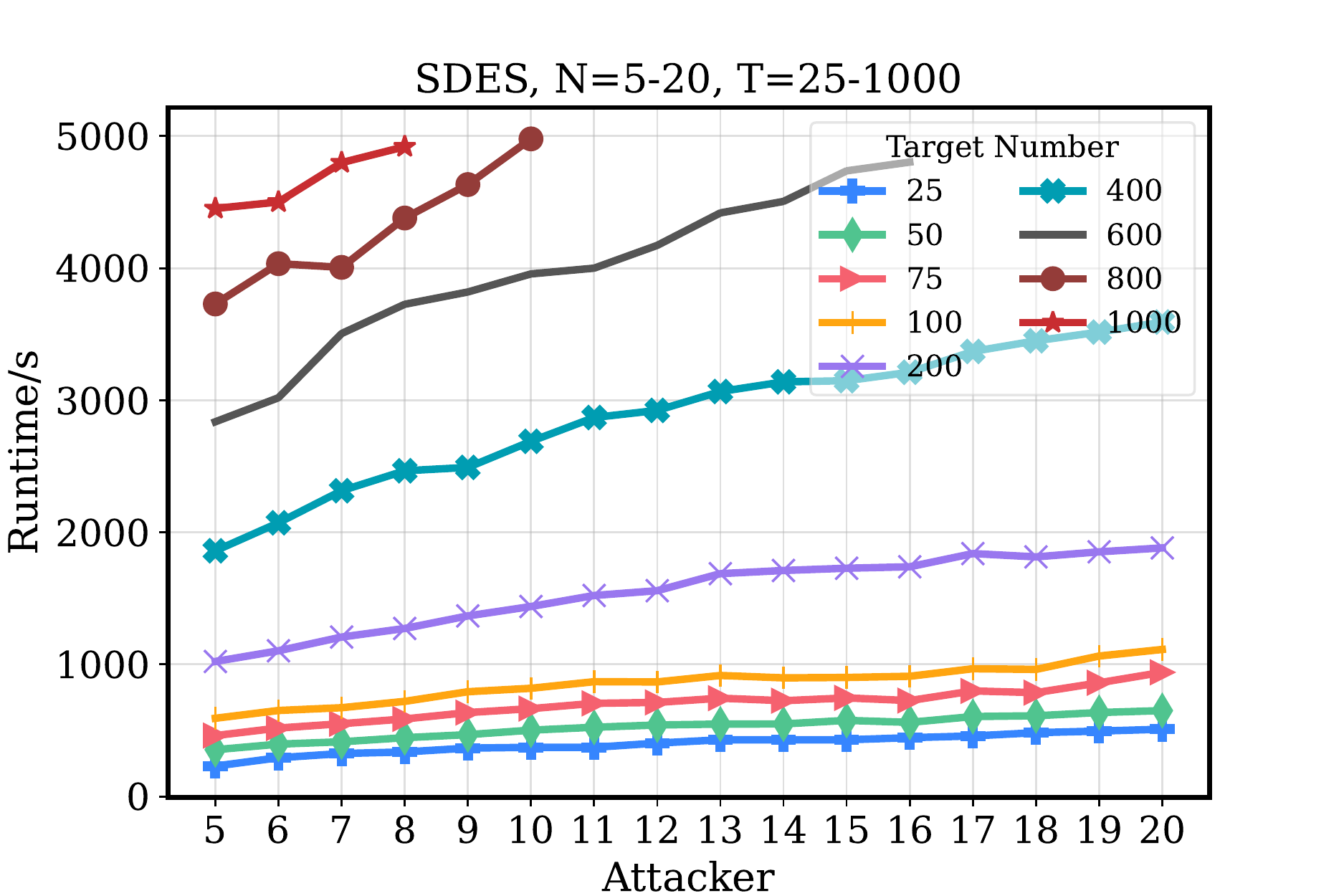}
\caption{The runtime of SDES for larger problem scales. The runtime of SDES increases linearly with the increasing number of attackers and targets when one of them is fixed.
}
\label{fig:Runtime:SDESF:larger}
\end{figure}

Figure~\ref{fig:runtime:all}(b) presents that all comparison algorithms fail when $N>8$, $T>25$, but SDES still works. Figure~\ref{fig:Runtime:SDESF:larger} shows the scalability boundary of SDES. In Figure~\ref{fig:Runtime:SDESF:larger}, SDES can handle $N=20$, $T=100$ problems within $M=30$ mins. 
In addition, the runtime of SDES increases linearly with the increasing number of targets when the number of attackers is fixed. Also, the same goes for attackers.

\subsection{Effectiveness}
\label{sec:effe}
\subsubsection{Convergence Speed and Effect}
\label{sec:effe:part1}
Lemma~\ref{pro:property1.2} and~\ref{pro:property1.1} have shown that $U_i^d(\bm{c})$ is $T$-dimensional step function that heuristic and gradient-based algorithms are difficult to directly optimize on the $C$ space, cf. Figure~\ref{fig:Ua_Ud} for an illustration.
To verify the superiority of SDES over the naive $c$-code EA and other comparison algorithms, see Figure~\ref{fig:convergence_curve}, we plot the convergence of \textit{IGD}$^+$ performance indicator with the number of iterations (\textit{IGD}$^+$ instead of \textit{HV} is selected due to its high calculation cost).

\begin{figure}[h!t]
\centering
\includegraphics[width=0.5\linewidth]{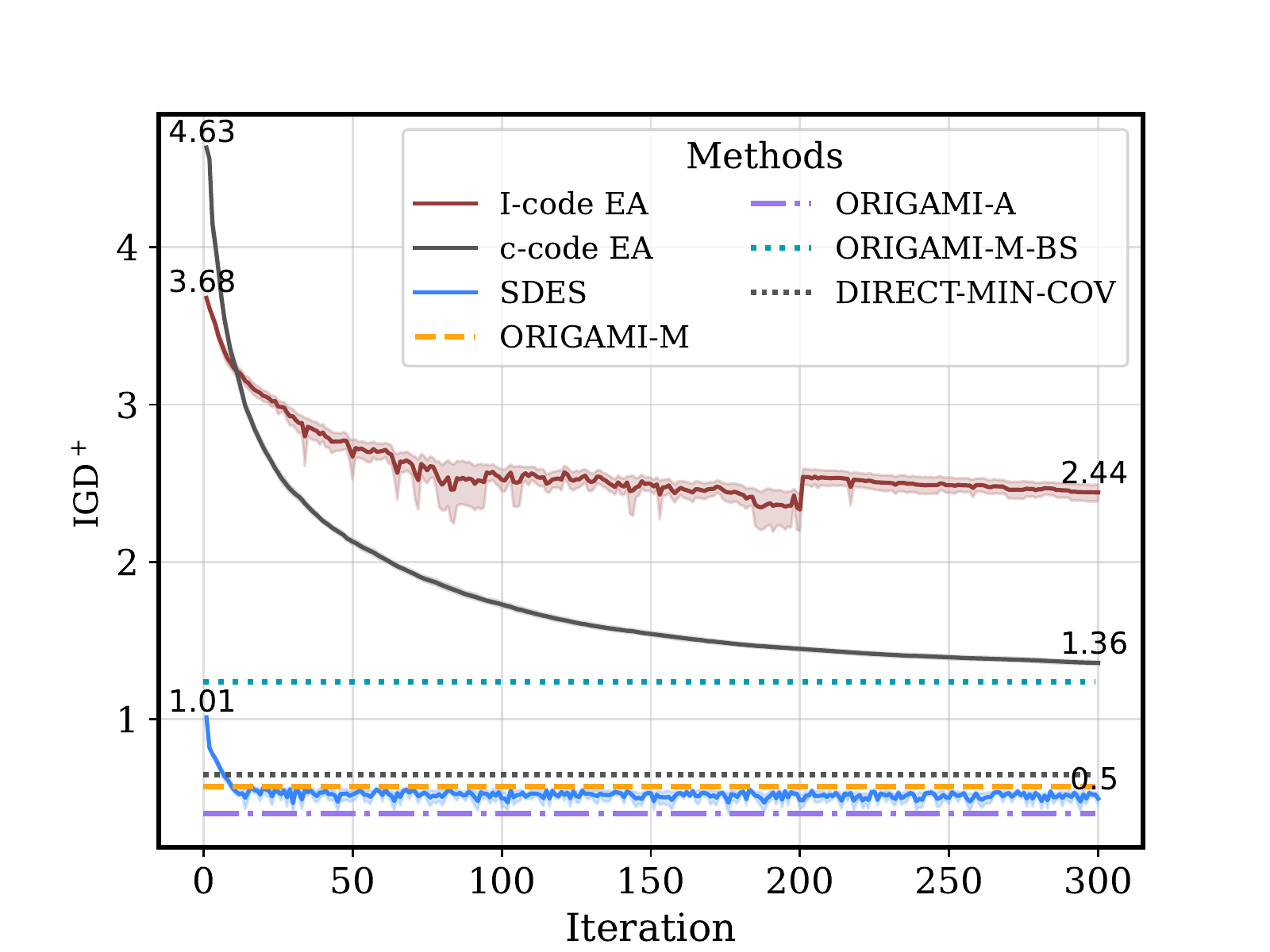}
\caption{The convergence curve. SDES converges fast, and it is comparable with other traditional methods. Since the comparison algorithms are based on the $\epsilon$-constraint framework and \textit{IGD}$^+$ score is returned after fully running, the comparison algorithms are plotted as the dotted line parallel to the x-axis.
}
\label{fig:convergence_curve}
\end{figure}

As shown in Figure~\ref{fig:convergence_curve}, the performance of SDES in the initialization stage outperforms the converged performance of $I$-code and $c$-code EAs, and its convergence speed is faster. SDES finally obtains comparable results to the SOTA methods. Meanwhile, the exact algorithm ORIGAMI~\citep{ERASER} is embedded in the iteration of SDES. Therefore, only initialization can guarantee a well performance with respect to \textit{IGD}$^+$ and \textit{HV} indicators. 

\subsubsection{Multi-objective Performance Indicators of All Scales}
\label{sec:effe:part2}
The multi-objective performance indicators \textit{HV}, \textit{IGD}$^+$ can measure convergence and diversity of the approximate PF. 
From Table~\ref{tab:runtime:all}, there are a total of 26 MOSG problems that at least one comparison algorithm can complete.

\begin{figure}[h!t]
  \centering
  \includegraphics[width=0.5\linewidth]{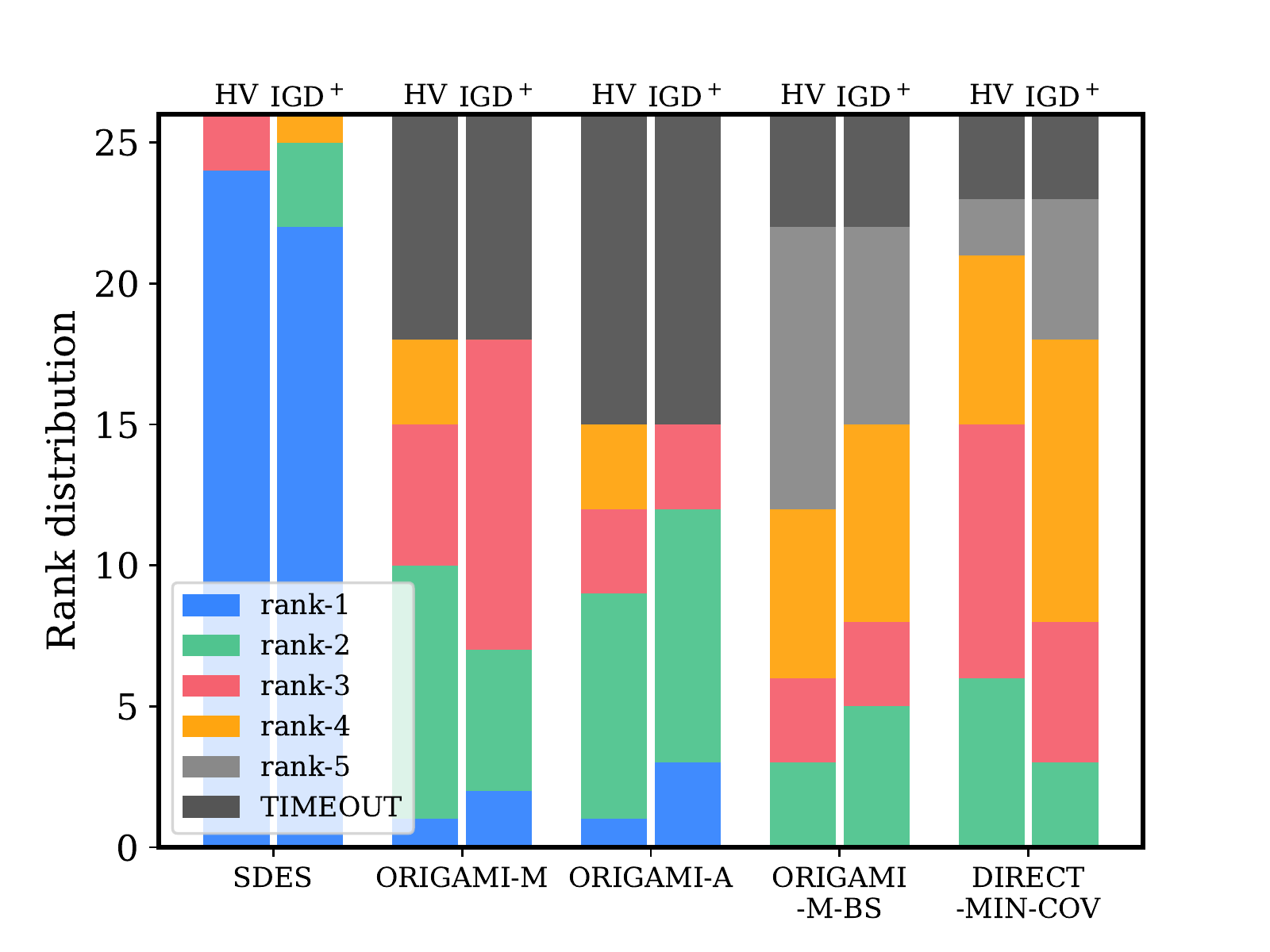}
  \caption{The rank distribution of 5 methods in 26 MOSG problems. SDES achieves most of rank-1 in both \textit{HV} (left bar) and \textit{IGD}$^+$ (right bar) indicator experiments. The SOTA methods ORIGAMI-M and ORIGAMI-A time out in many MOSG problems, while other efficient methods ORIGAMI-N-BS and DIRECT-MIN-COV are not effective. }
    \label{fig:Rankbar}
\end{figure}

SDES has a clear advantage in solution quality. Figure~\ref{fig:Rankbar} shows summarization results about all methods over 26 MOSG problems. In addition, Appendix~\ref{app:vis_ana} shows the excellent performance of SDES on each objective by visualization. Among 26 problems, SDES achieves 24 rank-1 in \textit{HV} indicator, and 22 rank-1 in \textit{IGD}$^+$ indicator. The score and rank result details are put in Appendix~\ref{app:more_eff_exp}. 
Although SOTA algorithms like ORIGAMI-M can provide well-convergent solutions, many-objective optimization task also requires well-spaced (good diversity and distribution) solutions, which makes these methods poor performance in the multi-objective indicator. 

\subsection{Ablation Studies}
\label{sec:Ablation}

The ablation study is conducted to analyze the effects of the three key parts discretization, restoration and refinement. Five scenarios are designed as Table~\ref{tab:ablation:our_method}. 
For discretization, \ding{53} means many-objective EA searches in the high-dimensional continuous space $c$-code, while $\checkmark$ means many-objective EA searches in the low-dimensional discrete space $I$-code. For restoration, \ding{53} means solutions are restored in random, while $\checkmark$ means solutions are restored by solution divergence. For refinement, \ding{53} and $\checkmark$ mark whether the PF approximation found by many-objective EA is refined or not.
The representative problem scale is $N=5, T=50$.
The reason of $T=50$ is that SDES and SOTA methods behave similarly, cf. Figure~\ref{fig:convergence_curve}. The reason of $N=5$ is that all methods perform well. 

\begin{table}[!tbp]
\caption{The ablation result of SDES in $N=5, T=50$ MOSG problem.}
\label{tab:ablation:our_method}
    \centering
    \begin{tabular}{c|ccc|ll}
    \toprule
        Num & Dis & Res & Ref & HV/Rank & IGD$^+$/Rank\\ \midrule
        1 & \ding{53} &\ding{53}&\ding{53}&3.68e5{\scriptsize $\pm$4.04e3}/4 & 1.31{\scriptsize $\pm$.02}/3 \\
        2 & $\checkmark$&\ding{53}&\ding{53}&\textbf{3.41e5}{\scriptsize $\pm$9.75e3}/5 & \textbf{2.45}{\scriptsize$\pm$.11}/5 \\
        3 & $\checkmark$&\ding{53}&$\checkmark$&4.02e5{\scriptsize $\pm$4.49e4}/3 & 2.39{\scriptsize$\pm$.15}/4 \\ 
        4 & $\checkmark$&$\checkmark$&\ding{53}&\textbf{5.58e5}{\scriptsize $\pm$1.73e3}/2 & \textbf{0.57}{\scriptsize $\pm$.01}/2 \\
        5 & $\checkmark$&$\checkmark$&$\checkmark$&5.61e5{\scriptsize $\pm$1.23e5}/1 & 0.50{\scriptsize $\pm$.06}/1 \\ \bottomrule
    \end{tabular}
\end{table}

From the ablation results, the following key questions are answered:

(i) Can MOSGs be solved simply by naive many-objective EA or discretization? The term ``naive" means the many-objective EA only uses the common EA operator but without discretization, restoration, and refinement components proposed in our paper. Num 1 and Num 2 indicate that both methods fail to converge and are inferior to SDES and all comparison algorithms. Furthermore, The naive EA with discretization is even inferior to the naive EA. The reason why naive EA fails is the $O(N^T)$ sample complexity in combinatorial optimization, and the reason why naive discretization fails is the need of the restoration component after the discretization component. In general, neither naive many-objective EA nor discretization can handle large-scale MOSGs well. 
    

(ii) How much improvement does the restoration component bring? The results from Num 2 to Num 4 in black show that the restoration component can greatly improve solutions' quality. 
If Assumption~\ref{asp:BitOpt} is not satisfied, the well-converged PF may not be found, where the refinement component is needed.

(iii) How much improvement does the refinement component bring? The results from Num 2 to Num 3 or from Num 4 to Num 5 show that clear gains can be obtained from the lightweight refinement. Observe that the refinement component works well and that there is a lot of room for the accuracy improvement of bit-wise optimization in the evaluation component.

\section{Conclusion}
\label{sec:conclusion}
This paper proposes the first linear sample complexity EA-based framework SDES for MOSGs if either the number of targets or attackers is fixed, which simultaneously extends targets and heterogeneous attackers to large-scale scenarios.
The contribution includes proposing a high-dimensional evaluation and low-dimensional optimization framework, theoretically disclosing the optimization consistency and convergence guarantee, and empirically verifying the superiority of SDES with respect to scalablity, time efficiency, and effectiveness. 
%
The limitation of SDES is that it greedily solves combinatorial optimization problems bit by bit, whose convergence assumption may become difficult to meet with the increase of the number of targets. Therefore, as traditional methods~\citep{MOSG,extended-MOSG}, SDES provides refinement components to ensure convergence. An important future work is to design a framework with theoretical guarantees for the monotonicity of the quality of the solution. 

\section*{Acknowledgement}
This work is supported by the Scientific and Technological Innovation 2030 Major Projects (No. 2018AAA0100902), the National Natural Science Foundation of China (No. 62106076), and the Natural Science Foundation of Shanghai (No. 21ZR1420300). The data and code of algorithms are available at \url{https://github.com/1589864500/SDES}.

\bibliographystyle{named}
\bibliography{mybibliography}

\newpage
\section*{Appendix} 
The appendix consists of three parts: (i) additional details of the methodology, (ii) proof details of theoretical analysis and (iii) additional experimental details and results. In the first part, Section~\ref{app:notation} first summarizes the variables, parameters, and symbols in this paper. Then, Section~\ref{app:protra} introduces the significance of the discretization and optimization components of our framework, mentioned in Section~\ref{sec:dis} and~\ref{sec:opt}. Furthermore, Section~\ref{app:IatC} introduces the pseudocode of ideal attacked target calculation IATC mentioned in Algorithm~\ref{alg:pseudo1}. Finally, Section~\ref{app:fra_pseu_BSM} introduces another pseudocode of the boolean scoring mechanism BSM mentioned in Algorithm~\ref{alg:pseudo1}. In the second part, we present the proof details about Lemmas and Theorems mentioned in Section~\ref{sec:met} and~\ref{sec:val_ana}. In the third part, firstly, Section~\ref{app:expset} introduces more experimental setup details, including hyperparameter setup, implementation setup of NSGA-III in the optimization component of SDES, solution initialization of SDES. Then, Section~\ref{app:more_eff_exp} exhibits the comprehensive results of the performance of all methods mentioned in Seciton~\ref{sec:effe:part2}. In addition, Section~\ref{app:more_time_effi} analyzes the effect of population size on the runtime of SDES. Seciton~\ref{app:more_sen_ana} displays additional sensitivity analysis of population size in SDES. Finally, combined with the performance result details, Section~\ref{app:vis_ana} further analyzes the characteristics of SDES and SOTA methods by visualization.

\section{Additional Details of Methodology}
\label{app:det_met}
\subsection{Notations}
The definitions of variables, parameters, and symbols involved in this paper is listed in Table~\ref{app:tab:notion}. 
\label{app:notation}

\begin{table*}[]
\caption{The notions involved in this paper.}
\label{app:tab:notion}
\begin{adjustbox}{width=\textwidth}
    \centering
    \begin{tabular}{c|l}
    \toprule
        Notation & Definition\\ \midrule
        {$\succeq, \succ$} & $\bm{c}\succeq\bm{c}'$ means that $\bm{c}$ weak dominates $\bm{c}'$. $\bm{c}\succ\bm{c}'$ means that $\bm{c}$ dominates $\bm{c}'$. \\
        {$\setminus$} & The set difference symbol. \\
        {$\Vert\cdot\Vert_1$} & The L1 norm of a vector. \\
        {$[\cdot]$} & $[N]=\{1,\ldots,N\}$ for a positive integer $N$. \\
        {$\mathcal{A}_i$, $\mathcal{D}$} & The $i$-th attacker and defender in the multi-objective security games. \\
        {$N$, $T$} & The number of attackers and targets. \\ 
        {$r$} & {The resource ratio. The total resource is $r\cdot  T$.} \\
        $t$ & The $t$-th target in all $T$ targets.  \\
        {$\bm{c}$, $\tilde{\bm{c}}$, $\hat{\bm{c}}$, $\bm{c}^*$} & The coverage vector (solution), the feasible solution, the ideal solution, \\
        {} & {and the optimal solution, respectively.}\\
        {$\bm{v}$} & A zero coverage vector except the $t'$-th component. \\
        {$C$, $C^*$, $\tilde{C}$} & The whole space of all solutions, the space of all optimal solutions, and    \\
        {} & {the space of all feasible solutions, respectively.}\\
        {$\bm{F}(\bm{c})$} & {The fitness of solution $\bm{c}$.}  \\
        {$\bm{a}_i$} & The attack vector of $\mathcal{A}_i$.  \\
        {$at_i$} & The target being attacked by $\mathcal{A}_i$.  \\
        $U_i^{u,a}(t)$, $U_i^{c,a}(t)$ & $\mathcal{A}_i$'s payoff on target $t$ if $t$ is \textit{uncovered} and \textit{covered}.\\
        $U_i^{u,d}(t)$, $U_i^{c,d}(t)$ & $\mathcal{D}$'s payoff for target $t$ if $t$ is \textit{uncovered} and \textit{covered}.\\
        $U_i^a(c_t)$, $U_i^d(c_t)$ & $\mathcal{A}_i$'s and $\mathcal{D}$'s payoff for one target targets with respect to the resources $c_t$.\\
        $U_i^a(\bm{c}, \bm{a}_i)$,, $U_i^d(\bm{c}, \bm{a}_i)$ & $\mathcal{A}_i$'s and $\mathcal{D}$'s payoff for $T$ targets with respect to the given strategy profile  \\
        {} & {$(\bm{c}, \bm{a}_i)$.}\\
        {$\bm{U}$} & The payoff structure, including  $U_i^a(\bm{c}, \bm{a}_i)$ and $U_i^d(\bm{c}, \bm{a}_i)$.\\
        {$t_{(1)},\ldots,t_{(T)}$} & {The order of all targets in descending order of $U_i^a(c_t)$.}  \\
        {$BR_i^a(\bm{c})$, $BR_i^d(\bm{c})$} & The best response of $\mathcal{A}_i$ and $\mathcal{D}$.  \\
        {$\Gamma_i(\bm{c})$} & {The attack set, a set of targets that can provide the maximum expected payoffs}  \\
        {} & {for $\mathcal{A}_i$ with the given $\bm{c}$, also denoted as $\Gamma_i$.}\\
        {$gap_i$} & The difference of the payoff if the coverage is changed from $\bm{c}$ to $\bm{c}$+$\bm{v}$ for $\mathcal{A}_i$. \\
        {$\Gamma_s(\bm{c})$} & {The attack group, a set of all $\Gamma_i(\bm{c})$, also denoted as $\Gamma_s$.}  \\
        {$\tilde{\bm{at}}$, $\tilde{\Gamma}_i$, $\tilde{\Gamma}_s$} & The attacked target, attack set and attack group of the feasible solution.  \\
        {$\hat{\bm{at}}$, $\hat{\Gamma}_i$, $\hat{\Gamma}_s$} & The attacked target, attack set, and attack group of the ideal solution.  \\
        {$\bm{\gamma}$, $\bm{\bm{\pi}}$} & A binary matrix and a continuous matrix.  \\
        {$K(\cdot,\cdot)$, $K_s(\cdot,\cdot)$} & The divergence between $\Gamma_i$ and $\Gamma_s$, respectively.  \\
        {$max\_{}gen$} & The maximum generation in EAs.  \\
        {$pop\_size$} & The population size in EAs.  \\
        \bottomrule
    \end{tabular}
\end{adjustbox}
\end{table*}

\subsection{Details about Benefits of Discretization and Optimization}
\label{app:protra}

MOSG is modeled by the leader-followers Stackelberg game, which is easy to generalize to a wide range of real-world problems. 
Lemma~\ref{pro:property1.2} and~\ref{pro:property1.1} explains the difficulty of MOSG optimization briefly. Here we further analyze the high difficulty of MOSG optimization. 
Firstly, under the fully rational attacker behavior assumption, $\mathcal{D}$'s payoff becomes a high-dimensional step function. Meanwhile, no feedback returns if the $\mathcal{D}$ allocates resources out of $\Gamma_i$. E.g., Figure~\ref{fig:Ua_Ud} shows the complex landscape with many flat and jump areas even in small-scale MOSG optimization. 
Secondly, the leader-followers Stackelberg game allows the attackers always go first, thus the behavior of $\mathcal{D}$ is limited to $\Gamma_i$ of $\mathcal{A}_i$, which means MOSG is constrained by $N$ conflicting constraints. Thirdly, the priority action of attackers results in $\mathcal{D}$'s payoff $U^d_i$ being affected by the $\mathcal{A}_i$'s payoff $U^a_i$, cf. $\bm{a}_i$ in Equation~\eqref{equ:equation2}, but we cannot control $U^a_i$ directly.  Figure~\ref{app:fig:Ua_Ud} displays a smooth landscape of $U_i^a$. For $U^d$ optimization, strategy $\bm{c}$ affects $U^d$ and $U^a$ simultaneously, and $U^a$ also affects $U^d$. The $U^d$ optimization becomes very complex because of the divergence in the optimization direction.

This paper proposes SDES with linear complexity for MOSG if either the number of targets or attackers is fixed, which models security games with many heterogeneous attackers. 
To solve large-scale MOSG problems, we use the discretization component transforms solution representation from high-dimensional continuous $c$-code to low-dimensional discrete $I$-code. Specifically, on the one hand, the transformation from $I$-code to $\bm{c}$-code eliminates the influence of hidden variables like $U^a$ by shifting search space from high-dimensional continuous $\bm{c}$ to low-dimensional discrete $\Gamma$. In addition, searching $\Gamma$ directly also eliminates the negative effect of priority action of $\mathcal{A}_i$. On the other hand, the original problem as Equation~\eqref{equ:equationOptFunc}) is transformed to a combinatorial optimization problem as Equation~\eqref{equ:equationBi-OptFunc} and use bit-wise optimization to solve the combinatorial optimization problem.

\label{app:payoff}
\begin{figure}[]
\centering
\includegraphics[width=.6\columnwidth]{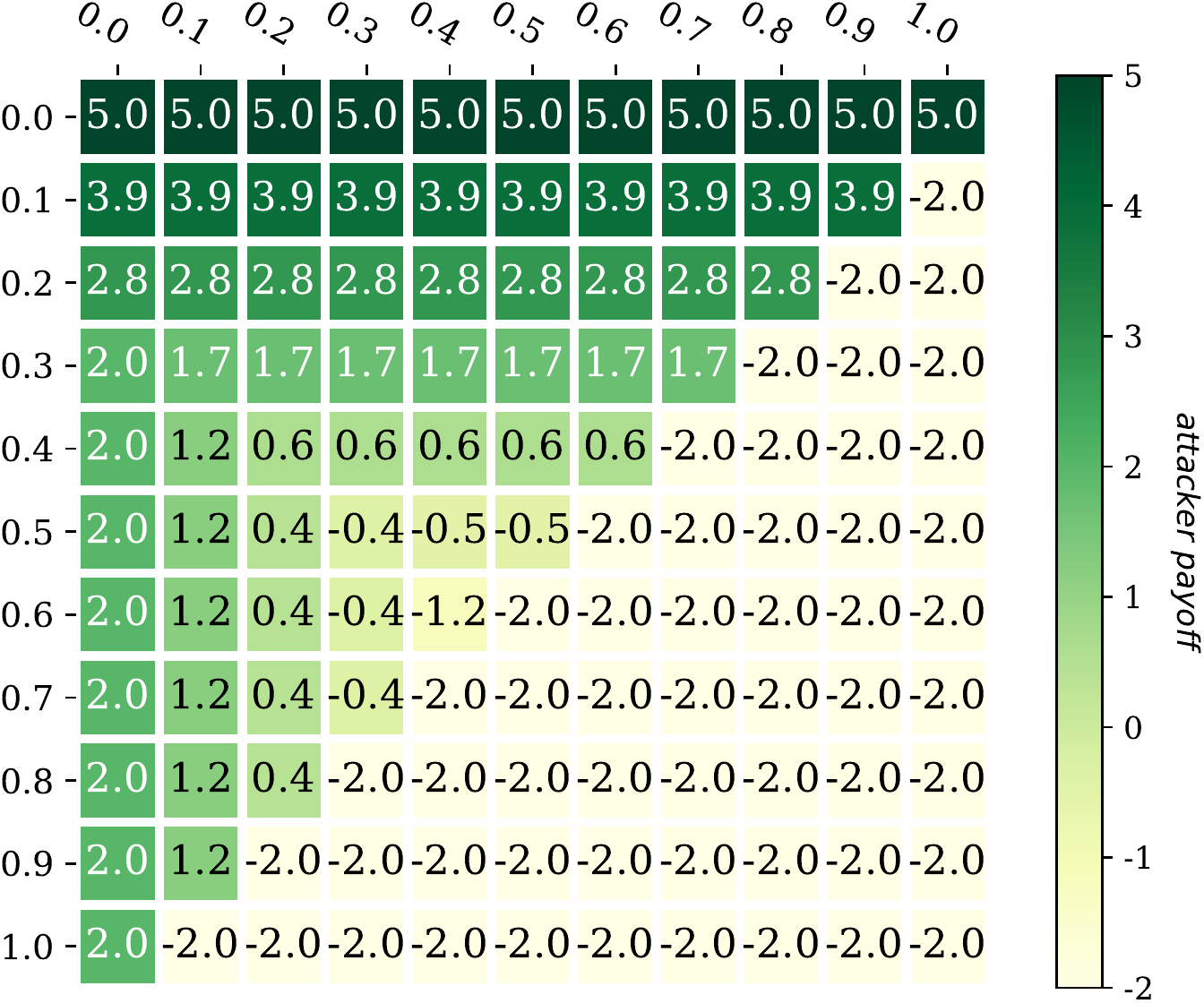}
\caption{The heatmap of attacker payoff $U^a_i(\bm{c})$. $\bm{c}$ directly affects continuous smooth $U^a_i$, which is easy to optimize. However, the goal of security games is not $\min_{t\in[T]}U^a_i(c_t)$, but $\max_{t\in[T]}U^d_i(c_t)$. To make matter worse, the optimization direction of $U^a$ can misdirect the optimization direction of $U^d$. }
\textbf{\label{app:fig:Ua_Ud}}
\end{figure}

\subsection{The Pseudocode of Ideal Attacked Target Calculation (IATC)}
\label{app:IatC}

Algorithm~\ref{alg:at_cal} aims to determine $\mathcal{A}$'s final attacked target by the maximal indifference property~\citep{ERASER} under the fully rational assumption. Line~\ref{code:IatC:line2}-~\ref{code:IatC:line3} describe the resources needed to keep $U^{u,a}_{obj}(t)$ the same, $\forall t\in\Gamma_{obj}$.   $x$ in Line~\ref{code:IatC:line2} is the adjusting goal. Line~\ref{code:IatC:line3} comes from setting $x$ equal to $U^{u,a}_{obj}(t)$ according to Equation~\eqref{equ:equation1}: $x=c_tU_{obj}^{c,u}(t)+(1-c_t)U_i^{u,a}(t)$~\citep{ERASER}. $U_{obj}^a$ calculation in Line~\ref{code:IatC:line4} comes from Equation~\eqref{equ:equation2}. Line~\ref{code:IatC:line5} comes from the fully rational assumption (i.e., $\mathcal{A}$ prefers to attack the target returning best payoff).

\begin{algorithm}
    \caption{Ideal Attacked Target Calculation (IATC)}
    \label{alg:at_cal}
    \begin{algorithmic}[1]
        \REQUIRE $U_{obj}^{u,a}$, $\Gamma_{obj}$.
        \ENSURE 
        
        \STATE$\bm{\hat{c}}\gets$ the zero ideal coverage vector.
        \STATE$x = U^{u,a}_{obj}(\Gamma_{obj}[-1])$. \small\textcolor{Green}{\% payoff of the last target in $\Gamma_{obj}$}\label{code:IatC:line2}
    
        \STATE$\hat{\bm{c}}[\Gamma_{obj}[:-1]] = \frac{x-U_{obj}^{u,a}(\Gamma_{obj}[:-1])}{U_{obj}^{c,a}(\Gamma_{obj}[:-1])-U^{u,a}_{obj}(\Gamma_{obj}[:-1])}$. \small\textcolor{Green}{\% adjust $\hat{\bm{c}}$ so that all targets' $U^a_{obj}$ equals $x$}\label{code:IatC:line3}
    
        \STATE Calculate $( U^a_{obj}(\hat{c}_t),t\in\Gamma_{obj})$.\label{code:IatC:line4}
    
        \STATE$\hat{\bm{at}}_{obj} = \Gamma_{obj}[\mathop{\arg\max}_{t\in[T]} ( U^a_{obj}(\hat{c}_t),t\in\Gamma_{obj})]$.\label{code:IatC:line5}
        \STATE \textbf{return} $\hat{\bm{at}}_{obj}$. \small\textcolor{Green}{\% the attacked target of $\mathcal{A}_{obj}$}
    \end{algorithmic}
\end{algorithm}

\subsection{The Framework and Pseudocode of Bit-wise Optimization (BitOpt)}
\label{app:fra_pseu_BSM}

The greedy algorithm divides the process of solution evaluation into $T$ bit-problems to bypass the potential combinatorial explosion. Each bit-problems is solved by the boolean scoring mechanism (BSM). BSM, the core of BitOpt, aims to search the $\bm{at}$ as similar as $\hat{\bm{at}}$. This section would introduce BSM in two perspectives, including pseudocode and visualization. 
Firstly, as shown in Figure~\ref{fig:heuristic_operator}  (as $t=2$ as an illustration), BSM aims to select $c_2^*$ from $\pi_{c_2}$. The $c_2^*$ optimization can be regarded as a binary string matching problem. 
Specifically, $(\bm{1}_{t_2}(\hat{at}_i),\ldots,\bm{1}_{t_2}(\hat{at}_M))$ is a fixed binary string with no rules, e.g., (0,1,\ldots,0,1). While $(\bm{1}_{t_2}(at_1),\ldots,\bm{1}_{t_2}(at_M))$ is a regular binary string, e.g. (0,\ldots,0,1,\ldots,1).
$Count(\cdot)$ measures the bit-wise difference between $(\bm{1}_{t_2}(\hat{at}_i),\ldots,\bm{1}_{t_2}(\hat{at}_M))$ and  $(\bm{1}_{t_2}(at_1),\ldots,\bm{1}_{t_2}(at_M))$. Finally, the $Count$ calculation part selects $c_t^*$ from $\pi_{c_t}$ by minimizing $Count(\cdot)$.
As for pseudocode perspective, those $Count$ calculation part is displaced in Algorithm~\ref{alg:pseudo2} Line~\ref{code:alg2:12}$\sim$\ref{code:alg2:20}.

\begin{figure*}[!t]
\centering
\includegraphics[width=.8\textwidth]{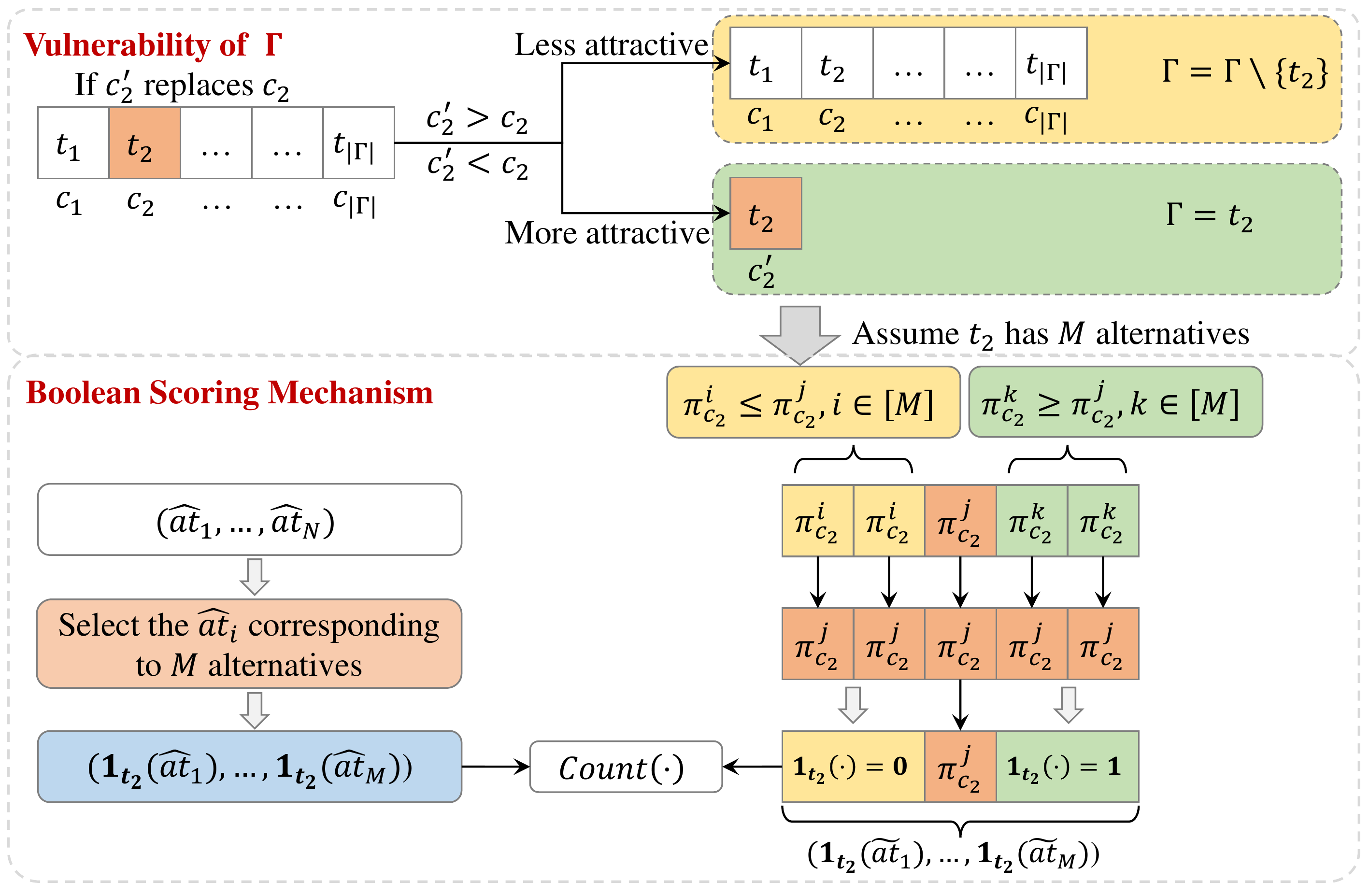}
\caption{The visualization of Lemma~\ref{pro:property1.2} at the top and the illustration of BSM details at the bottom. As shown in the top half, assume the coverage vector allocated on targets in $\Gamma$ is $\bm{c}=(c_1,\ldots,c_{\vert\Gamma\vert})$ and satisfies $c_1\leq\ldots\leq c_{|\Gamma|}$. If an allocation $c_2$ changes to $c_2'$, the $\Gamma$ changes to $\Gamma\setminus \{t_2\}$ or $\{t_2\}$ according to the value of $c_2$ and $c_2'$, e.g., $c_2<c_2'$ or $c_2>c_2'$. 
As shown in the bottom half, the right part analyzes the $\bm{at}$ corresponding to $\pi^j_{c_2}$, the left part calculates $\hat{\bm{at}}$, and the $Count(\cdot)$ calculates the distance between $\hat{\bm{at}}$ and $\bm{at}$. Specifically, suppose $M$ attackers provide target $t_2$ with $M$ alternatives $(\pi_{c_2}^1,\ldots,\pi_{c_2}^M)$ by ascending order. 
If $\mathcal{D}$ finally selects $\pi_{c_2}^j$ (orange) as $c_2^*$, each $\Gamma_i$ changes depending on the vulnerability of $\Gamma$ (the rule described in the top half). E.g., the $\Gamma$ of all $\mathcal{A}_i$ in yellow becomes $\Gamma\setminus \{t_2\}$, since $\pi_{c_2}^j>\pi_{c_2}^i$.
In addition, if $at_i=t_2$, then $\bm{1}_{t_2}(at_i)=1$. Otherwise, $\bm{1}_{t_2}(at_i)=0$. According to Lemma~\ref{pro:property1.2}, each $\Gamma_i$ in yellow must not contain $t_2$, while each $\Gamma_i$ in orange and green must contain $t_2$. Therefore, the $\bm{1}_{t2}(\cdot)$ value of blue area equals 0 while the $\bm{1}_{t2}(\cdot)$ value of red and green areas equals 1.}
\label{fig:heuristic_operator}
\end{figure*}

\begin{algorithm}
\caption{Bit-wise Optimization (BitOpt)}
\label{alg:pseudo2} 
\begin{algorithmic}[1]
    \REQUIRE $count_1$, $P$, $t$, $\bm{I}$, $\bm{A}$, $\bm{\hat{at}}$.
    \ENSURE 
    \STATE$\bm{c}^*\gets\ 0$.
    
    \small\textcolor{Green}{\% If there's only one $alternative$ (i.e., $\vert\bm{\pi_{c_t}}\vert=1$), choose it as $\bm{c}^*$. Otherwise, choose the one most similar to $\bm{\hat{at}}$.}
    \IF{$count_1=1$}
        \STATE$obj = P[t][0]$.
        \STATE$x = U^{u,a}_{obj}(\bm{A}_{[obj]}[I[obj]-1])$.
        \STATE$\bm{c}^* = \frac{x-U_{obj}^{u,a}(t)}{U_{obj}^{c,a}(t)-U^{u,a}_{obj}(t)}$.
    \ELSIF{$count_1>1$}
        \STATE$\pi_{obj}:List[int] = P[t]$.
        \STATE$\bm{x} = ( U^{u,a}_{obj}(\bm{A}_{[obj]}
        [I[obj]-1]),obj\in\pi_{obj})$.
        \STATE$\pi_{c_t} = (\frac{x[obj]-U_{obj}^{u,a}(t)}{U_{obj}^{c,a}(t)-U^{u,a}_{obj}(t)},obj\in\pi_{obj})$.
        \STATE sort $\bm{\hat{at},\pi_{c_t}}$ in ascending order of value by $\pi_{c_t}$.
        \STATE$\bm{count}_2:shapelike(\pi_{c_t})\gets$ the zero distance counter between $\tilde{\Gamma}_{s}$ and $\hat{\Gamma}_{s}$.

        \small\textcolor{Green}{\% Boolean Scoring Mechanism (BSM) part.}\\
        \small\textcolor{Green}{\% Update $\bm{count}_2[obj]$ based on $target$'s relationship to $\hat{\bm{at}}[obj-1]$ and $\hat{\bm{at}}[obj]$.}
        \FOR{$1\leq obj < \vert\pi_{obj}\vert$} \label{code:alg2:12}
            \STATE$count_2[obj] = \bm{count}_2[obj - 1]$.
            \IF{$\hat{\bm{at}}[obj-1]\neq t$}
                \STATE$\bm{count}_2[obj]\ +=\ 1$.
            \ENDIF
            \IF{$\hat{\bm{at}}[obj]\neq t$}
                \STATE$\bm{count}_2[obj]\ -=\ 1$.
            \ENDIF
        \ENDFOR \label{code:alg2:20}
        \STATE$c^*=\pi^{idx}_{c_t}, idx=\mathop{\arg\min}_i \bm{count}_2[i]$.
    \ENDIF
    \STATE\textbf{return} $\bm{c}^*$. \small\textcolor{Green}{\% the best alternative for component $c_t$}
\end{algorithmic}
\end{algorithm}

\section{Proof Details}
\label{app:proof}
\subsection{The Proof of Lemmas and Theorems in Section~\ref{sec:met}}
\label{app:proof:method}
\lemmaSGperspective*
Lemma~\ref{pro:property1.2} shows that $\bm{v}$ may cause a sharp change from $\Gamma_i(\bm{c})$ to $\Gamma_i(\bm{c}+\bm{v})$. Since the resource allocation between different targets is independent, we only discuss $\bm{v}$ with one non-zero component $v_{t'}$. Note that if $\bm{v}$ is a zero vector, $\Gamma_i(\bm{c})$ does not change, i.e., $\Gamma_i(\bm{c})=\Gamma_i(\bm{c}+\bm{v})$. In SGs, $\mathcal{A}_i$ only attacks the most attractive target $at_i$, i.e., $at_i\in\Gamma_i(\bm{c})$. The reason for Lemma~\ref{pro:property1.2} is two-fold: 
(i) If the resource allocation $\bm{v}$ causes the original $at_i$ is no longer the most attractive target, i.e., $at_i\notin \Gamma_i(\bm {c}+\bm{v})$, then $U_i^d(\bm {c}+\bm{v})$ changes dramatically. 
(ii) If the resource allocation $\bm{v}$ moves the targets that can not yield the maximum $U_i^d$ out of $\Gamma_i(\bm{c})$, then $U_i^d(\bm {c})$ would not change because of the nature of $Max$ function, i.e., $\Gamma_i(\bm{c})=\Gamma_i(\bm{c}+\bm{v})$. The proof detail is
\begin{proof}\label{pro:property1.2 proof}
We categorize the cases into (i) If $v_{t'}$ is allocated to the target $t'\notin\Gamma_i(\bm{c})$, then $\Gamma_i(\bm{c})$ does not be affected. (ii) Otherwise, $\Gamma_i(\bm{c})$ would change regularly.

If $t'\notin\Gamma_i(\bm{c})$, since the best payoff of $\mathcal{A}_i$ and $\mathcal{D}$ only come from $\Gamma_i(\bm{c})$, any minor variation $v_{t'}$ does not change $\Gamma_i(\bm{c})$, i.e., $\Gamma_i(\bm{c}+\bm{v})=\Gamma_i(\bm{c})$. 

If $t'\in\Gamma_i(\bm{c})$, an illustration is given in the top of Figure~\ref{fig:heuristic_operator}. 
Recall that $\Gamma_i(\bm{c})$ collects targets yielding $\mathcal{A}_i$'s maximum payoff $U_i^a$, i.e., $U_i^a(c_{t'}) = U_i^a(c_{t''}), \forall t',t''\in\Gamma_i(\bm{c})$. If $v_{t'}\neq0$, then $U_i^a(c_{t'})\neq U_i^a(c_{t''}),\forall t''\in\Gamma_i(\bm{c}),t''\neq t'$.
    
Next, we summarize how $v_{t'}$ changes $\Gamma_i(\bm{c})$. If $v_{t'}>0$, $\mathcal{D}$ allocates more resources on $t'$, so $\mathcal{A}_i$ gains less, i.e., $U_i^a(c_{t'}+v_{t'})<U_i^a(c_{t''}),t''\in\Gamma_i(\bm{c})$, $t''\neq t'$. Therefore,  $\Gamma_i(\bm{c} + \bm{v}) = \Gamma_i(\bm{c}) \setminus\{t'\}$. Similarly, if $v_{t'}<0$, $\mathcal{A}_i$ gains more, i.e., $U_i^a(c_{t'})>U_i^a(c_{t''}),t''\in\Gamma_i(\bm{c})$, $t''\neq t'$. Therefore, $\Gamma_i(\bm{c}+\bm{v})=\{t'\}$.
\end{proof}

\lemmaoptperspective*
\begin{proof}\label{pro:property1.1 proof}
The proof process consists of two parts: (i) the categorical discussion of the relationship between $U_i^d(\bm{c})$ and $U_i^d(\bm{c}+\bm{v})$. (ii) The calculation of the lower bound of changed value $gap_i$.

For the first part, $\mathcal{D}$'s payoff $U^d_i(\bm{c})$ is yielded by only one target $at_i\in\Gamma_i(\bm{c})$. Therefore, if $at_i$ still in $\Gamma_i(\bm{c}+\bm{v})$ when the \textit{converge vector} $\bm{c}$ changes to $\bm{c}+\bm{v}$, then $U_i^d(\bm{c}+\bm{v})=U_i^d(\bm{c})$. Otherwise, $U_i^d(\bm{c}+\bm{v})\neq U_i^d(\bm{c})$.
    
As shown in Lemma~\ref{pro:property1.2}, suppose $v_{t'}\neq0$, $\Gamma_i(\bm{c}+\bm{v})$ can be mathematically described with $\Gamma_i(\bm{c})$ and $t'$: 
If $t'\notin\Gamma_i(\bm{c})$, then $U^d_i(\bm{c}+\bm{v})=U^d_i(\bm{c})$; If $t'\in\Gamma_i(\bm{c})$ and $at_i\in\Gamma_i(\bm{c}+\bm{v})$, then $U^d_i(\bm{c}+\bm{v})=U^d_i(\bm{c})$; If $t'\in\Gamma_i(\bm{c}+\bm{v})$ and $at_i\notin\Gamma_i(\bm{c}+\bm{v})$, then $U^d_i(\bm{c}+\bm{v})\neq U^d_i(\bm{c})$.

In general, suppose $v_{t'}\neq 0$, if $t'\in\Gamma_i(\bm{c})$, then $\Gamma_i(\bm{c}+\bm{v})\neq\Gamma_i(\bm{c})$, and if  $t'\notin\Gamma_i(\bm{c})$, then $\Gamma_i(\bm{c}+\bm{v})=\Gamma_i(\bm{c})$. Note that if $v_{t'}=0$, then $\Gamma_i(\bm{c}+\bm{v})=\Gamma_i(\bm{c})$ because $\bm{v}$ is a zero vector.

For the second part, recall Equation~\eqref{equ:equation2-2} that $U_i^d(\bm{c}+\bm{v})=\max_{t\in\Gamma_i(\bm{c}+\bm{v})}U_i^d(c_t)$. Since $at_i\notin\Gamma_i(\bm{c}+\bm{v})$, $\Gamma_i(\bm{c}+\bm{v})\subseteq\Gamma_i(\bm{c})$. Therefore, one has:



    
\begin{equation}
\label{app:equ:Lemma2proof_1}
U^d_i(\bm{c}+\bm{v}) =\max_{t\in\Gamma_i(\bm{c}+\bm{v})} U^d_i(c_t) \leq\max_{t\in\Gamma_i(\bm{c})}U^d_i(c_t)\,.
\end{equation}

Since $at_i\notin\Gamma_i(\bm{c}+\bm{v})$, $at_i=\mathop{\arg\max}_{t\in\Gamma_i(\bm{c})}U_i^d(c_t)$, then $gap_i$ can be derived from Equation~\eqref{app:equ:Lemma2proof_1}:


\begin{equation}
\label{app:equ:Lemma2proof_2}
gap_i = \max_{t\in\Gamma_i(\bm{c})}U^d_i(c_t) - \max_{t\in\Gamma_i(\bm{c})\setminus\{at_i\}}U^d_i(c_t)\,.
\end{equation}
where $U_i^d(\bm{c})=\max_{t\in\Gamma_i(\bm{c})}U^d_i(c_t)$,  $U_i^d(\bm{c}+\bm{v})=\max_{t\in\Gamma_i(\bm{c})\setminus\{at_i\}}U^d_i(c_t)$. Since $U_i^d(\bm{c})\neq U_i^d(\bm{c}+\bm{v})$, $gap_i>0$.
\end{proof}

\lemmaoptimalexpansion*
The proof of Lemma~\ref{pro:property2} can be derived from the optimal expansion rules in Lemma~\ref{pro:property2}~\citep{ERASER}. 
In the Stackelberg game with multiple heterogeneous attackers, each attacker ultimately attacks only one target (without loss of generality because SSE always exists). In the fully rational SGs, $\mathcal{A}_i$ only focus on targets returning the maximal $U^a_i(\bm{c})$, which is defined as $\Gamma_i(\bm{c})$. 
Namely, the attacker's intentions are predictable.

The optimal expansion rule can provide $N$ candidate attack sets with sizes from 1 to $T$, in which an optimal attack set always exists no matter how many resources the defender owns because of SSE~\citep{SSE}.

\subsection{The Proof of Lemmas and Theorems in Section~\ref{sec:val_ana}}
\label{app:proof:3}
\lemmaTheoAnal*
\begin{proof}
From SSE assumption, there is an observation from~\citep{ERASER}. The optimal $\mathcal{D}$'s payoff for defending $\mathcal{A}_i$ is derived from the target in $\mathcal{A}_i$'s $\Gamma_i(\bm{c})$, thus if $\Gamma_i(\bm{c})\subseteq\Gamma_i(\bm{c}')$ and $c_t=c_t'$, $\forall t\in\Gamma_i(\bm{c})$, then $U_i^d(\bm{c})\leq U_i^d(\bm{c}')$.

For Lemma~\ref{lem:lemma1}, since $\Gamma_i(\bm{c})$ is constructed from the optimal expansion rule, i.e., add $t$ to $\Gamma_i(\bm{c})$ in descending order of $U_i^a(c_t)$, if $\Gamma_i(\bm{c}')$ is larger than $\Gamma_i(\bm{c})$, then we have $\Gamma_i(\bm{c})\subseteq\Gamma_i(\bm{c}')$ and $c'_t\geq c_t,\forall t\in\Gamma_i(\bm{c})$. Let's discuss them in categories (i) $c'_t=c_t$ and (ii) $c'_t>c_t$. 

\begin{itemize}
    \item [(i)]  
    According to the observation in~\citep{ERASER}, if $\Gamma_i(\bm{c})\subseteq\Gamma_i(\bm{c}')$ and $c_t=c_t'$, $\forall t\in\Gamma_i(\bm{c})$, then $U_i^d(\bm{c})\leq U_i^d(\bm{c}')$.
    
    \item[(ii)] If $c'_t>c_t$, since the more resources allocated to $t$, the higher $\mathcal{A}_i$ payoff of the defender $U^d_i(\bm{c})$, the observation in~\citep{ERASER} still holds.
\end{itemize}
\end{proof}

\obsalgorithmfir*
Since both $\bm{c}$ and $\hat{\bm{c}}^i$ are constructed from the optimal expansion rule according to Lemma~\ref{pro:property2}, the size of $\Gamma_i(\bm{c})$ corresponds to the content of $\Gamma_i(\bm{c})$, i.e., the target order adding to $\Gamma_i(\bm{c})$ must be in descending $U_i^a(c_t)$. Since the ideal solution $\hat{\bm{c}}^i$ uses the maximal resources defending $\mathcal{A}_i$, $\Vert\bm{c}\Vert_1\leq\Vert\hat{\bm{c}}^i\Vert_1$. $\Vert\tilde{\bm{c}}\Vert_1\leq\Vert\hat{\bm{c}}^i\Vert_1$ means that $\hat{\bm{c}}^i$ has more resources and $\mathcal{D}$ can protect more targets, i.e., $\Gamma_i(\bm{c})\subseteq\Gamma_i(\hat{\bm{c}}^i)$. Meanwhile, if $\Vert\bm{c}\Vert_1\leq\Vert\hat{\bm{c}}^i\Vert_1$ and $\Gamma_i(\bm{c})\subseteq\Gamma_i(\hat{\bm{c}}^{i})$, then $U_i^d(\bm{c})\leq U_i^d(\hat{\bm{c}}^{i})$, according to Lemma~\ref{lem:lemma1}. Finally, since each component $U_i^d(\hat{\bm{c}})$ is constructed by using all resources to defending $\mathcal{A}_i$, $U_i^d(\hat{\bm{c}})$ is the upper bound of any strategy for defending $\mathcal{A}_i$. Therefore, $(U_i^d(\hat{\bm{c}}^1),\ldots,U_i^d(\hat{\bm{c}}^N))$ cannot be dominated by any strategy on the PF.

\obsalgorithmthi*
\begin{proof}
The difference of \textit{attack set} and \textit{attack group} lies in model type of SGs. Attack set focuses on single-objective modeling with one attacker, while attack group focuses on multi-objective modeling with multiple attackers. Conclusions in attack set perspective can easily be extended to conclusions in attack group perspective, because the attackers are independent of each other. The discussion below focuses on conclusions in the attack set perspective.

Under SSE assumption in the single attacker scenario, $\mathcal{A}_i$ finally attacks only one target $at_i$ returning the maximal payoff. $at_i\in\Gamma_i(\bm{c})$ means that $at_i$ is protected by $\mathcal{D}$. $at_i\in\Gamma_i(\bm{c})$ and $at_i\in\Gamma_i(\hat{\bm{c}}^i)$ means that both $\bm{c}$ and $\hat{\bm{c}}^i$ can protect $at_i$. Otherwise, $\bm{c}$ cannot protect $at_i$ well. Therefore, the divergence $K(\Gamma_i(\bm{c}), \Gamma_i(\hat{\bm{c}}^i))$ calculation is determined only by $at_i$ but not all elements of $\Gamma_i(\bm{c})$, $\Gamma_i(\hat{\bm{c}}^i))$. After generalizing the above result from the single attacker scenario to the $N$-attacker scenario, one has that the divergence $K_s(\Gamma_s(\bm{c}), \hat{\Gamma}_s)$ calculation is determined only by $\bm{at}$ but not all elements of $\Gamma_s(\bm{c})$, $\hat{\Gamma}_s$.

As for Theorem~\ref{obs:algorithm3}, $\Gamma_i(\bm{c})$ is similar to $\Gamma_i(\hat{\bm{c}}^i)$ means that both $\bm{c}$ and $\hat{\bm{c}}^i$ allow $\mathcal{D}$ to protect $at_i$ from $\mathcal{A}_i$. Meanwhile, since $\bm{c}$ and $\hat{\bm{c}}^i$ are constructed according to the optimal expansion rule, the feasible solution obtains the same payoff as the ideal solution if $at_i\in\Gamma_i(\bm{c})$ and $at_i\in\Gamma_i(\hat{\bm{c}}^i)$ for all $i\in[N]$. Finally, from Lemma~\ref{obs:algorithm1}, $(U_i^d(\hat{\bm{c}}^1),\ldots,U_i^d(\hat{\bm{c}}^N))$ cannot be dominated by any feasible strategy. If there is a feasible strategy $\bm{c}$ can achieve $(U_i^d(\hat{\bm{c}}^1),\ldots,U_i^d(\hat{\bm{c}}^N))$, then $\bm{c}$ is the solution on the PF.

Meanwhile, Lemma~\ref{lem:lemma1} has been supplemented by Theorem~\ref{obs:algorithm3}. In detail, even though $\Gamma_i(\bm{c})\subseteq\Gamma_i(\hat{\bm{c}}^i)$, if they have the same $at_i$, then they are similar and have the same payoff. Namely, the feasible solution $\tilde{\bm{c}}$ can achieve the upper bound, i.e., $(U_i^d(\hat{\bm{c}}^1),\ldots,U_i^d(\hat{\bm{c}}^N))$.
\end{proof}

\splitassumption*
\theoremRatBitbybit*
\begin{proof}
Theorem~\ref{obs:val_ana1} shows that the upper bound $(U_i^d(\hat{\bm{c}}^1),\ldots,U_i^d(\hat{\bm{c}}^N))$ can be achieved by $\bm{c}$ as long as $\Gamma_s(\bm{c})$ contains the attacked target vector $\bm{at}$. In the top half of Figure~\ref{fig:heuristic_operator}, it describes what happened to the $\Gamma_i(\bm{c})$ when the resources on $t\in\Gamma_i(\bm{c})$ changes. Specifically, recall Lemma~\ref{pro:property1.2} and suppose $c_2$ changes to $c_2'$, one has that: (i) If $c_2'>c_2$, then $\Gamma_i(\bm{c})=\Gamma_i(\bm{c})\setminus \{t_2\}$. (ii) If $c_2'<c_2$, then $\Gamma_i(\bm{c})=\{t_2\}$. In the bottom half of Figure~\ref{fig:heuristic_operator}, it describes what happened to $\Gamma_i(\bm{c})$ when choosing $\pi_{c_2}^j$ from $\pi_{c_2}$. Suppose $\pi_{c_2}$ is sorted in increasing order and $\pi_{c_2}^j$ is selected, one has that: (i) the $\Gamma_i(\bm{c})$ corresponding to all $\pi_{c_2}^i$, which are less than $\pi_{c_2}^j$, satisfies $\Gamma_i(\bm{c})=\Gamma_i(\bm{c})\setminus \{t_2\}$.  (ii) The $\Gamma_i(\bm{c})$ corresponding to all $\pi_{c_2}^i$, which are greater than $\pi_{c_2}^j$, satisfies $\Gamma_s(\bm{c})=\{t_2\}$. To ensure $\Gamma_i(\bm{c})$ contains all $\bm{at}$, we summarize the following conditions

If BitOpt subjects to Assumption~\ref{asp:BitOpt}, all $\bm{at}$ dose not affected. As longe as all $\bm{at}$ dose not affected when searching each $c_i^*\in\{\pi_{c_i}^1,\ldots,\pi_{c_i}^N\}$, $\bm{c}^*$ can achieve the upper bound by Theorem~\ref{obs:algorithm3}.  Finally, by Theorem~\ref{obs:algorithm3}, the solutions found by BitOpt must on the true PF.

\end{proof}

\section{Additional Experimental Details and Results}
\subsection{More Details about Experiment Setup}
\label{app:expset}
\subsubsection{Hyperparameter Setup}

SDES has a few hyperparameters needed to tune, including $max\_{}gen$ and $pop\_{}size$. On the one hand, we set $max\_{}gen=300$ in all MOSG problems. Figure~\ref{fig:convergence_curve} shows the fast convergence speed of our method (within 50 iterations), so 300 iterations are able to ensure algorithm convergence. Meanwhile, SDES is insensitive to $max\_{}gen$ because of the linear complexity with both target and heterogeneous attacker dimensions. On the other hand, $pop\_{}size$ determines the granularity of the subspace, i.e., $\vert\tilde{C}\vert$. The larger $pop\_{}size$, the better the space distribution. Although solution quality is guaranteed, the relationship between $pop\_{}size$ and runtime is undetermined. Section~\ref{app:more_time_effi} further analyze the effect of increasing $pop\_{}size$ on runtime.

\subsubsection{Implementation Setup of NSGA-III in the Optimization Component of SDES}
\label{app:NSGA-III}
The basic configuration of NSGA-III includes sampling=Int Random Sampling, selection=Tournament Selection, crossover=Simulated Binary Crossover or Half Uniform Crossover (we choose SBX in this paper), Mutation=Int Polynomial Mutation. The implementation of NSGA-III uses an open-source framework: multi-objective optimization in Python\footnote{The document and code are available at \url{https://pymoo.org/index.html}.}~\citep{pymoo} (PyMOO). In PyMOO, crossover rate, mutation rate, and other hyperparameter use default values, and the final result saves all non-dominated solutions in the iteration process. The pseudocode of each generation of NSGA-III is shown in Algorithm~\ref{app:alg:NSGA-III}~\citep{NSGA-III}.

In addition, the NSGA-III in the optimization component of SDES uses fixed population size and generation number. Specifically, since Riesz can adaptively provide a set of well-spaced reference direction, the many-objective EA with a fixed but sufficiently large population size is still able to find a well-spaced PF. Meanwhile, the ORIGAMI method is shown to converge a single solution in constant generations~\citep{MOSG}.

\begin{algorithm}
    \caption{Generation $t$ of NSGA-III Procedure}
    \label{app:alg:NSGA-III}
    \begin{algorithmic}[1]
        \REQUIRE $H$ structured reference points $Z^s$ or supplied aspiration points $Z^a$, parent population $P_t$.
        \ENSURE
        
        \STATE $S_t=\emptyset$, $i=1$.
        \STATE $Q_t=$ Crossover+Mutation$(P_t)$.
        \STATE $R_t=P_t\cup Q_t$
        \STATE $(F_1,F_2,\ldots)=$ Non-dominated-sort$(R_t)$.
        \REPEAT
        \STATE $S_t=S_t\cup F_i$ and $i=i+1$.
        \UNTIL{$\vert S_t\vert\geq N$} 
        \STATE Last front to be included: $F_l=F_i$.
        \IF{$\vert S_t\vert=N$}
        \STATE $P_{t+1}=S_t$, \textbf{return} $P_{t+1}$.
        \ELSE
        \STATE $P_{t+1}=\cup_{j=1}^{l-1}F_j$.
        \STATE Points to be chosen from $F_l$: $K=N-\vert P_{t+1}\vert$.
        \STATE Normalize objectives and create reference set $Z^r$: $Normalize(\bm{\mathrm{f}}^n,S_t,Z^r,Z^s,Z^a)$.
        \STATE Associate each member $\bm{s}$ of $S_t$ with a reference point: $[\pi_{\bm{s}},d(\bm{s})]=Associate(S_t,Z^r)$.  \small \textcolor{Green}{\% $\pi(\bm{s})$: closest reference point, $d$: distance between $\bm{s}$ and $\pi(\bm{s})$}
        \STATE Compute niche count of reference point $j\in Z^r$: $\rho_j=\sum_{\bm{s}\in S_t/F_l}((\pi(\bm{s})=j)\,?\,1:0)$. \small \textcolor{Green}{\% $((\cdot)\,?\,1:0)$: if $(\cdot)$ is true then return $1$, otherwise return $0$}
        \STATE Choose $K$ members one at a time from $F_l$ to construct $P_{t+1}$: $Niching(K,\rho_j,\pi,d,Z^r,F_l,P_{t+1})$
        \ENDIF
        \STATE \textbf{return} $P_{t+1}$.
    \end{algorithmic}
\end{algorithm}

\subsubsection{Solution Initialization of SDES}
The initialization of solutions is also the result of the discretization component of SDES. According to the definition of $\Gamma_s$, the solution initialization can be defined as a $N$-integers random vector $\bm{I}$, $I_i\in[T]$. However, due to limited resources and other reasons, the scenario $I_i=T$ (all targets are added to $\Gamma_i$) generally does not appear in real problems. Therefore, we can optimize the initialization with redundancy as follows: Firstly, calculate the maximum attack set  $\vert\Gamma_i\vert_{max}$ that can be obtained by allocating all resources for defending $\mathcal{A}_i$. Secondly, initialize $I_i\in[\vert\Gamma_i\vert_{max}]$. Each component of $\bm{I}$ means the size of attacker set $\Gamma$.

\subsection{More Details about Effectiveness Experiments}
\label{app:more_eff_exp}

We show the details of the effectiveness experiment, cf. Section~\ref{sec:effe:part2}, in Table~\ref{app:tab:performance:N=3},~\ref{app:tab:performance:N=4},~\ref{app:tab:performance:N=5},~\ref{app:tab:performance:N=6},~\ref{app:tab:performance:N=7}, and~\ref{app:tab:performance:N=8}, including \textit{HV}, \textit{IGD}$^+$ scores and ranks of all methods and the constructed PF. The problem scale is $N=[3,\ldots,8]$, $T=[25,\ldots,1000]$. We further consider the $N=3$ case, although it does not conform to the provisions of MaOPs ($N>3$). That is because all the comparison algorithms fail successively When the objective dimension increases. Specially, comparison algorithms fail when $T\geq600$ in $N=4$ case, $T\geq100$ in $N=5,6$ case, $T\geq50$ in $N=7$ case, and $T\geq25$ in $N=8$ case. The exponential growth in time consumption of comparison algorithms limited the content we could analyze, so we also showed the experiment with $N=3$. 

Table~\ref{app:tab:performance:N=3},~\ref{app:tab:performance:N=4},~\ref{app:tab:performance:N=5},~\ref{app:tab:performance:N=6},~\ref{app:tab:performance:N=7}, and~\ref{app:tab:performance:N=8},  show the \textit{HV} and \textit{IGD}$^+$ indicators of SDES and comparison algorithms on different problem scales. \textit{HV} and \textit{IGD}$^+$ are relative indicators: the larger \textit{HV}, the better the algorithm, while \textit{IGD}$^+$ is the opposite. As result, when the problem size is small, all the algorithms can complete within the maximum time $M=30$ mins. However, as the problem scale grows, methods fail in sequence: 
ORIGAMI-M-BS $\prec$ ORIGAMI-A $\prec$ ORIGAMI-M $\prec$ DIRECT-MIN-COV $\prec$ SDES. 
The symbol - means the runtime of the corresponding algorithm is out of $M$ in Table~\ref{app:tab:performance:N=3}$\sim$~\ref{app:tab:performance:N=8}. 

As results, SDES achieve rank-1 in all MOSG problems in $N=3,4,6,7,8$ cases, except $N=5$. In $N=5$ case, SDES achieve 2 rank-1, 3 rank-2, 2 rank-3, 1 rank-4. The reason for the poor performance of SDES is (i) the SOTA methods are exact algorithms. Although they face the curse of dimension, they performance very well in medium-scale problems. (ii) $N=5$ happens to be the last problem scale for the SOTA methods to time out. They produce enough high-quality solutions in a limited time. However, they cannot accomplish larger problem scales because of the curse of dimensionality.

\begin{table*}[]
\caption{Multi-objective performance indicator of SDES and comparison algorithms across all problem scales in $N=3$}
\label{app:tab:performance:N=3}
\begin{adjustbox}{width=\textwidth}
\centering
\begin{tabular}{lrrrrrrrr}
\toprule
        \multicolumn{1}{c}{\textbf{$N=3$}} & \multicolumn{2}{c}{\textbf{$T=25$}} & \multicolumn{2}{c}{\textbf{$T=50$}} & \multicolumn{2}{c}{\textbf{$T=75$}} & \multicolumn{2}{c}{\textbf{$T=100$}} \\  
        \multicolumn{1}{c}{\textbf{$\bm{Method}$}} & $HV$/Rank & $IGD^{+}$/Rank & $HV$/Rank & $IGD^{+}$/Rank & $HV$/Rank & $IGD^{+}$/Rank & $HV$/Rnak & $IGD^{+}$/Rank \\ 
        \midrule
        PF & 243 & 0 & 1480 & 0 & 2163 & 0 & 2201 & 0 \\  
        \textbf{SDES} & \textbf{180}{\scriptsize $\pm$2}/1 &  \textbf{0.27}{\scriptsize $\pm$0.00}/1 & \textbf{680}{\scriptsize $\pm$514}/1 & \textbf{0.69}{\scriptsize $\pm$0.46}/1 & \textbf{868}{\scriptsize $\pm$512}/1 & \textbf{0.74}{\scriptsize $\pm$0.33}/1 & \textbf{754}{\scriptsize $\pm$492}/1 &  \textbf{0.97}{\scriptsize $\pm$0.39}/1 \\  
        ORIGAMI-M & 172/4 & 0.48/3 & 318/3 & 1.57/3 & 377/2 & 1.78/3 & 403/2 & 1.54/3 \\  
        ORIGAMI-A & 179/2 & 0.45/2 & 323/2 & 1.54/2 & 365/3 & 1.74/2 & 397/4 & 1.52/2 \\  
        ORIGAMI-M-BS & 148/5 & 0.78/5 & 232/5 & 1.89/5 & 286/5 & 2.25/4 & 279/5 & 2.26/5 \\  
        DIRECT-MIN-COV & 174/3 & 0.57/4 & 293/4 & 1.87/4 & 323/4 & 2.43/5 & 401/3 & 1.79/4 \\
        \textbf{RANK(OURS)}&\textbf{1}&\textbf{1}&\textbf{1}&\textbf{1}&\textbf{1}&\textbf{1}&\textbf{1}&\textbf{1}\\
        
        \midrule
        \multicolumn{1}{c}{\textbf{$N=3$}} & \multicolumn{2}{c}{\textbf{$T=200$}} & \multicolumn{2}{c}{\textbf{$T=400$}} & \multicolumn{2}{c}{\textbf{$T=600$}} & \multicolumn{2}{c}{\textbf{$T=800$}} \\ \midrule
        PF & 1161 & 0 & 1642 & 0  &1549  &0&1209 &0\\  
        \textbf{SDES} & \textbf{640}{\scriptsize $\pm$230}/1 & 1.07{\scriptsize $\pm$0.43}/1 & \textbf{1107}{\scriptsize $\pm$298}/1 & \textbf{0.57}{\scriptsize $\pm$0.40}/1  & \textbf{863}{\scriptsize $\pm$497}/1 & \textbf{0.53}{\scriptsize $\pm$0.49}/1 & \textbf{728}{\scriptsize $\pm$199}/1 & \textbf{0.57}{\scriptsize $\pm$0.36}/1\\  
        ORIGAMI-M & 457/3 & 1.19/2 & 557/3 & 1.80/3  &290/2 &1.58/3 & 447/2 & 1.70/3\\  
        ORIGAMI-A & 426/4 & 1.32/3 & \multicolumn{2}{c}{--}&\multicolumn{2}{c}{--}&\multicolumn{2}{c}{--}\\  
        ORIGAMI-M-BS & 415/5 & 1.47/5 & 535/4 & 2.02/4  & 265/4 & 1.57/2  & 388/3 &1.63/2\\  
        DIRECT-MIN-COV & 483/2 & 1.36/4 & 655/2 & 1.66/2  & 267/3 &1.89/4 &\multicolumn{2}{c}{--}\\ 
        \textbf{RANK(OURS)}&\textbf{1}&\textbf{1}&\textbf{1}&\textbf{1}&\textbf{1}&\textbf{1}&\textbf{1} &\textbf{1}\\

        \midrule
        \multicolumn{1}{c}{\textbf{$N=3$}} & \multicolumn{2}{c}{\textbf{$T=1000$}}\\ \midrule
        PF & 609 & 0 &  &  &  &&\\  
        \textbf{SDES} & \textbf{355}{\scriptsize $\pm$58}/1 & \textbf{0.72}{\scriptsize $\pm$0.27}/1 & &  & &&\\  
        ORIGAMI-M &\multicolumn{2}{c}{--} & & & &&\\  
        ORIGAMI-A & \multicolumn{2}{c}{--}&  &  & &&\\  
        ORIGAMI-M-BS & 212/2 & 1.69/2 &  & & && \\  
        DIRECT-MIN-COV &\multicolumn{2}{c}{--} & &  & && \\ 
        \textbf{RANK(OURS)}&\textbf{1}&\textbf{1}& & &&&\\
\bottomrule
\end{tabular}
\end{adjustbox}
\end{table*}

\begin{table*}[]
\caption{Multi-objective performance indicator of SDES and comparison algorithms across all problem scales in $N=4$}
\label{app:tab:performance:N=4}
\begin{adjustbox}{width=\textwidth}
\centering
\begin{tabular}{lrrrrrrrr}
\toprule
        \multicolumn{1}{c}{\textbf{$N=4$}} & \multicolumn{2}{c}{\textbf{$T=25$}} & \multicolumn{2}{c}{\textbf{$T=50$}} & \multicolumn{2}{c}{\textbf{$T=75$}} & \multicolumn{2}{c}{\textbf{$T=100$}} \\
        \multicolumn{1}{c}{\textbf{$\bm{Method}$}} & $HV$/Rank & $IGD^{+}$/Rank & $HV$/Rank & $IGD^{+}$/Rank & $HV$/Rank & $IGD^{+}$/Rank & $HV$/Rnak & $IGD^{+}$/Rank \\
        \midrule
        PF & 6.42e4 & 0 & 3.66e4 & 0 & 9.25e4 & 0 & 1.05e5 & 0 \\  
        \textbf{SDES} & \textbf{3.12e4}{\scriptsize $\pm$6018}/1 & \textbf{0.47}{\scriptsize $\pm$0.08}/1 & \textbf{2.31e4}{\scriptsize $\pm$3709}/1 & \textbf{0.52}{\scriptsize $\pm$0.13}/1 & \textbf{4.73e4}{\scriptsize $\pm$7691}/1 & \textbf{0.64}{\scriptsize $\pm$0.12}/1 & \textbf{5.15e4}{\scriptsize $\pm$9247}/1 & \textbf{0.56}{\scriptsize $\pm$0.14}/1 \\  
        ORIGAMI-M & 2.84e4/2 & 0.56/2 & 1.85e4/3 & 0.89/2 & 4.03e4/2 & 0.86/3 & 4.37e4/3 & 0.86/3 \\  
        ORIGAMI-A & 2.76e4/3 & 0.67/3 & 1.88e4/2 & 0.93/3 & 3.95e4/3 & 0.73/2 & 4.42e4/2 & 0.70/2 \\  
        ORIGAMI-M-BS & 2.38e4/5 & 1.06/4 & 1.61e4/5 & 1.30/4 & 3.45e4/4 & 1.24/4 & 3.85e4/4 & 1.06/4 \\  
        DIRECT-MIN-COV & 2.46e4/4 & 1.63/5 & 1.63e4/4 & 1.71/5 & 3.42e4/5 & 1.63/5 & 3.75e4/5 & 2.33/5 \\ 
        \textbf{RANK(OURS)}&\textbf{1}&\textbf{1}&\textbf{1}&\textbf{1}&\textbf{1}&\textbf{1}&\textbf{1}&\textbf{1}\\
        \midrule
        \multicolumn{1}{c}{\textbf{$N=4$}} & \multicolumn{2}{c}{\textbf{$T=200$}} & \multicolumn{2}{c}{\textbf{$T=400$}} &  \multicolumn{4}{c}{${T\geq600}$ $\textbf{TIMEOUT}$}\\  
        \midrule
        PF & 3.71e4 & 0 & 4.96e4 & 0 & \multicolumn{4}{c}{--} \\  
        $\textbf{SDES}$ & \textbf{2.71e4}{\scriptsize $\pm$4658}/1 & \textbf{0.50}{\scriptsize $\pm$0.17}/1 & \textbf{3.04e4}{\scriptsize $\pm$4589}/1 & \textbf{0.86}{\scriptsize $\pm$0.18}/1 &  \multicolumn{4}{c}{--} \\  
        ORIGAMI-M & 2.06e4/2 & 0.66/2 & 2.56e4/2 & 1.01/2 & \multicolumn{4}{c}{--}\\  
        ORIGAMI-A & \multicolumn{2}{c}{--} & \multicolumn{2}{c}{--} & \multicolumn{4}{c}{--} \\  
        ORIGAMI-M-BS & 1.58e4/4 & 1.41/4 & \multicolumn{2}{c}{--} & \multicolumn{4}{c}{--} \\  
        DIRECT-MIN-COV & 1.97e4/3 & 1.25/3 & \multicolumn{2}{c}{--} & \multicolumn{4}{c}{--} \\ 
        \textbf{RANK(OURS)}&\textbf{1}&\textbf{1}&\textbf{1}&\textbf{1}&\multicolumn{4}{c}{--}\\
\bottomrule
\end{tabular}
\end{adjustbox}
\end{table*}

\begin{table*}[]
\caption{Multi-objective performance indicator of SDES and comparison algorithms across all problem scales in $N=5$}
\label{app:tab:performance:N=5}
\begin{adjustbox}{width=\textwidth}
\centering
\begin{tabular}{lrrrrrrrr}
\toprule
        \multicolumn{1}{c}{\textbf{$N=5$}} & \multicolumn{2}{c}{\textbf{$T=25$}} & \multicolumn{2}{c}{\textbf{$T=50$}} & \multicolumn{2}{c}{\textbf{$T=75$}} & \multicolumn{2}{c}{$T=100$}\\  
        \multicolumn{1}{c}{\textbf{$\bm{Method}$}} & $HV$/Rank & $IGD^{+}$/Rank & $HV$/Rank & $IGD^{+}$/Rank & $HV$/Rank & $IGD^{+}$/Rank & $HV$/Rnak & $IGD^{+}$/Rank \\
        \midrule
        PF & 6.46e4 & 0 & 7.73e4 & 0 & 2.30e6 & 0 & 1.59e6&0 \\  
        $\textbf{SDES}$ & 5.03e5{\scriptsize $\pm$2.59e4}/3 & 0.85{\scriptsize $\pm$0.03}/4 & \textbf{5.61e5}{\scriptsize $\pm$1.23e5}/1 & 0.50{\scriptsize $\pm$0.06}/2 & 7.68e5{\scriptsize $\pm$2.27e4}/3 & 0.55{\scriptsize $\pm$0.01}/2 & \textbf{7.55e5}{\scriptsize $\pm$1.57e5}/1 & 0.63{\scriptsize $\pm$0.10}/2 \\  
        ORIGAMI-M & \textbf{5.20e5}/1 & \textbf{0.49}/1 & 5.27e5/4 & 0.57/3 & 7.29e5/4 & 0.63/3 & 6.92e5/2 & \textbf{0.55}/1\\  
        ORIGAMI-A & 5.02e5/4 & \textbf{0.49}/1 & 5.59e5/2 & \textbf{0.40}/1 & \textbf{8.10e5}/1 & \textbf{0.36}/1  & \multicolumn{2}{c}{--} \\  
        ORIGAMI-M-BS & 4.32e5/5 & 1.56/5 & 4.52e5/5 & 1.23/5 & 6.57e5/5 & 1.12/5 & 5.97e5/4&1.05/3 \\  
        DIRECT-MIN-COV & 5.17e5/2 & 0.61/3 & 5.57e5/3 & 0.64/4 & 7.81e5/2 & 0.79/4  & 6.37e5/3&1.68/4 \\ 
        \textbf{RANK(OURS)}&\textbf{3}&\textbf{4}&\textbf{1}&\textbf{2}&\textbf{3}&\textbf{2}&\textbf{1}&\textbf{2}\\
\bottomrule
\end{tabular}
\end{adjustbox}
\end{table*}

\begin{table*}[]
\caption{Multi-objective performance indicator of SDES and comparison algorithms across all problem scales in $N=6$}
\label{app:tab:performance:N=6}
\begin{adjustbox}{width=\textwidth}
\centering
\begin{tabular}{lrrrrrrrr}
\toprule
        \multicolumn{1}{c}{\textbf{$N=6$}} & \multicolumn{2}{c}{\textbf{$T=25$}} & \multicolumn{2}{c}{\textbf{$T=50$}} & \multicolumn{2}{c}{\textbf{$T=75$}} & \multicolumn{2}{c}{\textbf{$T=100$}}\\  
        \multicolumn{1}{c}{\textbf{$\bm{Method}$}} & $HV$/Rank & $IGD^{+}$/Rank & $HV$/Rank & $IGD^{+}$/Rank & $HV$/Rank & $IGD^{+}$/Rank & $HV$/Rnak & $IGD^{+}$/Rank \\
        \midrule
        PF & 5.40e7 & 0 & 6.10e6 & 0 & 3.46e7 & 0 & 1.15e7&0 \\  
        $\textbf{SDES}$ & \textbf{9.45e6}{\scriptsize $\pm$9.02e6}/1 & \textbf{0.29}{\scriptsize $\pm$0.13}/1 & \textbf{4.78e6}{\scriptsize $\pm$1.73e5}/1 & \textbf{0.32}{\scriptsize $\pm$0.02}/1 & \textbf{1.00e7}{\scriptsize $\pm$3.78e6}/1 & \textbf{0.42}{\scriptsize $\pm$0.10}/1 & \textbf{7.91e6}{\scriptsize $\pm$5.36e5}/1 & \textbf{0.41}{\scriptsize $\pm$0.02}/1 \\  
        ORIGAMI-M & \multicolumn{2}{c}{--} & \multicolumn{2}{c}{--} & \multicolumn{2}{c}{--} & \multicolumn{2}{c}{--}\\  
        ORIGAMI-A & 5.58e6/2 & 0.66/2 & \multicolumn{2}{c}{--} & \multicolumn{2}{c}{--}  & \multicolumn{2}{c}{--} \\  
        ORIGAMI-M-BS & 5.16e6/3 & 0.90/3 & 4.34e6/2 & 0.92/2 & \multicolumn{2}{c}{--} & \multicolumn{2}{c}{--} \\  
        DIRECT-MIN-COV & 4.22e6/4 & 1.66/4 & 3.08e6/3 & 2.48/3 & 5.64e6/2 & 3.05/2  & 5.23e6/2&2.85/2 \\ 
        \textbf{RANK(OURS)}&\textbf{1}&\textbf{1}&\textbf{1}&\textbf{1}&\textbf{1}&\textbf{1}&\textbf{1}&\textbf{1}\\
\bottomrule
\end{tabular}
\end{adjustbox}
\end{table*}

\begin{table*}[]
\caption{Multi-objective performance indicator of SDES and comparison algorithms across all problem scales in $N=7$}
\label{app:tab:performance:N=7}
\begin{adjustbox}{width=\textwidth}
\centering
\begin{tabular}{lrrrrrrrr}
\toprule
        \multicolumn{1}{c}{\textbf{$N=7$}} & \multicolumn{2}{c}{\textbf{$T=25$}} & \multicolumn{2}{c}{\textbf{$T=50$}} & \multicolumn{4}{c}{${T\geq75}$ $\textbf{TIMEOUT}$}\\  
        \multicolumn{1}{c}{\textbf{$\bm{Method}$}} & $HV$/Rank & $IGD^{+}$/Rank & $HV$/Rank & $IGD^{+}$/Rank & \multicolumn{4}{c}{--} \\
        \midrule
        PF & 2.21e8 & 0 & 3.37e8 & 0 & \multicolumn{4}{c}{--} \\  
        $\textbf{SDES}$ & \textbf{8.81e7}{\scriptsize $\pm$3.43e7}/1 &  \textbf{0.39}{\scriptsize $\pm$0.04}/1 & \textbf{1.64e8}{\scriptsize $\pm$4.82e7}/1 & \textbf{0.39}{\scriptsize $\pm$0.04}/1 & \multicolumn{4}{c}{--} \\  
        ORIGAMI-M & \multicolumn{2}{c}{--} & \multicolumn{2}{c}{--}  & \multicolumn{4}{c}{--}\\  
        ORIGAMI-A & 6.95e7/2 & 0.45/2 & \multicolumn{2}{c}{--}  & \multicolumn{4}{c}{--} \\  
        ORIGAMI-M-BS & 5.66e7/3 & 1.17/3 & 1.14e8/2 & 1.25/2& \multicolumn{4}{c}{--} \\  
        DIRECT-MIN-COV & 5.28e7/4 & 1.58/4 & 1.03e8/3 & 2.04/3  & \multicolumn{4}{c}{--} \\ 
        \textbf{RANK(OURS)}&\textbf{1}&\textbf{1}&\textbf{1}&\textbf{1}&\multicolumn{4}{c}{--}\\
\bottomrule
\end{tabular}
\end{adjustbox}
\end{table*}

\begin{table}[]
\caption{Multi-objective performance indicator of SDES and comparison algorithms across all problem scales in $N=8$}
\label{app:tab:performance:N=8}
\centering
\begin{tabular}{lrrrrrrrr}
\toprule
        \multicolumn{1}{c}{\textbf{$N=8$}} & \multicolumn{2}{c}{\textbf{$T=25$}}  & \multicolumn{6}{c}{${T\geq50}$ $\textbf{TIMEOUT}$}\\  
        \multicolumn{1}{c}{\textbf{$\bm{Method}$}} & $HV$/Rank & $IGD^{+}$/Rank & \multicolumn{6}{c}{--} \\
        \midrule
        PF & 4.75e9 & 0 &  \multicolumn{6}{c}{--} \\  
        $\textbf{SDES}$ & \textbf{1.56e9}{\scriptsize $\pm$4.98e8}/1 & \textbf{0.37}{\scriptsize $\pm$0.05}/1 & \multicolumn{6}{c}{--} \\  
        ORIGAMI-M & \multicolumn{2}{c}{--} &  \multicolumn{6}{c}{--}\\  
        ORIGAMI-A & 4.94e8/2 & 0.80/2 & \multicolumn{6}{c}{--} \\  
        ORIGAMI-M-BS &\multicolumn{2}{c}{--} &\multicolumn{6}{c}{--} \\  
        DIRECT-MIN-COV & 4.71e8/3 & 2.16/3 &\multicolumn{6}{c}{--} \\ 
        \textbf{RANK(OURS)}&\textbf{1}&\textbf{1}&\multicolumn{6}{c}{--}\\
\bottomrule
\end{tabular}
\end{table}

\subsection{Additional Experiments about Time Efficiency}
\label{app:more_time_effi}
Section~\ref{sec:Runtime} shows the linear complexity of our method SDES with fixed $pop\_{}size$ in both target and heterogeneous attacker dimensions. SDES is designed to deal with the following requirements: high-dimensional problems in limited time but does not require so many solutions. Other algorithms often fail to complete tasks on time (within the maximum runtime cup $M$) by providing too many solutions, while the general user does not need thousands of solutions. For example, the SOTA algorithm ORIGAMI-M provides about 10000 non-dominant solutions when $N=6$, $T=50$, which cannot accomplish MOSG problems in time. Meanwhile, the effectiveness experiment displays a good performance of our method in multi-objective indicators compared with SOTA algorithms. Therefore, SDES with fixed $pop\_{}size$ can handle more large-scale problems and can provide users with a considerable amount of non-dominated solutions. Those solutions can form a better distributed solution set than other SOTA algorithms.

\begin{figure}[!t]
\centering
\includegraphics[width=.6\textwidth]{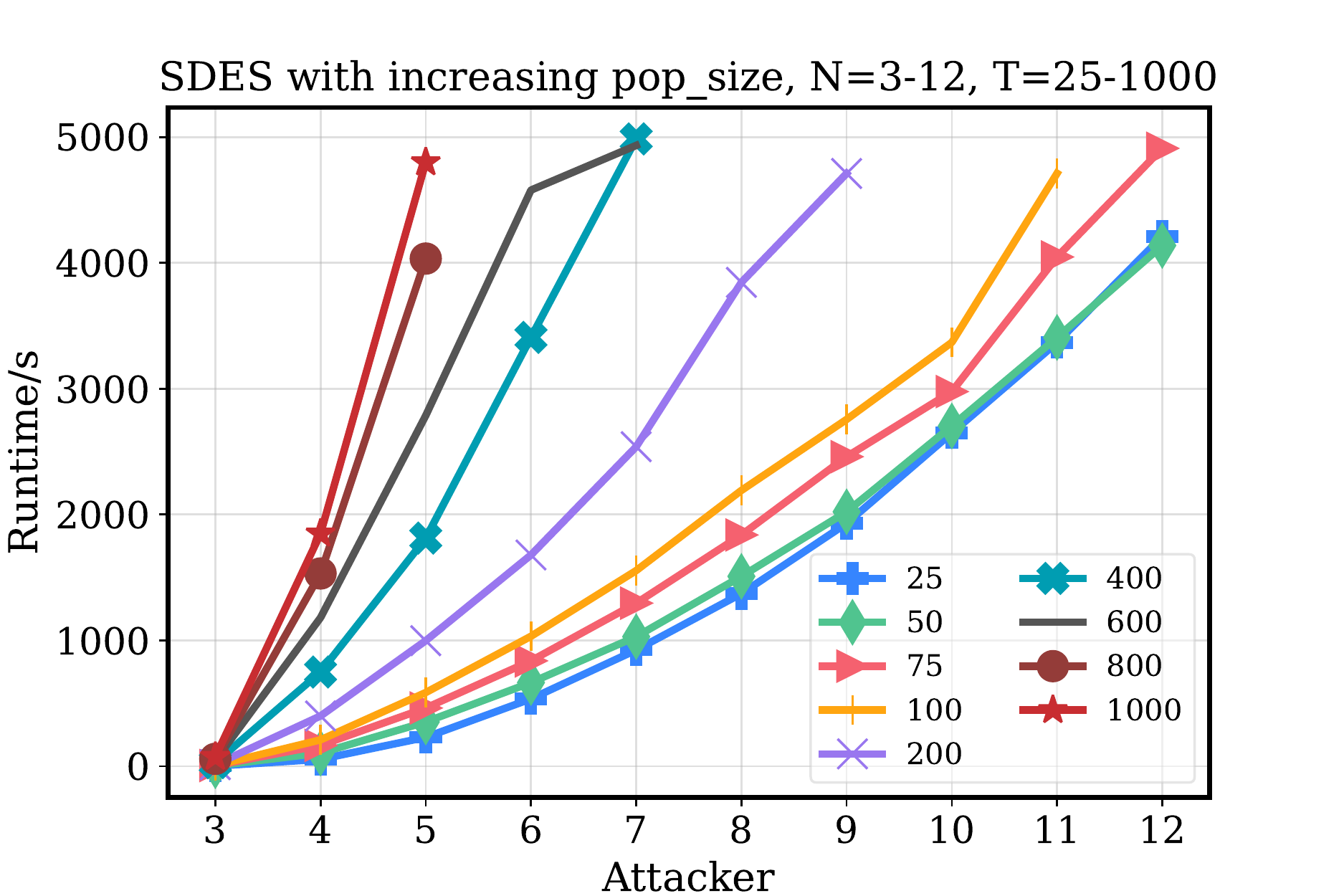}
\caption{The runtime of SDES with increasing $pop\_{}size$. The $pop\_{}size$ is an approximate arithmetic progression $[50, 200, 400,\ldots,1800]$ corresponding to the number of attackers $[3, 4, 5,\ldots, 12]$. The time efficiency of SDES with increasing $pop\_{}size$ is still better than comparison algorithms.}
\label{app:fig:Runtime:SDESF:larger}
\end{figure}

Here we further analyze the influence of $pop\_{}size$ on runtime. In the addition experiment, the population growth of SDES is an increasing sequence with constant steps, cf. Figure~\ref{app:fig:Runtime:SDESF:larger}. As the number of attacker scale up, the runtime of SDES with increasing $pop\_{}size$ exhibits sub-linear growth, which is still faster than all traditional comparison methods. As targets scale up, since $pop\_{}size$ is fixed when the number of attacker is fixed, the runtime of SDES with increasing $pop\_size$ still exhibits linear growth.

\subsection{Additional Experiments about Sensitivity Analysis of Population Size in SDES}
\label{app:more_sen_ana}

\begin{figure}[!hb]
\centering
\begin{minipage}[l]{\columnwidth}\centering
\includegraphics[width=0.34\textwidth]{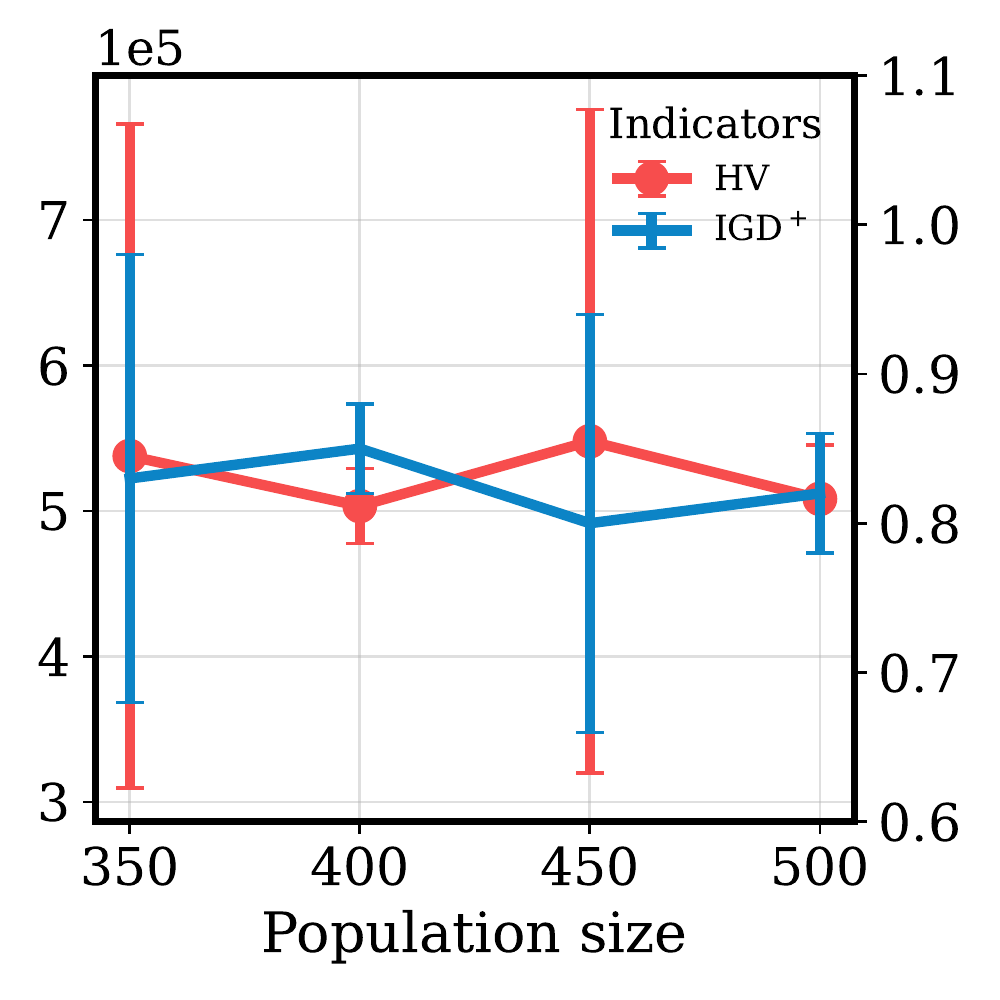}
\includegraphics[width=0.34\textwidth]{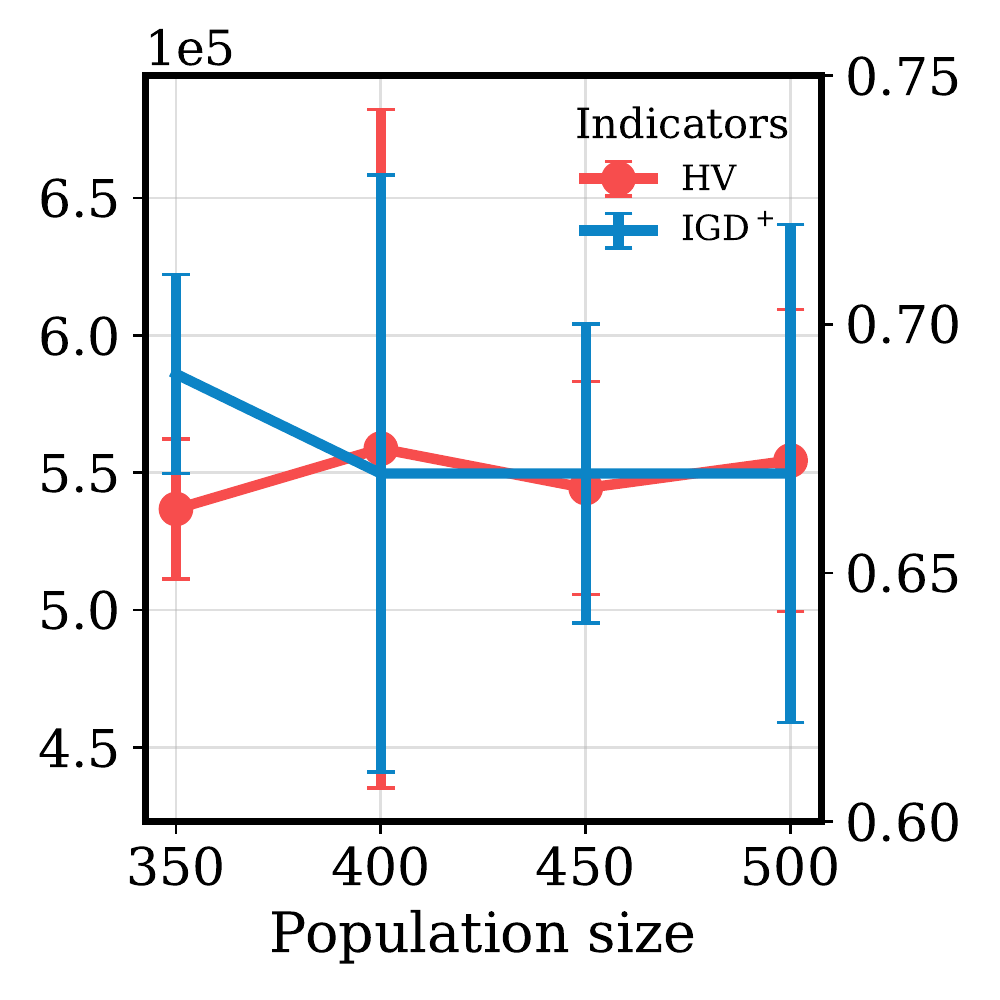}
\includegraphics[width=0.34\textwidth]{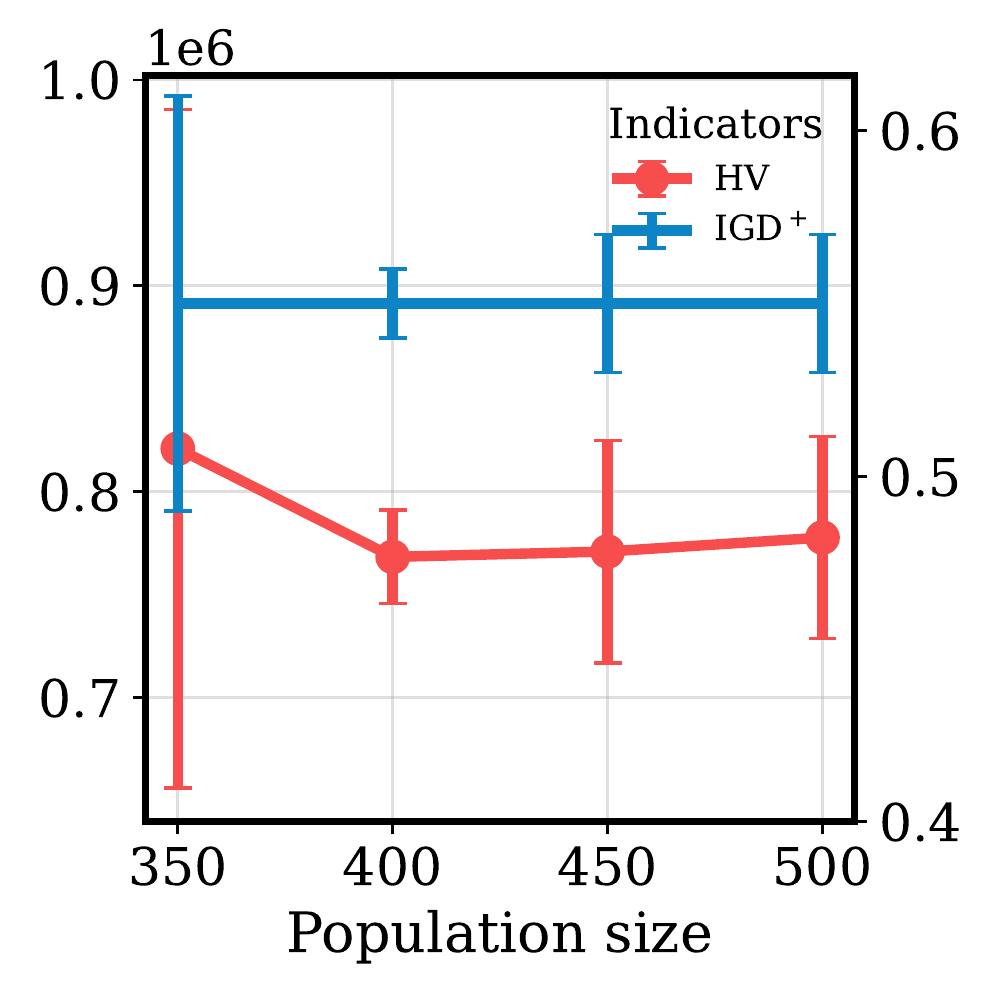}
\includegraphics[width=0.34\textwidth]{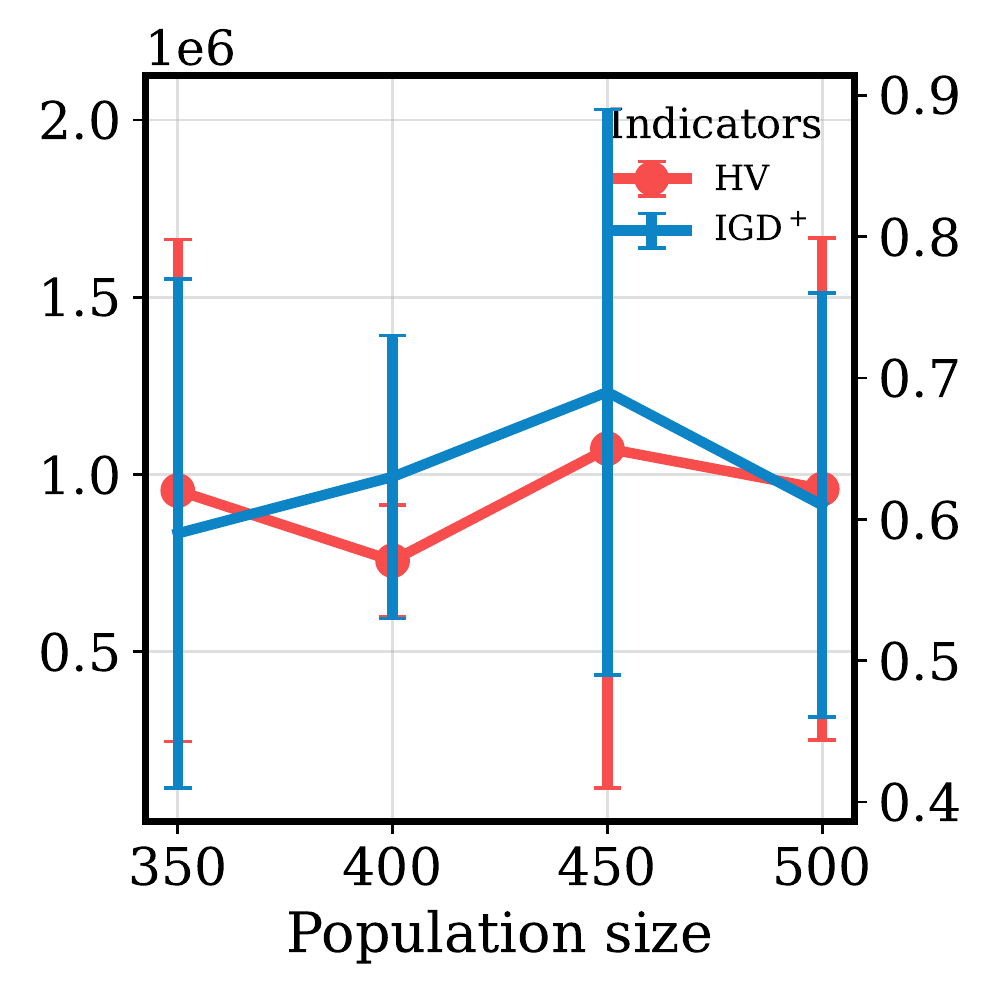}
\caption{Sensitivity Analysis of SDES. The range of population size is $[350, 400, 450, 500]$. The range of problem scale is $N=5$, $T=[25, 50, 75, 100]$. }
\label{app:fig:Sensitivity:N=5}
\end{minipage}
\end{figure}

We analyze how the performance of SDES varies with the population size $pop\_{}size$ and the result is shown in Figure~\ref{app:fig:Sensitivity:N=5}. It is measured by the maximization indicator \textit{HV} and minimization indicator \textit{IGD}$^+$. 
Ideally, the larger $pop\_{}size$, the lower \textit{IGD}$^+$ and higher \textit{HV}. Figure~\ref{app:fig:Sensitivity:N=5} shows that SDES's sensitivity to $pop\_{}size$ depends mainly on the problem difficulty and the randomness of many-objective EAs. 
For MOSGs, Figure~\ref{app:fig:Sensitivity:N=5} illustrates that $pop\_{}size$ has little influence on multi-objective performance indicators. Furthermore, increasing $pop\_{}size$ has less effect on \textit{HV} and \textit{IGD}$^+$ indicators, and \textit{HV} indicator may be reduced by large $pop\_{}size$. In addition, Increasing $pop\_{}size$ does not reduce the standard deviation of the indicator. 
The more difficult the problem, the more sensitive SDES is to $pop\_{}size$. On the one hand, when $T=25$ or $T=50$, \textit{HV} and \textit{IGD}$^+$ show stable changes on the whole and SDES has low sensitivity to $pop\_{}size$. On the other hand, SDES is more sensitive to $pop\_{}size$ when $T=100$ (a harder problem). 
Furthermore, the performance of SDES is more influenced by the randomness of many-objective EAs. There is no obvious connection between the value of \textit{IGD}$^+$ or \textit{HV} and $pop\_{}size$.

\subsection{Additional Experiments about Visualization Analysis}
\label{app:vis_ana}

\begin{figure}[!th]
\centering
\includegraphics[width=.6\textwidth]{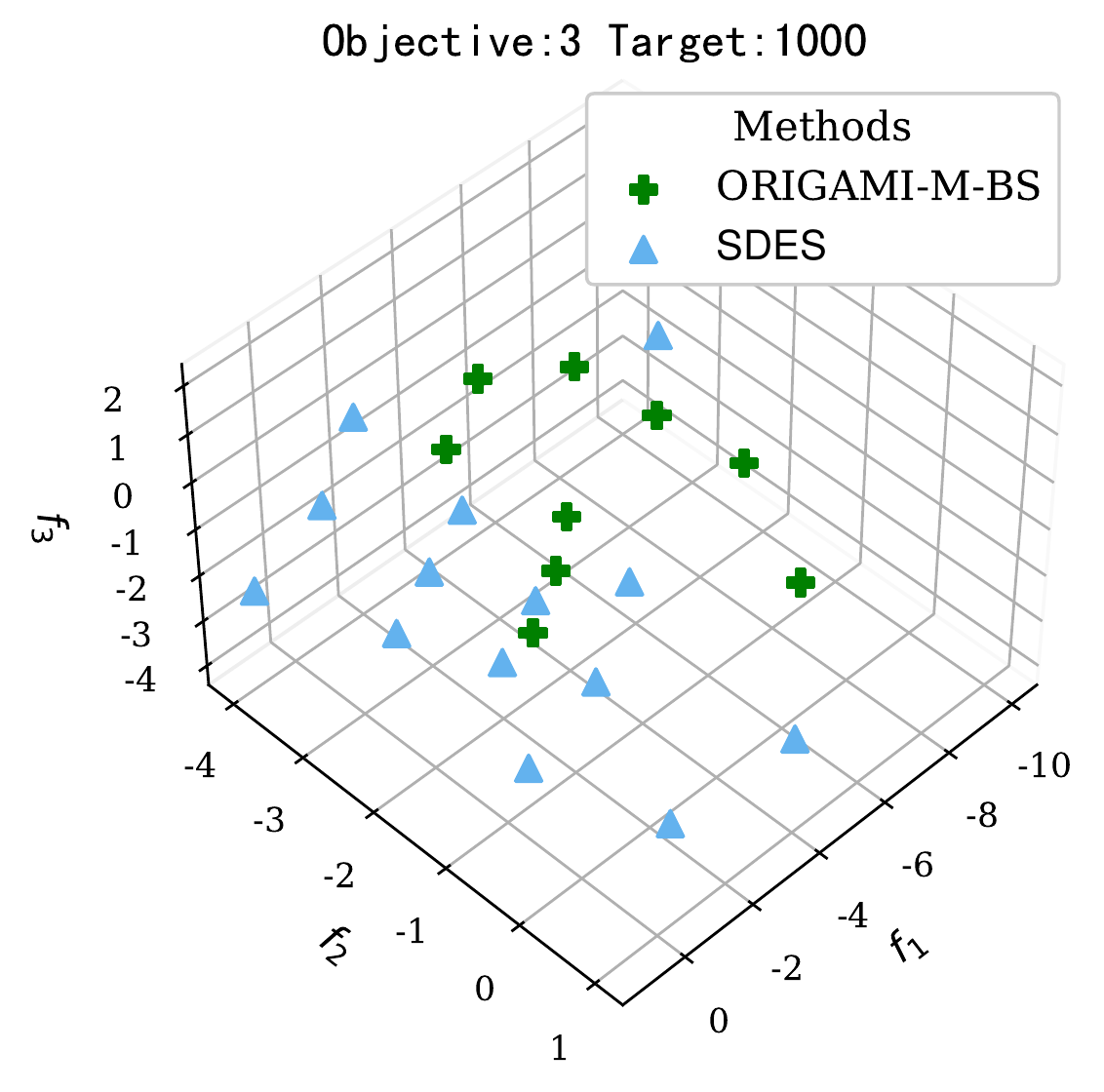}
\caption{PSP visualization of $N=3, T=1000$ MOSG problem. Compared with ORIGAMI-M-BS, SDES can find a well-spaced PF.}
\label{app:fig:visualization:N=3}
\end{figure}

\begin{figure}[!th]
\centering
\includegraphics[width=.6\textwidth]{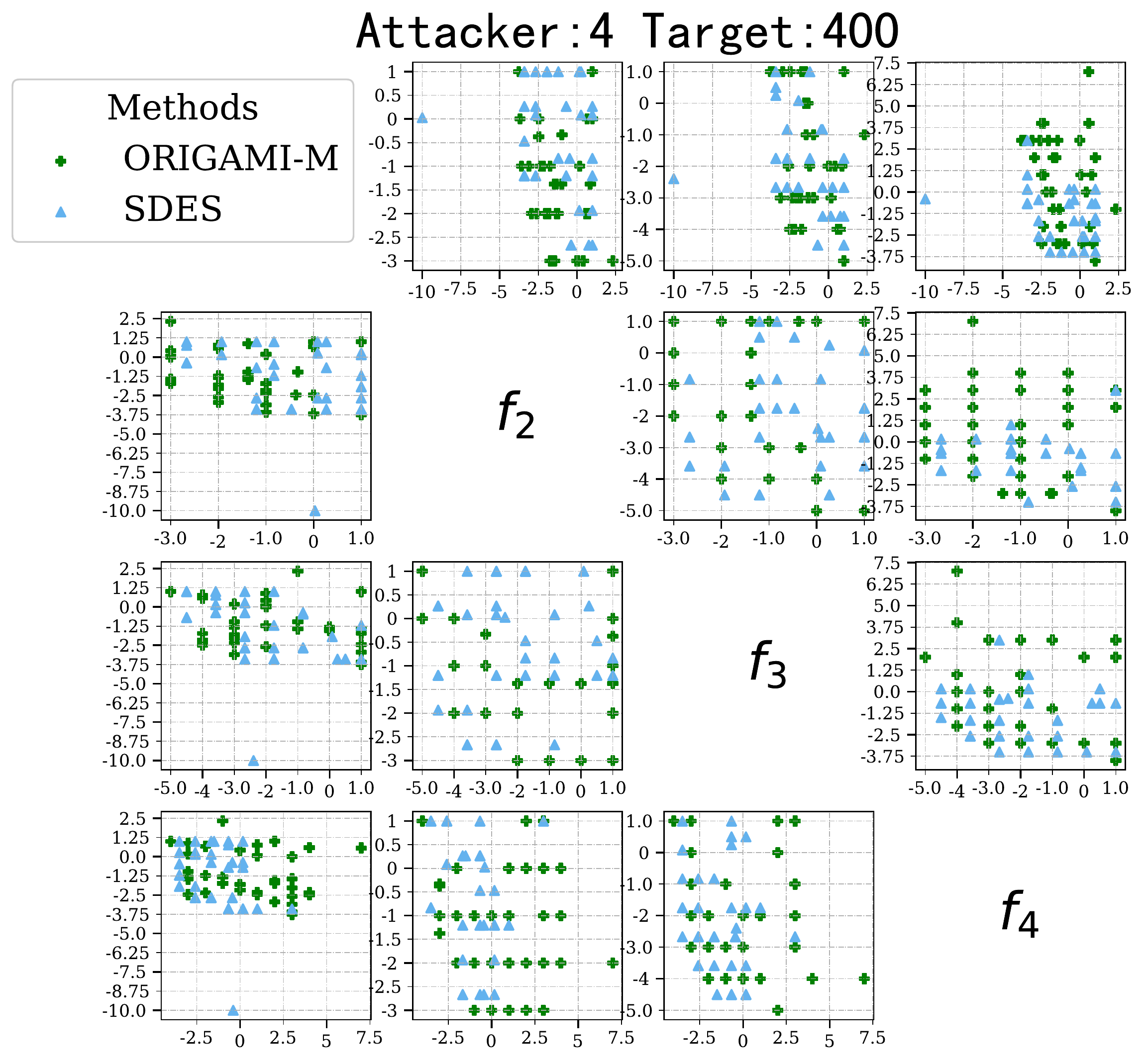}
\caption{PSP visualization of $N=4, T=400$ MOSG problem. MOSG is a minimum task. The closer and more diverse the result to the original, the better the result. From the first column subplots ($f_{2,3,4}$, $f_1$), SDES has the potential to widen the boundaries. From other subplots, SDES is comparable to ORIGAMI-M.}
\label{app:fig:visualization:N=4}
\end{figure}

\begin{figure}[!th]
\centering
\includegraphics[width=.6\textwidth]{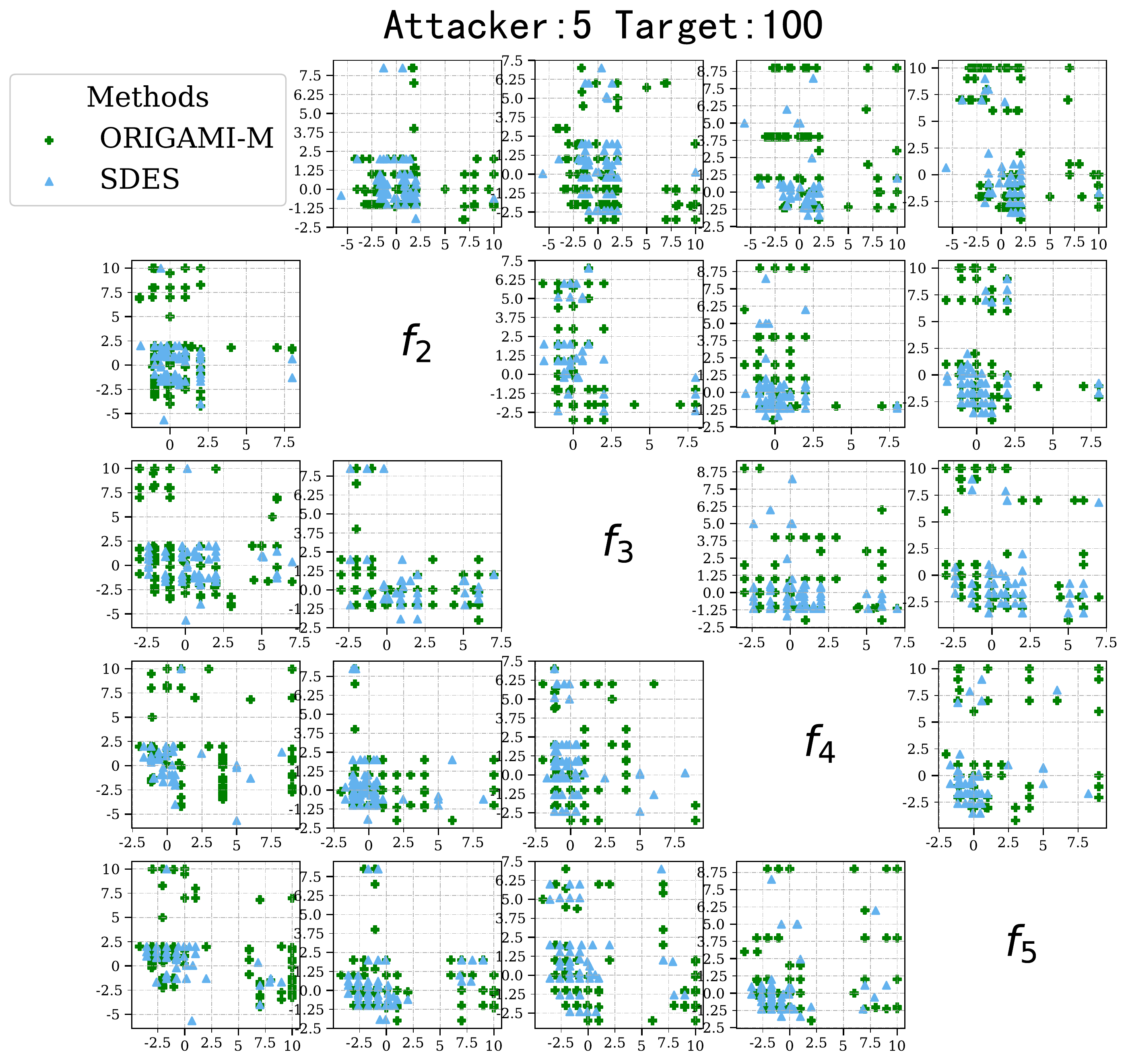}
\caption{PSP visualization of $N=5, T=100$ MOSG problem. In general, there is no significant difference between SDES and ORIGAMI-M in middle-scale MOSG problems. Since the bottom left PF of both methods is similar, the convergence ability of both methods is similar, but the capability of boundary exploration of different methods is different. }
\label{app:fig:visualization:N=5}
\end{figure}

The visualization method can show the PF distribution information that cannot be reflected by performance indicators. The commonly used visualization methods of MaOPs are Pairwise Scatter Plots (PSP) and Parallel Coordinate Plots (PCP). PSP is a classic visualization method to demonstrate the distribution and the breadth of exploration of the PF in the objective space, and PCP is a powerful technique to analyze the PF distribution on each coordinate in the objective space. PSP and PCP can not only judge the optimization difficulty of each target through the optimal value but also analyze the optimization preference of different algorithms through the distribution density. 

We display all visualization results about large-scale MOSG problems by PSP (responsible for $N\leq 5$) and PCP (responsible for $N>5$). The visualization shows the performance of both SDES and the SOTA method (depending on the $HV$ indicator). The SOTA method is drawn in red and SDES is drawn in blue. The table ticks of axes represent the defender payoff $(U_1^d(\bm{c}),\ldots,U_N^d(\bm{c}))$. MOSG problem aims to find the minimum defender loss. 

\begin{figure}[!bh]
\centering
\includegraphics[width=.5\textwidth]{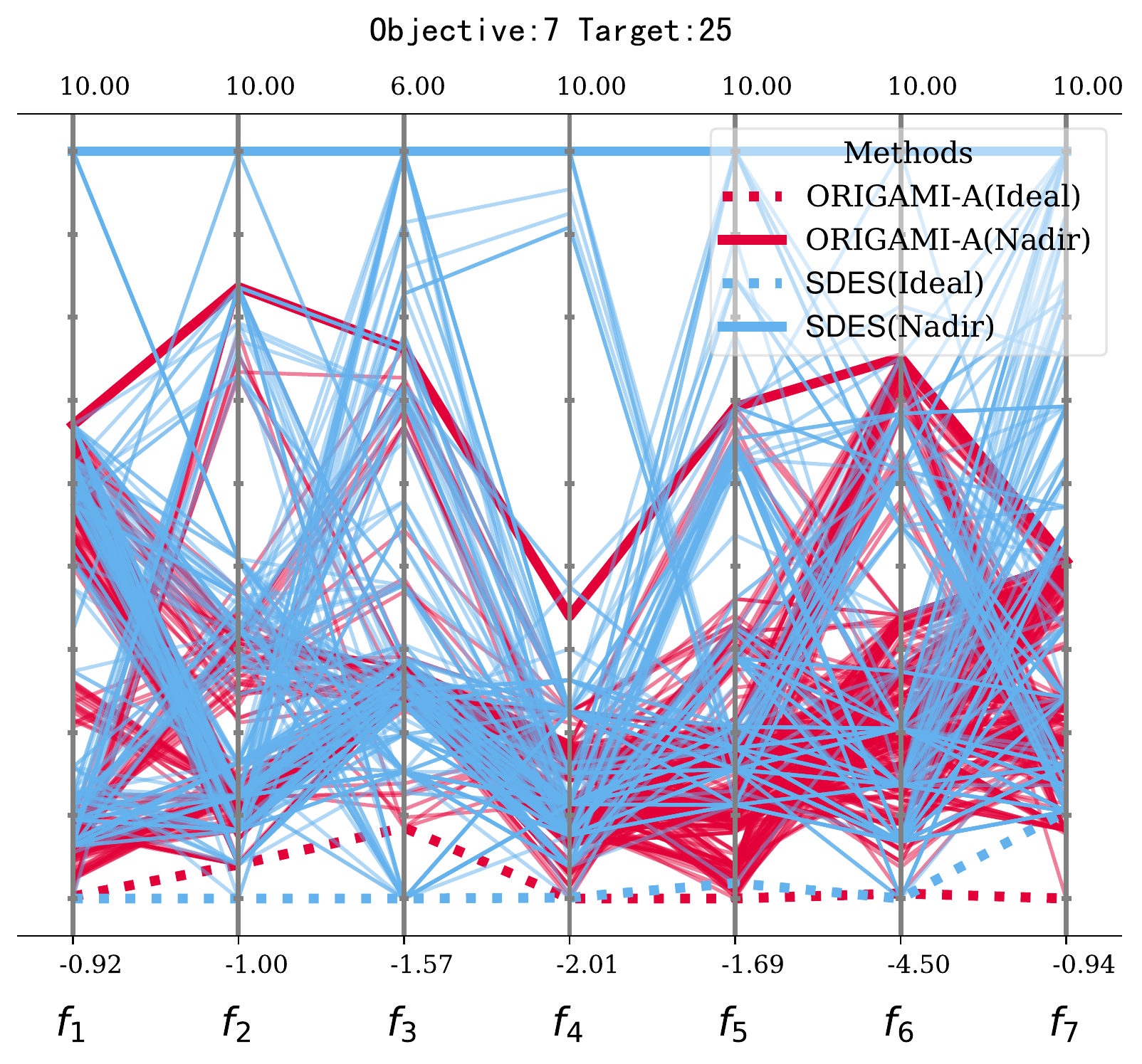}
\caption{PCP visualization of $N=6, T=100$ MOSG problem. On the one hand, since the ideal solution (dotted lines) and the nadir solution (solid lines) of SDES (blue lines) and DIRECT-MIN-COV (red lines) cannot completely include each other, the ability of boundary exploration of SDES and DIRECT-MIN-COV is comparable. On the other hand, compared with DIRECT-MIN-COV, SDES can find a well-spaced PF. }
\label{app:fig:visualization:N=6T=100}
\end{figure}

\begin{figure}[!bh]
\centering
\includegraphics[width=.5\textwidth]{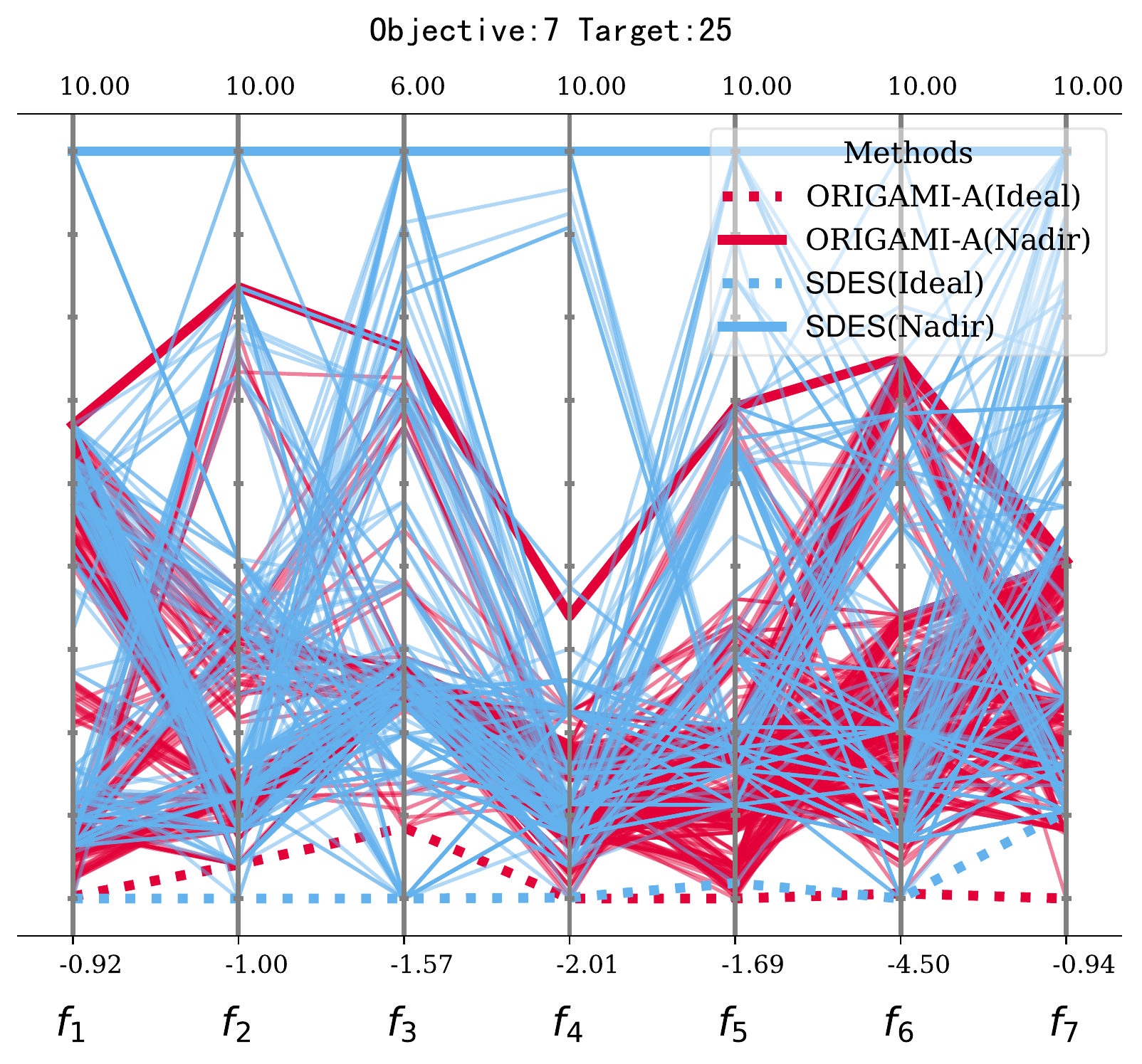}
\caption{PCP visualization of $N=7, T=25$ MOSG problem. When $N$ scales up to 7, the superiority of SDES is revealed both in the ability of boundary exploration and the uniformity and diversity of PF results.}
\label{app:fig:visualization:N=7T=25}
\end{figure}

\begin{figure}[!bh]
\centering
\includegraphics[width=.5\textwidth]{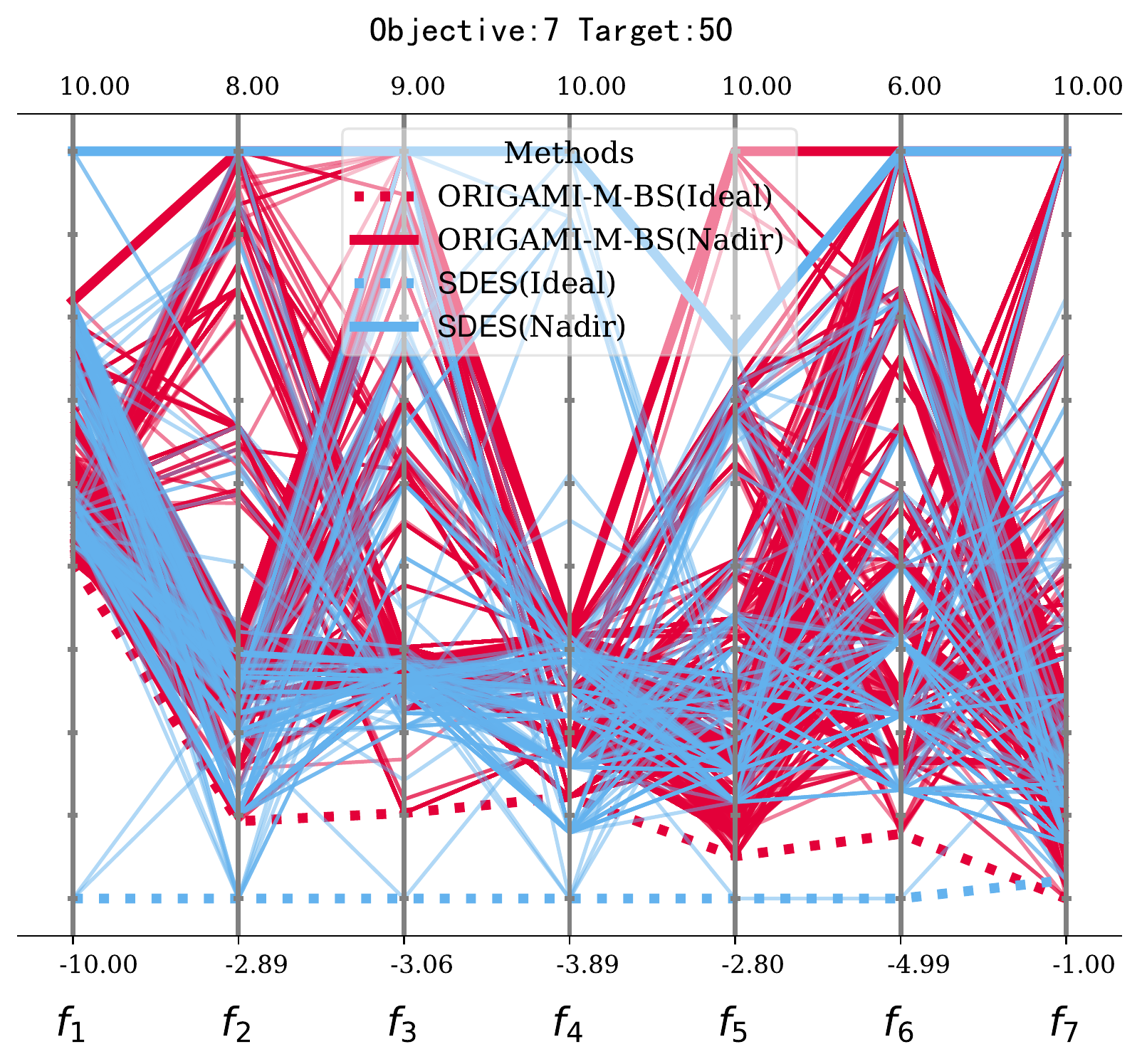}
\caption{PCP visualization of $N=7, T=50$ MOSG problem. When $N$ scales up to 7, the superiority of SDES is revealed both in the ability of boundary exploration and the uniformity and diversity of PF results.}
\label{app:fig:visualization:N=7T=50}
\end{figure}

\begin{figure}[!bh]
\centering
\includegraphics[width=.5\textwidth]{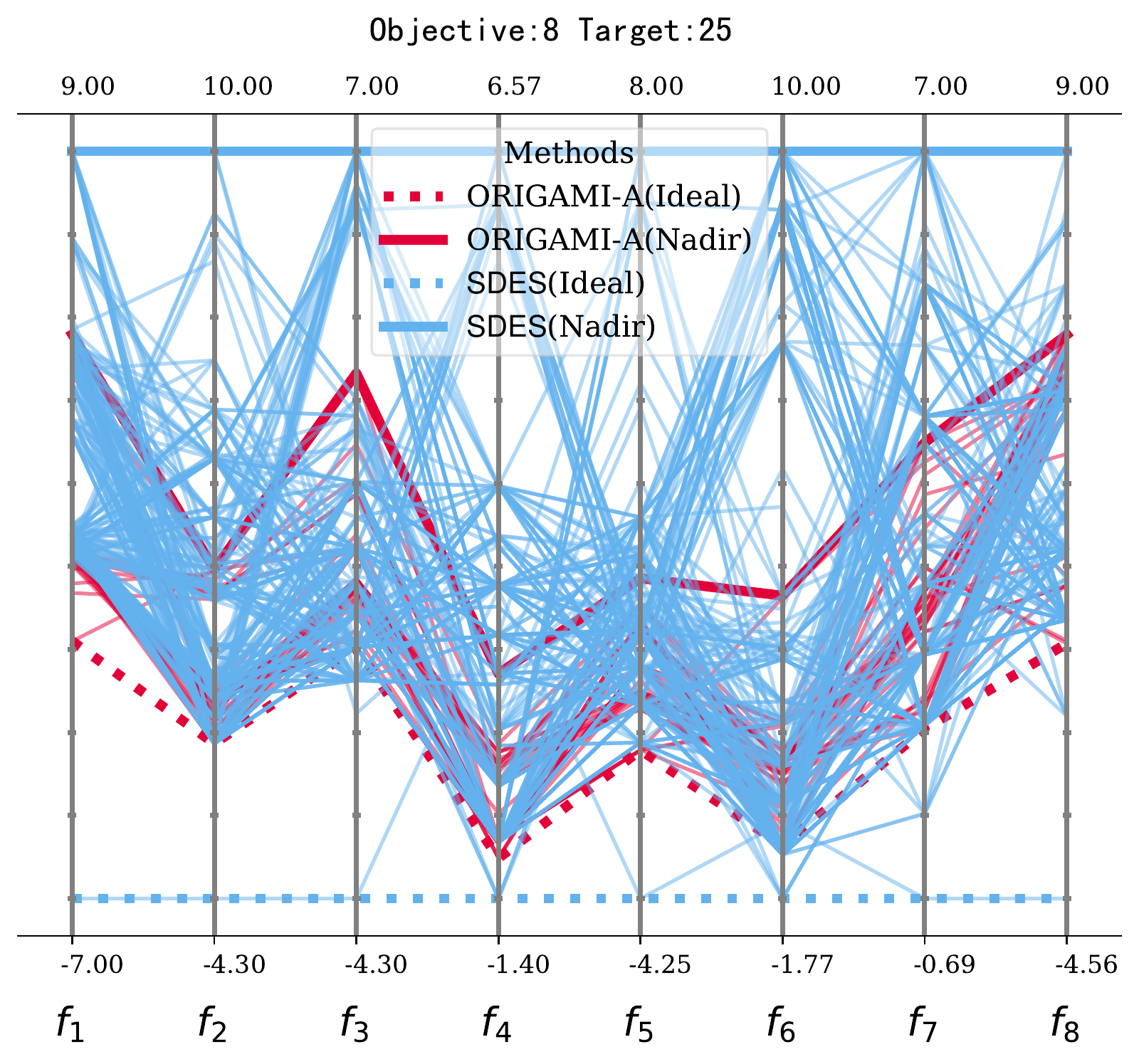}
\caption{PCP visualization of $N=8, T=25$ MOSG problem. When $N$ scales up to 8, although ORIGAMI-A is one of the SOTA methods in middle-scale MOSG problems, the exploration range of ORIGAMI-A is limited in large-scale MOSG problems. The performance of our methods SDES is satisfactory in both solution distribution and convergence. }
\label{app:fig:visualization:N=8T=25}
\end{figure}

In PSP, as shown in Figure~\ref{app:fig:visualization:N=3},~\ref{app:fig:visualization:N=4}, and~\ref{app:fig:visualization:N=5}, the closer and more diverse the PF found by the method to the origin, the better the method.
PSP is used in medium-scale MOSG problem visualization, e.g., $N=3, 4$. As for $N=3$, PSP is a 3D figure displayed from a 45-degree viewing angle, e.g., Figure~\ref{app:fig:visualization:N=3}. As for $N>3$, PSP is symmetrical, we only need to observe the lower triangle, e.g., Figure~\ref{app:fig:visualization:N=4}. 
When analyzing, we use coordinates to refer to subplots, for example, the subplot of Figure~\ref{app:fig:visualization:N=4} in the lower right corner is called $(f_4, f_1)$. Figure~\ref{app:fig:visualization:N=4} illustrates (i) Compared with the SOTA comparison method, SDES usually can find a better PF. (ii) SDES can explore larger boundaries (e.g., the boundary point at the first column subplots ($f_{2,3,4}$, $f_1$) of Figure~\ref{app:fig:visualization:N=4}). More visualization analysis details in the figure captions.

In PCP, cf. Figure~\ref{app:fig:visualization:N=6T=100},~\ref{app:fig:visualization:N=7T=25},
~\ref{app:fig:visualization:N=7T=50}, and~\ref{app:fig:visualization:N=8T=25}, the thin lines depict the distribution of the PF found by SDES and the SOTA method in the objective space, and the thick lines at the top and bottom of the figure depict the exploration capability of methods. The more diverse the thin lines, the better the method, and the wider the distance between the top thick line to the bottom thick line, the better the method. The top thick lines are also called ideal lines, composed of the optimal values of all objectives, and the bottom thick lines are also called nadir lines, composed of the worst values of all objectives.
All PCP figures illustrate (i) The blue thin lines (SDES) are basically covered by the red lines (SOTA comparison method), indicating that the PF returned by SDES is better than others in distribution. (ii) The Ideal and Nadir of SDES basically surround the SOTA comparison method, which means that the width of the PF returned by SDES is also better than others. More details of visualization analysis are described in the figures' captions.
\end{document}